\documentclass[11pt]{article} % Anonymized submission

\usepackage{macro_yihan}

\title{Spectral Estimators for Multi-Index Models:\\ Precise Asymptotics and Optimal Weak Recovery}

\author{
    Filip Kova\v{c}evi\'c\thanks{Institute of Science and Technology Austria. Email: \href{mailto:filip.kovacevic@ist.ac.at}{\texttt{filip.kovacevic@ist.ac.at}}.} 
    \and
    Yihan Zhang\thanks{University of Bristol. Email: \href{mailto:yihan.zhang@bristol.ac.uk}{\texttt{yihan.zhang@bristol.ac.uk}}.}
    \and
    Marco Mondelli\thanks{Institute of Science and Technology Austria. Email: \href{mailto:marco.mondelli@ist.ac.at}{\texttt{marco.mondelli@ist.ac.at}}.}
}

\begin{document}

\maketitle

\newcounter{asmpctr} % counter for assumptions
\newcounter{asmpadd} % counter for additional assumptions
\setcounter{asmpctr}{\value{enumi}}
\setcounter{asmpadd}{\value{enumi}}

\renewcommand{\qedsymbol}{$\blacksquare$}

\begin{abstract}%
Multi-index models provide a popular framework to investigate the learnability of functions with low-dimensional structure and, also due to their connections with neural networks, they have been object of recent intensive study. In this paper, we focus on recovering the subspace spanned by the signals via spectral estimators -- a family of methods routinely used in practice, often as a warm-start for iterative algorithms. Our main technical contribution is a precise asymptotic characterization of the performance of spectral methods, when sample size and input dimension grow proportionally and the dimension $p$ of the space to recover is fixed. Specifically, we locate the top-$p$ eigenvalues of the spectral matrix and establish the overlaps between the corresponding eigenvectors (which give the spectral estimators) and a basis of the signal subspace. Our analysis unveils a phase transition phenomenon in which, as the sample complexity grows, eigenvalues escape from the bulk of the spectrum and, when that happens, eigenvectors recover directions of the desired subspace. The precise characterization we put forward enables the optimization of the data preprocessing, thus allowing to identify the spectral estimator that requires the minimal sample size for weak recovery.
\end{abstract}

\let\citet\cite

\section{Introduction}
\label{sec:intro}

Modern machine learning practices operate on high-dimensional datasets that are believed to possess low-dimensional structures, and multi-index models are a popular statistical framework for studying such scenarios \cite{Li1,Li2,Li3}. Specifically, 
for a labeled dataset $ \cD = \brace{(a_i, y_i)}_{i = 1}^n $, where $ a_i \in \bbR^d $ and $ 
y_i \in \bbR $ denote features and responses respectively, a multi-index model postulates that each data pair follows a generalized regression and the responses are function of a \emph{low-dimensional} projection of the high-dimensional features. 
In formulas: 
\begin{equation}\label{eq:multi}
y_i = q\paren{ \inprod{a_i}{w_1^*}, \cdots, \inprod{a_i}{w_p^*}, \eps_i }, 
\end{equation}
where the link function $ q $ is a known nonlinearity operating on a $p$-dimensional linear transformation of $a_i$, given by $ W^* = \matrix{ w_1^*, & \cdots, & w_p^* } \in \bbR^{d\times p} $, and on additional randomness from $\eps_i$. 

We study the problem of estimating the signals $W^*$ given $ \cD $ under a Gaussian design where the $ a_i $'s are i.i.d.\ standard Gaussian. We focus on the proportional asymptotic regime where $n,d\to\infty$ with $ n/d \to \delta \in ]0,\infty[ $, while keeping the dimension of the low-dimensional projection $p$ fixed. 
Of particular interest in this work is the \emph{weak recovery} of the subspace spanned by $ w_1^*, \cdots, w_p^* $ which seeks an estimator $ \wh{W} = \matrix{\wh{w}_1, & \cdots, & \wh{w}_p} \in \bbR^{d\times p} $ s.t.\ at least one of $ \wh{w}_i $'s is non-trivially correlated with some linear combination of the signals. 
Formally, there exists $ v\in\bbS^{p-1} $ and some $ \wh{w}_i $ in $ \wh{W} $ such that the normalized correlation (or \emph{overlap}) $ \frac{\abs{\inprod{\wh{w}_i}{W^* v}}}{\normtwo{\wh{w}_i} \normtwo{W^* v}} $ is asymptotically non-vanishing. 

To solve the aforementioned problem, this paper focuses on %considers a family of efficient estimators known as spectral estimators for weak recovery. 
spectral estimators that construct the matrix
%\begin{equation}\label{eq:intro}
$D = \frac{1}{n} \sum_{i = 1}^n \cT(y_i) a_ia_i^\top$,
%\end{equation}
from the dataset $\cD$ and then output its top-$p$ eigenvectors. % $ \wh{W}^{\textnormal{s}} = \matrix{ v_1^D, & \cdots, & v_p^D } $. % as the eigenvectors corresponding to the largest $p$ eigenvalues of $D$.
Here, $\cT \colon \bbR \to \bbR$ refers to a user-defined preprocessing function, with popular choices including binary quantization and truncation \cite{Wang_subset,Chen_Candes}. %, see \Cref{sec:experiments}
%for concrete expressions. %also \Cref{itm:subset,itm:trim} on page \pageref{pg:T}. 
Spectral estimators are easy to design, efficient to compute, and effective in practice \cite{Chen_monograph}. %,Couillet_Liao_book}. %They have been applied to single-index  
%The idea behind spectral methods finds its root in \cite{Li2} where it was first developed in the low-dimensional regime with $d$ fixed and $n$ large. 
%Spectral methods have since been applied to single-index models (for which $p = 1$; a.k.a.\ generalized linear models) \cite{lu2020phase}, 
%community detection \cite{abbe2017community},
%clustering \cite{ng2001spectral}, angular synchronization \cite{singer2011angular}, low-rank matrix \cite{Mont_Venk_AOS} and tensor estimation \cite{MR_tensor_PCA}, etc. 
However, such class of methods remains understudied for multi-index models, with existing results falling short of producing exact asymptotics \cite{Chen_Meka} and being restricted to special cases, such as single-index (for which $p=1$, corresponding to generalized linear models \cite{mccullagh2019generalized}) \cite{lu2020phase,mondelli-montanari-2018-fundamental}, polynomial link functions \cite{Chen_Meka}, and mixed regression \cite{mixed-zmv-arxiv}, see also \Cref{sec:related}. 
In contrast, this paper tackles the problem for fully general multi-index models, precisely identifying under what conditions spectral estimators are effective. 

A compelling motivation for the precise analysis of spectral estimators comes from the choice of the preprocessing function $\cT$. In the single-index case, its optimization has led to significant performance gains for Gaussian \cite{mondelli-montanari-2018-fundamental,Luo_Alghamdi_Lu}, correlated \cite{Zhang_COLT} and rotationally invariant \cite{maillard2022construction} designs. Remarkably, the related optimal spectral estimators have also been shown to achieve the computational limits of weak learnability for a class of link functions. In the proportional regime of interest in this work ($n=\Theta(d)$),\footnote{Related, albeit different, perspectives are to consider a sample complexity polynomial in $d$ \cite{damian2023smoothing,Damian} or information-theoretic limits \cite{GLM_PNAS,Aubin_comm_machine}.} such limits have been connected to the performance of Approximate Message Passing (AMP) \cite{mondelli-montanari-2018-fundamental,Maillard} and, more precisely, to the stability of its trivial fixed point. Here, AMP refers to a family of iterative algorithms that is provably optimal among first-order methods \cite{celentano2020estimation,montanari2024statistically} and, in fact, there is significant evidence that AMP is optimal even among all polynomial-time algorithms, i.e., it achieves Bayes-optimal performance unless a statistical-to-computational gap is present \cite{GLM_PNAS,PCA_PNAS, Mont_Venk_AOS,VAMP,VKM}. Most recently, the multi-index case %with orthogonal signals 
has been considered by \citet{troiani2024fundamental} and a threshold capturing its computational limits has been computed by studying whether AMP is able to improve on a non-trivial initialization, see \Cref{rmk:comp} for a connection with these results.

\paragraph{Main contributions.} This paper tackles the two main problems mentioned above, i.e., \emph{(i)} the lack of a precise analysis of spectral estimators for multi-index models, and \emph{(ii)} the design of an optimal preprocessing function $\cT$. Specifically, our results are summarized below.

\begin{enumerate}
 [leftmargin=6mm]
    \item For any preprocessing function $\cT$ satisfying mild assumptions, we precisely locate the top-$p$ eigenvalues of the spectral matrix $D$ and characterize the overlaps between the top-$p$ eigenvectors and a basis of the subspace spanned by the signals. Our results describe a phase transition phenomenon, akin to the classical BBP transition in the spiked covariance model \cite{BBAP}, with spectral outliers of $D$ corresponding to eigenvectors employed for signal recovery.

    \item Using the above characterization, we identify the optimal preprocessing function for weak recovery. 
    Our optimality result guarantees that no other choice of preprocessing function results in spectral estimators with a lower %weak recovery 
    threshold, and therefore it also implies the suboptimality of existing heuristics to choose $\cT$.  
\end{enumerate}

\section{Related work}
\label{sec:related}

% % these have been covered in the intro, so I commented it 
% \paragraph{Spectral estimators for single-index models.}
% Our work builds on a line of works on spectral estimators for single-index models. 

\paragraph{Multi-index models.}
Several approaches have been proposed to perform statistical inference in multi-index models, e.g., 
structural adaption via maximum minimization \cite{dalalyan2008new}, 
projection pursuit regression \cite{yuan2011identifiability} and
techniques from compressed sensing 
\cite{fornasier2012learning}. % and the estimation of score functions \cite{babichev2018slice}. 
\citet{andoni2014learning,Chen_Meka} consider polynomial link functions, with the latter work proposing a spectral warm start that requires a sample size $ n \gtrsim d (\log(d))^{\deg(q)} $, where $\deg(q)$ is the degree of the link. Due to the connection of multi-index models with two-layer neural networks, the area has witnessed a recent renewed interest with a focus on the performance of gradient-based methods. In particular, sample complexity bounds for gradient descent and statistical query lower bounds are provided by \citet{damian2022neural} when the link function is polynomial and by  \citet{Oko_GAM} when the multi-index model is the sum of single-index models. 
\citet{abbe2022merged,abbe2023sgd} introduce the concept of leap complexity and show that a class of staircase functions is learned via one-pass stochastic gradient descent (SGD) with $n = \Theta(d)$ samples. The leap exponent also appears in \cite{MIM_GF} as the time required by gradient flow to escape a saddle point. 
\citet{collins2024hitting} prove a deterministic equivalent of the SGD dynamics.  
\citet{Ren_Lee} provide an algorithm that recovers orthogonal multi-index models with a sample complexity matching the information exponent 
\cite{arous2021online}. 
We note that none of these methods pin-points exactly the sample complexity required to recover a multi-index model, which constitutes the focus of our work and is achieved via the class of spectral methods reviewed below.

\paragraph{Spectral estimators.} %The idea behind spectral methods finds its root in \cite{Li2} where it was first developed in the low-dimensional regime with $d$ fixed and $n$ large. 
Spectral estimators have been applied to a variety of problems, e.g., community detection \cite{abbe2017community}, %clustering \cite{ng2001spectral}, 
angular synchronization \cite{singer2011angular} and principal component analysis (PCA) \cite{Mont_Venk_AOS}. % and tensor estimation \cite{MR_tensor_PCA}. 
 In the setting of Gaussian design and proportional scaling between $n$ and $d$, their asymptotic performance %of spectral methods 
 for single-index models is characterized by \cite{lu2020phase}. 
 Optimal weak recovery thresholds and optimal  overlaps are identified by \citet{mondelli-montanari-2018-fundamental} and \cite{Luo_Alghamdi_Lu}, respectively. 
 The above results are extended by \cite{dudeja-2020-rigorous-analysis} to subsampled Haar designs %where $ A = \matrix{a_1, & \cdots, & a_n}^\top \in \bbR^{n\times d} $ is obtained by truncating a random orthogonal matrix, 
 and by \cite{Zhang_COLT} to correlated Gaussian designs. % where the $ a_i $'s are i.i.d.\ Gaussian with a given covariance. 
 Rotationally invariant designs are considered by \cite{maillard2022construction}, which conjecture the form of the optimal spectral estimator using a linearization of AMP and the analysis of the Bethe Hessian. Such conjecture is partly addressed by \citet{Zhang_COLT}, when the covariance of the $a_i$'s is rotationally invariant. 
We note that, in the single-index case, optimal spectral methods match computational thresholds obtained from the stability of AMP \cite{mondelli-montanari-2018-fundamental,Zhang_COLT} and, in special cases, information-theoretic thresholds as well \cite{mondelli-montanari-2018-fundamental}. An optimally-designed spectral estimator is able to meet the information-theoretic limits of weak recovery also for a class of heteroscedastic PCA problems \cite{MatrixDenoising}. Most closely related to our setting is work by \cite{mixed-zmv-arxiv}: the authors consider mixtures of single-index models with independent signals and provide precise asymptotics for spectral estimators by using a mix of tools from random matrix theory and the theory of AMP. In contrast, our approach is purely random matrix theoretic, and it allows us to handle a general class of multi-index models with arbitrary correlation among the signals. %\fk{Should we drop the emphasize on arbitrary correlation as it can WLOG be absorbed in the link function?}

We finally note that a parallel paper by \citet{defilippis2025optimal} also considers a %setting \yihan{similar to ours} (
multi-index model with Gaussian design % \yihan{and orthogonal signals}) 
and introduces spectral estimators based on the linearization of AMP. \citet{defilippis2025optimal} then conjecture that such spectral estimators are optimal in the sense that they achieve the computational threshold identified by \cite{troiani2024fundamental}, providing both numerical and rigorous evidence in favor of the conjecture. Our work focuses on a family of spectral estimators popular in the related literature (cf.\ Eq.\ (41) in \citet{defilippis2025optimal} and \eqref{eqn:D} in our work), and it resolves the conjecture of \citet{defilippis2025optimal} for a wide class of link functions $q$, including all permutation-invariant ones, see \Cref{rmk:comp} for details. 

\section{Preliminaries}
\label{sec:prelim}

\paragraph{Notation.} Given a positive integer $n$, we use the shorthand $[n]:=\{1, \ldots, n\}$. We denote by $0_n$ the vector of length $n$ with all zeros. For a symmetric matrix $M$, we denote its Moore-Penrose pseudo-inverse as $M^\dagger$, the set of all its eigenvalues as $\Lambda^M$, its $i$-th largest eigenvalue as $\lambda_i^M$ and the corresponding $i$-th eigenvector of unit norm as $v_i^M$. 
We use $ (e_i^{(d)})_{i\in[d]} $ to denote the canonical basis of $ \bbR^d $ and suppress the superscript whenever there is no confusion. 
Unless otherwise specified, all limits of sequences of random quantities are computed in an almost sure sense as $n,d\to\infty$. 

\paragraph{Multi-index models and weak recovery.} We consider the problem of estimating $p$ signals using $n$ responses $(y_i)_{i\in [n]}$ generated i.i.d.\ from \eqref{eq:multi},
where $ \eps_i \in \bbR $ accounts for noise and $q:\bbR^{p+1}\to \bbR$ is the link function. 
We denote by $ A \coloneqq \matrix{ a_1, & \cdots, & a_n }^\top \in \bbR^{n\times d} $ the design matrix, by $ W^* \coloneqq \matrix{ w_1^*, & \cdots, & w_p^* } \in \bbR^{d\times p} $ the signal matrix, and by $ y \coloneqq \matrix{y_1, & \cdots, & y_n}^\top \in \bbR^n $ the response vector. 
Throughout the paper, we impose the following assumptions.

\begin{enumerate}[label=(A\arabic*)]
\setcounter{enumi}{\value{asmpctr}}

    \item \label[asmp]{asmp:A} $ A \in \bbR^{n\times d} $ contains i.i.d.\ standard Gaussian elements, i.e., $ A_{i,j} \iid \cN(0,1) $. 

    \item \label[asmp]{asmp:W*} $ (w_i^*)_{i\in[p]} $ are linearly independent, have unit norm and, if random, they are independent of $A$. 

    \item \label[asmp]{asmp:proportional} $p\ge1$ is fixed and $ n,d \to \infty $ s.t.\ $ n/d\to\delta\in]0,\infty[ $. 

    \item \label[asmp]{asmp:noise} $ \eps_1, \cdots, \eps_n $ are independent of $ A, W^* $, and they are i.i.d.\ according to a distribution $ P_\eps $ on $\bbR$ with finite first two moments. 

\setcounter{asmpctr}{\value{enumi}}
\end{enumerate}

\noindent The linear independence requirement in \Cref{asmp:W*} is mild.
For the purpose of subspace recovery (formally defined below), the presence of linearly dependent signals does not change the recoverability of a given estimator. 

%To estimate $ p $ linearly independent signals, it is expected that $p$ linearly independent estimators are needed. For a given collection of estimators $ \wh{W} = \matrix{\wh{w}_1, & \cdots, & \wh{w}_p} \in \bbR^{d\times p} $, we consider the following notion of weak recovery. 

\begin{definition}[Weak recovery]
    \label{def:weak_rec}
    Consider the model \Cref{eq:multi}. 
    Let $ \wh{W} \equiv \wh{W}(A, y) = \matrix{\wh{w}_1, & \cdots, & \wh{w}_p} \in \bbR^{d\times p} $ be an estimator such that $ \normtwo{\wh{w}_i} = 1 $ for all $i\in[p]$. 
    We say that $ \wh{W} $ weakly recovers the subspace $ \spn\brace{ w_1^*, \cdots, w_p^* } $ if 
    \begin{align}\label{eq:ournotion}
        \max_{v\in\bbS^{p-1}} \brace{ \liminf_{d\to\infty} \frac{\normtwo{\wh{W}^\top W^* v}}{\normtwo{W^* v}} } &> 0 ,  
    \end{align}
    where the almost sure limit is taken with respect to the proportional scaling in \Cref{asmp:proportional}. 
\end{definition}
In words, %an estimator 
$ \wh{W} $ weakly recovers $ \cW \coloneqq \spn\brace{ w_1^*, \cdots, w_p^* } $ if at least one of the $ \wh{w}_i $'s attains non-vanishing correlation with \emph{some} linear combination of the signals $ w_1^*, \cdots, w_p^* $. %, or equivalently, the column span of $ \wh{W} $ (asymptotically) does not lie inside the orthogonal complement of the column span of $ W^* $.  
Other notions of weak recovery for multi-index models have been considered in the literature. 
For instance, \citet{troiani2024fundamental}  also discuss the weak recovery of the \emph{whole} subspace $\cW$, replacing the $\max_{v\in\bbS^{p-1}}$ in \Cref{eq:ournotion} with $\min_{v\in\bbS^{p-1}}$. 
This requirement is clearly stronger than \Cref{def:weak_rec} since \emph{every} vector in $\cW$ needs to be weakly recovered by some $ \wh{w}_i $ in $ \wh{W} $. 
In this work, we focus exclusively on \Cref{def:weak_rec}. 

\paragraph{Spectral estimators.}

For a %ny fixed 
preprocessing function $\cT:\bbR\to \bbR$, let  
$ z_i \coloneqq \cT(y_i)$ for $i\in[n]$ and $Z \coloneqq \diag \matrix{ z_1, & \cdots, & z_n } \in \bbR^{n\times n}$.
Furthermore, we define the matrix $D \equiv D_n \in \bbR^{d\times d}$ as
\begin{align}
    D_n = \frac{1}{n}A^\top Z A = \frac{1}{n}\sum_{i=1}^{n}z_ia_ia_i^\top. \label{eqn:D}
\end{align}
The spectral estimator then outputs the eigenvectors corresponding to the $p$ largest eigenvalues of $D$, i.e., $ 
\matrix{ v_1^D, & \cdots, & v_p^D } $.

We now make a simplifying assumption on $ W^* $ without loss of generality. 
Note that, for any $ W^* $ subject to \Cref{asmp:W*}, by orthogonal invariance of the random design matrix $A$, the law of %the stochastic process 
$ \brace{ \frac{\abs{\inprod{v_i^D}{W^* v}}}{\normtwo{W^* v}} : i\in[p], v\in\bbS^{d-1}} $ remains unchanged under the rotation mapping $ W^* \mapsto W^* O $ for any orthogonal matrix $ O \in \bbR^{p\times p} $. 
Therefore, for convenience, we take the unique rotation $O$ that, for all $i$, maps $ w_i^* $ to $\sum_{j = 1}^i c_{i,j} e_j$, with $\sum_{j = 1}^i c_{i,j}^2 = 1$, and  study $ \brace{ \abs{\inprod{v_i^D}{e_j}} : i,j\in[p] } $, i.e., the overlap of the eigenvectors with the basis $\{e_1, \cdots, e_p\}$ spanned by the rotated signals. We note that, after this rotation,  
%For each $ i\in[p] $, we take $ w_i^* \in \bbS^{d-1} $ of the form: 
%\begin{align}
%    w_i^* &= \sum_{j = 1}^i c_{i,j} e_j , \qquad 
%    \textnormal{where } \sum_{j = 1}^i c_{i,j}^2 = 1 . \notag 
%\end{align}
%In particular, 
$ W^*$ %$ = \matrix{w_1^*, & \cdots, & w_p^*} $ 
can be written as 
\begin{align}
    W^* &= \matrix{\wt{W}^* \\ 0_{(d-p)\times p}} \in \bbR^{d\times p},\qquad     \wt{W}^* = \matrix{
        c_{1,1} & c_{2,1} & \cdots & c_{p,1} \\
                & c_{2,2} & \cdots & c_{p,2} \\
                &         & \ddots & \vdots  \\
                &         &        & c_{p,p}
    } \in \bbR^{p\times p} . \label{eqn:W*} 
\end{align}
%for an upper triangular matrix $ \wt{W}^* $ given by 
%\begin{align}
%    \wt{W}^* &= \matrix{
%        c_{1,1} & c_{2,1} & \cdots & c_{p,1} \\
%                & c_{2,2} & \cdots & c_{p,2} \\
%                &         & \ddots & \vdots  \\
%                &         &        & c_{p,p}
%    } \in \bbR^{p\times p} . \notag 
%\end{align}
We further
define the random variables 
\begin{align}
    (s, \eps) \sim \cN(0_p, I_p) \ot P_\eps , \quad 
    y = q((\wt{W}^*)^\top s, \eps) , \quad 
    z = \cT(y) , \label{eqn:RV}
\end{align}
and make the following assumptions on the preprocessing function $\cT$. % satisfying the following condition. 

\begin{enumerate}[label=(A\arabic*)]
\setcounter{enumi}{\value{asmpctr}}
    \item \label[asmp]{asmp:T} $\cT$ is bounded and $\prob{z = 0}<1$. 
\setcounter{asmpctr}{\value{enumi}}
\end{enumerate}

\section{Main results}
\label{sec:results}

Consider the model \Cref{eq:multi} under \Cref{asmp:A,asmp:W*,asmp:proportional,asmp:noise}. 
Our first result accurately locates, in the high-dimensional limit, the $p$ largest eigenvalues of the matrix $D$ in \Cref{eqn:D} for any preprocessing function $\cT$ subject to \Cref{asmp:T}, thereby unveiling a spectral phase transition phenomenon. 
To present it, a sequence of definitions is needed. 

Let $ \sT $ be the collection of preprocessing functions $\cT$ subject to \Cref{asmp:T}. 
For any $ \cT \in \sT $, let $ z = \cT(y) $ as in \Cref{eqn:RV} and let 
$\tau = \inf\{c:\prob{z\leq c} = 1\}$ 
be the right edge of the support of $z$ (by \Cref{asmp:T}, $ \tau < \infty $). 
Define $\psi_\delta \colon ]\tau,\infty[ \to \bbR$ and its point of minimum as
\begin{equation}
    \psidelta(\lambda) \coloneqq \lambda\paren{\frac{1}{\delta}+\expt{\frac{z}{\lambda-z}}}, \qquad \lambdabardelta \coloneqq \argmin_{\lambda>\tau}\psidelta(\lambda). \notag 
\end{equation}
Note that $\psi_\delta$ is convex and, thus, its minimum is unique. % at 
%$$$$
Finally, define $\zetadelta\colon \bbR \to \bbR$ and $ R^\infty \colon ]\tau,\infty[ \to \bbR^{p\times p} $ as 
\begin{align}
    \zetadelta(\lambda) \coloneqq \psidelta(\max\{\lambdabardelta, \lambda\}), \qquad R^\infty(\alpha) \coloneqq \expt{\frac{\alpha ss^\top z}{\alpha-z}} ,
    \label{eqn:zeta_R}
\end{align}
where $s, z$ are jointly distributed according to \Cref{eqn:RV}. %, and denote $\lambda_i^\infty (\alpha) = \lambda_i(R^\infty(\alpha))$
%the $i$-th largest eigenvalue of $ R^\infty(\alpha) $. 

\begin{theorem}\label{thm:eigvalconv}
    Let $ \cT \colon \bbR \to \bbR $ be a preprocessing function subject to \Cref{asmp:T}, and let $ D\in\bbR^{d\times d} $ be defined in \Cref{eqn:D}.
    %    Let $ \cT \colon \bbR \to \bbR $ be a preprocessing function subject to \Cref{asmp:T}, and let $ D\in\bbR^{d\times d} $ be defined in \Cref{eqn:D}. Denote by
    Let $\alpha_1\geq \dots\ \geq \alpha_j > \tau$ (for some $j\in[p]$) be all the solutions to 
    \begin{equation}\label{eq:master_eq2}
    \det\paren{\zetadelta(\alpha)I-R^\infty(\alpha)}=0.    
    \end{equation}
    %to the equation
    %    \begin{equation}
    %        \det\paren{\zetadelta(\alpha)I_p-R^\infty(\alpha)}=0.
    %    \end{equation}
    Then, for the top $j$ eigenvalues of $D$, it holds that
    \begin{equation}\label{eq:outsidebulkeigconv}
        \lambda_1^D,\dots,\lambda_j^D \asconv \zetadelta(\alpha_1), \dots, \zetadelta(\alpha_j),
    \end{equation}
    and for the remaining $p-j$ eigenvalues, it holds that
    $$\lambda_{j+1}^D,\dots,\lambda_p^D \asconv \zetadelta(\lambdabardelta).$$ 
\end{theorem}
In words, \Cref{thm:eigvalconv} shows a phase transition for the $j$ largest eigenvalues of $D$ as the (normalized) sample complexity $\delta$ varies. In fact, by the definition of $ \zeta_\delta(\cdot) $ in \Cref{eqn:zeta_R}, \Cref{eq:outsidebulkeigconv} implies that, for any $ k\in[j] $, if $ \alpha_k \le \lambdabardelta $, then $ \lambda_i^D $ asymptotically coincides with $\zetadelta(\lambdabardelta)$, which corresponds to the right edge of the bulk of the spectrum of $D$; conversely, if $ \alpha_k > \lambdabardelta $, then the asymptotic value of $ \lambda_i^D $ is strictly larger than the right edge, thereby forming a spectral outlier. 
This phenomenon mirrors the classical BBP phase transition for the spiked covariance model \cite{BBAP}. 

% The result on the eigenvalues of $D$ is accompanied by one on its eigenvectors, as stated below. 

We note that \eqref{eq:master_eq2} has at most $p$ solutions in $]\tau,+\infty[$. Furthermore, if $\cT$ satisfies 
    \begin{equation}\label{assmp:psol}
        \inf_{\norm{2}{x}=1}\lim_{\alpha\to\tau^+}\expt{\frac{\alpha z\inprod{s}{x}^2}{\alpha-z}} = +\infty,
    \end{equation}
     we are guaranteed that \eqref{eq:master_eq2} has exactly $p$ solutions. Both statements are proved in \Cref{prop:exactlyp} deferred to \Cref{app:numsol}. %as proved in \Cref{subsec:asymptotic behavior}. 
     The condition in \eqref{assmp:psol} provides a natural generalization of the assumption made for $p=1$ in previous work, see Equation (82) in \cite{mondelli-montanari-2018-fundamental}.

%\begin{remark}\label{rem:psolutions}
%    In the statement of \Cref{thm:eigvalconv}, it is implicitly said that  \eqref{eq:master_eq2} has at most $p$ solutions in $]\tau,+\infty[$. However, if $\cT$ satisfies 
%    \begin{equation}\label{assmp:psol}
%        \inf_{\norm{2}{x}=1}\lim_{\alpha\to\tau^+}\expt{\frac{\alpha z\inprod{s}{x}^2}{\alpha-z}} = +\infty,
%    \end{equation}
%     we are guaranteed that the \eqref{eq:master_eq2} has exactly $p$ solutions, as proved in \Cref{subsec:asymptotic behavior}. Assumption \eqref{assmp:psol} would be a natural generalization of the assumption made for 1-dimensional case as in \cite[(82)]{mondelli-montanari-2018-fundamental}.
%\end{remark}

Our second result characterizes the asymptotic performance, in terms of overlaps, of the eigenvectors corresponding to the spectral outliers of $D$. 

\begin{theorem}\label{thm:main}
    In the setting of \Cref{thm:eigvalconv}, let $\alpha_k = \alpha_{k+1} = \cdots = \alpha_{k+m-1} $ be solutions to \Cref{eq:master_eq2} of multiplicity $m$, i.e., $ \alpha_{k-1} > \alpha_k > \alpha_{k+m-2} $ ($ k\ge2, k+m-2\le j $), and let $E_k^\infty\subset\bbR^p$ be the $m$-dimensional eigenspace of $R^\infty(\alpha_k)$.
    %If $\alpha_k>\lambdabardelta$, then %it holds 
    If $\lambda_i^D \asconv \zetadelta(\alpha_k)>\zetadelta(\lambdabardelta)$ ($i\in \{k, \ldots, k+m-1\}$), then
    \begin{equation}\label{eq:liminfconv}
        \max_{l\in[p]}\liminf_{d\to\infty}\sum_{i=k}^{k+m-1}\abs{\inprod{v_i^D}{e_l^{(d)}}}^2>0.
    \end{equation}
    More precisely, under the additional assumption that either $m=1$ or the eigenspace $E_k^\infty$ stays invariant in a neighbourhood of $\alpha_k$, for any $ l\in[p] $,    
    \begin{align}
        \sum_{i=k}^{k+m-1}\abs{\inprod{v_i^D}{e_l^{(d)}}}^2\asconv \frac{\zetadelta'(\alpha_k) \sum_{i=k}^{k+m-1}\abs{\inprod{h_i^\infty}{e_l^{(p)}}}^2}{\zetadelta'(\alpha_k)+{h_k^\infty}^\top \frac{d}{d\alpha}R^\infty(\alpha_k)h_k^\infty},
        \label{eq:liminfconv_mult}
    \end{align}
    where $\brace{h_i^\infty : k \le i\le k+m-1}\subset\bbR^p$ is an orthonormal basis of $E_k^\infty$ and $\frac{d}{d\alpha}R^\infty(\cdot)$ denotes the entry-wise derivative of the matrix 
    $R^\infty(\cdot)$. %In all remaining cases ($i\in \{k, \ldots, k+m-1\}$ and $\alpha_k \leq \lambdabardelta$ or $i>j$), then
    Conversely, for all $i$ s.t.\ $\lambda_i^D \asconv \zetadelta(\lambdabardelta)$, then
\begin{equation}\label{eq:liminfconv-conv}
        \max_{l\in[p]}\lim_{d\to\infty}\abs{\inprod{v_i^D}{e_l^{(d)}}}^2=0.
    \end{equation}

\end{theorem}

In words, \Cref{thm:main} shows that, if $ \lambda_k^D $ is an outlier (which happens when $ \alpha_k > \lambdabardelta$ by \Cref{thm:eigvalconv}), the corresponding eigenvectors $ V_k = \matrix{v_k^D, & \cdots, & v_{k+m-1}^D} \in \bbR^{d\times m} $ weakly recover $ \cW=\spn\brace{e_1^{(d)}, \cdots, e_p^{(d)}} $ and \Cref{eq:liminfconv_mult} expresses the asymptotic (squared) overlap $ \normtwo{V_k^\top e_l^{(d)}}^2 $ in terms of $p$-dimensional equations. Conversely, if $ \lambda_k^D $ converges to the right edge of the bulk (which happens if either $ \alpha_k \le \lambdabardelta$ or $k$ exceeds the number of solutions of \Cref{eq:master_eq2}), then the overlap vanishes. 

Let us further elaborate on the invariance of the eigenspace $E_k^\infty$ required for \Cref{eq:liminfconv_mult} to hold. Specifically, if $E_k^\infty$ is an eigenspace of $R^\infty(\alpha_k)$, then it is also an eigenspace of $R^\infty(\alpha_k+\Delta)$ for small enough $\Delta$. This condition ensures that the denominator on the RHS of \Cref{eq:liminfconv_mult} is independent of the choice of the vector from the eigenspace. If this was not the case, the norm of the overlap vectors, i.e. $\sum_{j=1}^{p}\abs{\inprod{v_i^D}{e_j}} ^2$, may not have an asymptotic limit, making analysis hard. The condition appears to be a proof artifact, and it is not even clear how to construct a case with eigenvalue multiplicity that violates the invariance condition. For example, the assumption is satisfied by models in which the eigenvalue multiplicity arises from the permutation-invariance of the link function $q$, see \Cref{prop:invlink} in \Cref{app:invariance}.

Taken together, \Cref{thm:eigvalconv,thm:main} generalize earlier work by \cite{mondelli-montanari-2018-fundamental} on the single-index model and by \cite{mixed-zmv-arxiv} on mixtures of single-index models with independent signals. The independence of signals is crucial to the approach in \cite{mixed-zmv-arxiv} which decomposes $D_n$ in \eqref{eqn:D} as the free sum of matrices each corresponding to one of the signals. To circumvent this difficulty, we pursue a $p$-dimensional analog of the strategy in \cite{mondelli-montanari-2018-fundamental}, with additional adjustments. Namely, all results for eigenvalues and overlaps in \cite{mondelli-montanari-2018-fundamental} (Eq.\ (95), (96), (97), and (99) therein) come from Eq.\ (94), for which there is no direct $p$-dimensional analogue. Our work identifies appropriate alternatives in the form of \Cref{eq:defLtilde} and \Cref{eq:eigenvec-body}. We then show that the eigenvalues of $D_n$ solve a fixed point equation (see \Cref{eq:eigvalrec-body}) and that the eigenvectors are related to a $p\times p$ matrix (see \Cref{eq:eigenvec-body}). Finally, studying the limiting behavior (as $n,d\to\infty$) of the $p$-dimensional objects gives the claimed asymptotic characterizations. For $p>1$, eigenvalue multiplicities complicate the analysis of the limits of eigenvalue derivatives and associated eigenspaces, which are needed for characterizing asymptotic overlaps. We handle such complications via tools from perturbation theory. The proof is outlined in \Cref{sec:pftec}, with several auxiliary results deferred to \Cref{sec:auxpf,app:asymptotbehav,app:asymptotbehav-eigvec}.

Our result on overlaps in \Cref{thm:main} concerns the basis vectors. 
However, unless the signals have vanishing correlation, one needs additional side information to assemble the overlaps with the basis into overlaps %it is unclear how 
%to obtain asymptotic overlaps 
with the signals $ (w_i^*)_{i\in [p]} $, due to the sign ambiguity of eigenvectors. % ($ v $ is an eigenvector if and only if $ -v $ is one). 

Finally, our third main result identifies an optimal preprocessing function allowing the spectral estimator to attain the lowest weak recovery threshold as per \Cref{def:weak_rec}. 
For any $\delta>0$, let 
%\begin{align}
%    \sT_\delta &\coloneqq \brace{ \cT \in \sT : \max\brace{ \alpha : \alpha \textnormal{ solves } \Cref{eq:master_eq2} } > \lambdabardelta } \notag 
%\end{align}
% Formally, for any $\cT \in \sT$, let 
% \yihan{Caution: if we don't have converse for overlap, we can't quite define $ \delta_c(\cT) $ this way. 
 \begin{align}
     \delta_c(\cT) &\coloneqq \inf\brace{ \delta > 0 : \max_{l\in[p]} \lim_{d\to\infty} \sum_{i=1}^{p}\abs{\inprod{v_i^D}{e_l^{(d)}}}^2 > 0 } \notag 
 \end{align}
% }
% \begin{align}
%     \delta_c(\cT) &\coloneqq \inf\brace{ \delta > 0 : \max\brace{ \alpha : \alpha \textnormal{ solves } \Cref{eq:master_eq2} } > \lambdabardelta } \notag 
% \end{align}
% \fk{For consistency we should choose the notation for eigenvectors of matrix, $v_i^D$ vs $v_i^D$. I think either are good.}
% be the weak recovery threshold of the spectral estimators constructed from $\cT$ via \Cref{eqn:D}.
% Note that the equation \Cref{eq:master_eq2} and hence its solutions $ \alpha $ depend on $\delta$. 
%be the set of preprocessing functions which achieve weak recovery at a fixed aspect ratio $\delta$. \fk{This sentence no longer relates to the object above it. Instead we should maybe write "
be the recovery threshold for $\cT$, that is, the smallest aspect ratio above which it achieves weak recovery.
Then, define the optimal weak recovery threshold as 
\begin{align}
    % \delta_c &\coloneqq \inf\brace{ \delta \in ]0, \infty[ : \sT_\delta \ne \emptyset } . 
    \delta_c &\coloneqq \inf_{\cT \in \sT} \delta_c(\cT) . \label{eq:deltacdef}
\end{align}
% and denote by $\sT^*$ the set of minimizers for $\delta_c$. 
Finally, for random variables $(s,y)\in \bbR^p \times \bbR$ jointly distributed according to \Cref{eqn:RV}, we use $ p(y \mid s) $ to denote the conditional density of $y$ given $s$. 

\begin{theorem}\label{thm:opt}
    The optimal weak recovery threshold $\delta_c$ equals 
    \begin{align}\label{eq:specweak}
        \delta_c &= \brack{ \max_{u\in\bbS^{p-1}}\int_{\bbR}\frac{\paren{\expt[s]{p(y \mathrel{\vert} s)\cdot(\inprod{s}{u}^2-1)}}^2}{\expt[s]{p(y \mathrel{\vert} s)}}dy }^{-1}
        , 
    \end{align}
     where expectations are intended over $ s\sim\cN(0_p,I_p) $.    
    Furthermore, denoting by $u_c\in\bbS^{p - 1}$ a maximizer in the above expression,  for any $ \delta > \delta_c $, weak recovery is achieved by taking $\cT_\delta^*(y)$ as in 
   \begin{align}\label{eq:opt}
        \cT^*(y) &\coloneqq 1 - \frac{\expt[s]{p(y \mathrel{\vert} s)}}{\expt[s]{p(y \mathrel{\vert} s)\cdot\inprod{s}{u_c}^2}} , \qquad
        \cT_\delta^*(y) \coloneqq \frac{\sqrt{\delta_c} \cdot \taustar(y)}{\sqrt{\delta}-(\sqrt{\delta}-\sqrt{\delta_c})\cdot \taustar(y)} .  
    \end{align}
%    Then, for any $ \delta > \delta_c $, $ \cT_\delta^* \in \sT_\delta $. 
\end{theorem}

In words, \Cref{thm:opt} shows that 
the preprocessing in \Cref{eq:opt} leads to weak recovery of the signal subspace with a sample complexity that is minimal \emph{among all spectral methods}. % obtained from \Cref{eqn:D}. 
As for \Cref{thm:eigvalconv,thm:main}, the result generalizes earlier work  \cite{mondelli-montanari-2018-fundamental,mixed-zmv-arxiv} to the multi-index case with arbitrarily correlated signals. %\fk{Should we put here "to the general multi-index case" as @Marco mentioned in slack, omitting the correlated signals part?} 
The proof of \Cref{thm:opt} is in \Cref{app:pfopt}.

The design of the optimal preprocessing function does not require the knowledge of $ \wt{W}^* $, but only of the link function $q$ and of the limiting sample covariance matrix of the signals $\Sigma:=(\wt{W}^*)^\top \wt{W}^*$. In fact, $ \cT^*(y) $ remains unchanged under $ \wt{W}^* \mapsto O \wt{W}^* $ for any orthogonal matrix $ O\in\bbR^{p\times p} $. 
Indeed, one can verify that under the above mapping, $ \expt[s]{p(y \mid s)} $ remains unchanged and $ u_c \mapsto O u_c $. 
We can then take $ O = ({{}\wt{W}^*}^\top \wt{W}^*)^{-1/2} {{}\wt{W}^*}^\top $ %(which is orthogonal) 
and equivalently define $$y = q\paren{(\wt{W}^*)^\top O^\top s, \eps}
    = q\paren{\brack{(\wt{W}^*)^\top \wt{W}^*}^{1/2} s, \eps}.$$ We note that the link $q$ and the sample covariance $\Sigma$ are assumed to be known e.g.\ when $y$ is the output of a neural network with $p$ neurons in the teacher-student model. %\fk{just to check, do we know $ \Sigma = (\wt{W}^*)^\top \wt{W}^* $ for a neural network with $p$ neurons?}.
If $q$ and $\Sigma$ are unknown, the problem is in general much harder, and a separate set of samples may be used to first estimate $q, \Sigma$ (sample splitting, as in \cite{sawaya2024high}). If $q$ is parameterized by $\theta$ of fixed dimension, then one can obtain $ \theta $ and $ \Sigma$ from the moments of the random variable $y$ in Eq.\ \eqref{eqn:RV} in some cases (e.g., $y=s_1 s_2$ with generic $\Sigma$; $q=\sum_{i=1}^p s_i^2+w$ with $\Sigma=I$ and $w$ of unknown mean and variance). In such cases, $ \theta $ and $ \Sigma$ can be consistently estimated using the empirical moments of the response vector with $o(n)$ samples. 
The same guarantees then continue to hold if one applies the optimal spectral estimator on the remaining $n-o(n)$ samples with the consistent estimate of $q$ and $\Sigma$.

%Therefore, evaluating $ \cT_\delta^*(y) $ only requires the limiting sample covariance matrix of the signals $ \Sigma^{1/2} \coloneqq \lim_{d\to\infty} \brack{ (W^*)^\top W^* }^{1/2} = \brack{(\wt{W}^*)^\top \wt{W}^*}^{1/2} $. 

\begin{remark}\label{rmk:comp}
Let $\cF(M) \coloneq \expt[y]{E(y)ME(y)}$
with $E(y)\coloneq \expt[s]{ss^\top - I_p\mathrel{\vert} y}$. Then, the threshold identified in Lemma 4.1 of \cite{troiani2024fundamental} %for the special case of orthogonal signals 
is given by 
%    \begin{equation}\label{eq:compweak}
 $   \left(\sup_{\substack{M\succ 0_{p\times p}\\ \normf{M}=1}}\normf{\cF(M)}\right)^{-1}$.  %     
%    \end{equation}
%where %$\cF$ is an operator defined as 
    %$
    %One can readily verify that this expression coincides with \Cref{eq:specweak} when the $\sup$ is achieved by a rank-1 matrix, as well as 
    %for all the examples in Appendix F of \cite{troiani2024fundamental}. We formally verify the equality for the link function $q(s_1, s_2, \eps) = s_1s_2$ considered in the numerical experiments in \Cref{subsec:prod}.     
    One can readily verify that this expression coincides with \Cref{eq:specweak} when the $\sup$ is achieved by a rank-1 matrix. This is the case when the matrices $E(y)$ are simultaneously diagonalizable for all $y$, a condition satisfied by permutation invariant link functions and all examples in Appendix F of \cite{troiani2024fundamental}. 
    % \yihan{Including $ z_1^2 + \sgn(\prod_{i = 1}^p z_i) $ right?}\fk{Yes yes, when writing out you can see that the E(y) is simultaneously diagonalizable so the same conclusion follows.}
    The proof is outlined in \Cref{subsec:equivTroiani}.
    %We formally verify the equality for the link function $q(s_1, s_2, \eps) = s_1s_2$ considered in the numerical experiments in \Cref{subsec:prod}. 
\end{remark}
Note that the simultaneous diagonalizability of $E(y)$ also implies the simultaneous diagonalizability of $R^\infty(\alpha)$. 
Namely,
\begin{align*}
    R^\infty(\alpha) &= \expt{\frac{\alpha ss^\top z}{\alpha-z}}%= \expt[y]{\expt[s]{\frac{\alpha ss^\top \cT(y)}{\alpha-\cT(y)}\ \Big|\ y}} \\ & 
    = \expt[y]{\frac{\alpha\cT(y)}{\alpha-\cT(y)}\expt[s]{ss^\top\mathrel{\vert} y}}= \expt[y]{\frac{\alpha\cT(y)}{\alpha-\cT(y)}(E(y)+I_p)}. 
\end{align*}
Thus, if $(E(y)+I_p)$ has the same eigenvectors for every $y$, so does $R^\infty(\alpha)$ for every $\alpha$. 
This implies that the invariance condition on the eigenspace is satisfied, and \Cref{eq:liminfconv_mult} holds for all eigenvectors with eigenvalues outside the bulk.

\paragraph{Numerical results.}

%\label{sec:experiments}

%We complement our theoretical results with numerical simulations. 
In \Cref{fig:mpr,fig:prod}, we consider instances of %the model 
\Cref{eq:multi} with $d = 1500$ and $p = 2$, and we compute overlaps between %an eigenvector 
$ v_i^D $ and %a parameter vector 
$ w_j^* $ ($ 1 \le i,j\le 2 $). 
Each data point %(labeled by $ \times $ (and also by $ \circ $ in \Cref{fig:prod}) 
is obtained by averaging over $10$ i.i.d.\ trials, and error bars are reported at $1$ standard deviation. 
The corresponding theoretical predictions are plotted using solid curves. 
Both numerical and theoretical values of overlaps are plotted as a function of the aspect ratio $\delta$ and captioned `num' and `thy' respectively. 
\begin{wrapfigure}{r}{0.30\textwidth}
    \centering
    \includegraphics[width=0.30\textwidth]{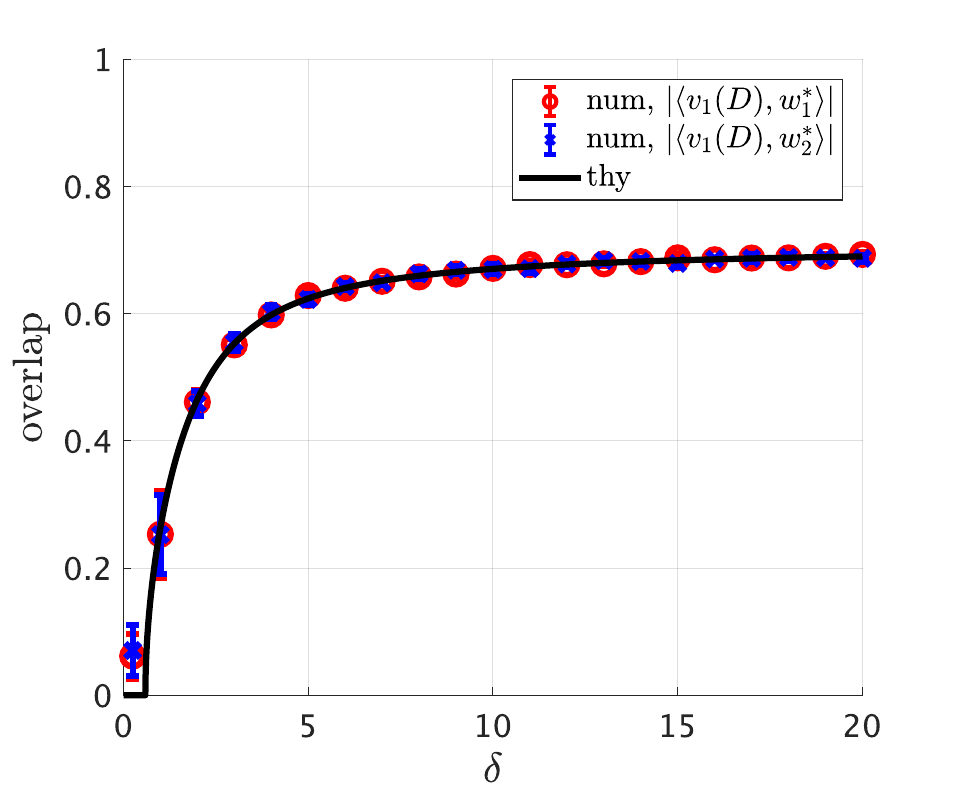}
    \caption{$q(s_1, s_2, \eps) = s_1s_2$. %The signals $ w_1^*, w_2^* $ are asymptotically orthogonal and the optimal preprocessing function in \Cref{eqn:T_prod} is used. The o
    Overlaps $ \abs{\inprod{v_1^D}{w_1^*}} $, $ \abs{\inprod{v_1^D}{w_2^*}} $ and theoretical predictions from \Cref{thm:main} are plotted as a function of $ \delta $.
}        
\label{fig:prod}
\end{wrapfigure}

In \Cref{fig:prod}, the link function is  $ q(s_1, s_2, \eps) = s_1 s_2 $ and the preprocessing function $\cT$ is the optimal one, see \Cref{subsec:prod} for the derivation. The signals are sampled i.i.d.\ from independent isotropic Gaussians. %, i.e., $ (w_1^*, w_2^*) \sim \cN(0_d, I_d / d) \ot \cN(0_d, I_d / d) $. 
    Due to the permutation invariance of the model, for any eigenvector $ v_i^D $, the overlap $ \abs{\inprod{v_i^D}{w_j^*}} $ is asymptotically the same for $ j\in \{1, 2\}$. 
    As there is a single outlier in the model, we only plot $ \abs{\inprod{v_1^D}{w_1^*}} $ and $ \abs{\inprod{v_1^D}{w_2^*}}$. %\fk{The outline got shifted, so should fix this probably as the last thing once the camera-ready is completely stable.} 

In \Cref{fig:mpr}, we consider a two-component mixed phase retrieval model: $ q(\xi_1, \xi_2, \eps) = \abs{\xi_\eps} $, where the mixing variable $ \eps $ is $\brace{1,2}$-valued with $ \prob{\eps = 1}= 0.6 $. 
   The signals are sampled from a bivariate correlated Gaussian with correlation $\rho=0.3$. % i.e., 
%    \begin{align}
%        \matrix{w_{1,k}^* \\ w_{2,k}^*} &\iid \cN\paren{\matrix{0 \\ 0}, \matrix{1 & \rho \\ \rho & 1}} , \qquad \forall k\in[d] , \notag 
%    \end{align}
%    for $ \rho = 0.3 $. 
    The performance of spectral estimators %\Cref{eqn:D} 
    is compared among five choices of preprocessing functions: \emph{(i)}
%    \begin{enumerate}\label{pg:T}
%        \item\label{itm:YCS} 
the quadratic function $ \cT(y) = \min\brace{ y^2, 10 } $ used by \cite{Yi_MLR} for mixed linear regression; % $ q(\xi_1, \xi_2, \eps) = \xi_\eps $; 
\emph{(ii)} 
%
%        \item\label{itm:trim} Tt
the trimming function $ \cT(y) = y^2 \indicator{ y^2 \le 7 } $ considered by \cite{Chen_Candes} which zeros out labels with large magnitude;  
\emph{(iii)} 
%        \item\label{itm:subset} T
the subset function $ \cT(y) = \indicator{ y^2 > 2 } $ considered by \cite{Wang_subset} which quantizes the labels to binary values; \emph{(iv)}  
%
%        \item\label{itm:nonmixed} T
the optimal preprocessing function $ \cT(y) = \min\brace{ 1 - 1/y^2, -10 } $ for the \emph{non-mixed} phase retrieval model  derived by \cite{Luo_Alghamdi_Lu}; and \emph{(v)}
%
%        \item\label{itm:opt} T
the optimal preprocessing function that is guaranteed by our theory to maximize the asymptotic overlap $ \abs{\inprod{v_1^D}{w_1^*}} $, see \Cref{subsec:mp} for details. 
 %       
%    \end{enumerate}
%\end{itemize}
%
As we assume boundedness of $ \cT $ (see \Cref{asmp:T}), we truncate such functions whenever necessary. 
The truncation levels in \emph{(i)} and \emph{(iv)} are taken to be large enough so as not to significantly affect the performance; 
the truncation/quantization levels in \emph{(ii)} and \emph{(iii)} are % 7 $ in \Cref{itm:subset} and $ 2 $ in \Cref{itm:nonmixed} are 
%the optimal ones 
taken from \cite{mondelli-montanari-2018-fundamental}. %, where they are optimized % which are optimized therein 
%within the set $ \brace{0.25i : i\in[40]} $ to yield the smallest empirical weak recovery threshold on the \emph{non-mixed} phase retrieval model. 

\begin{figure}[htbp]
    \centering
    \includegraphics[width=0.9\textwidth]{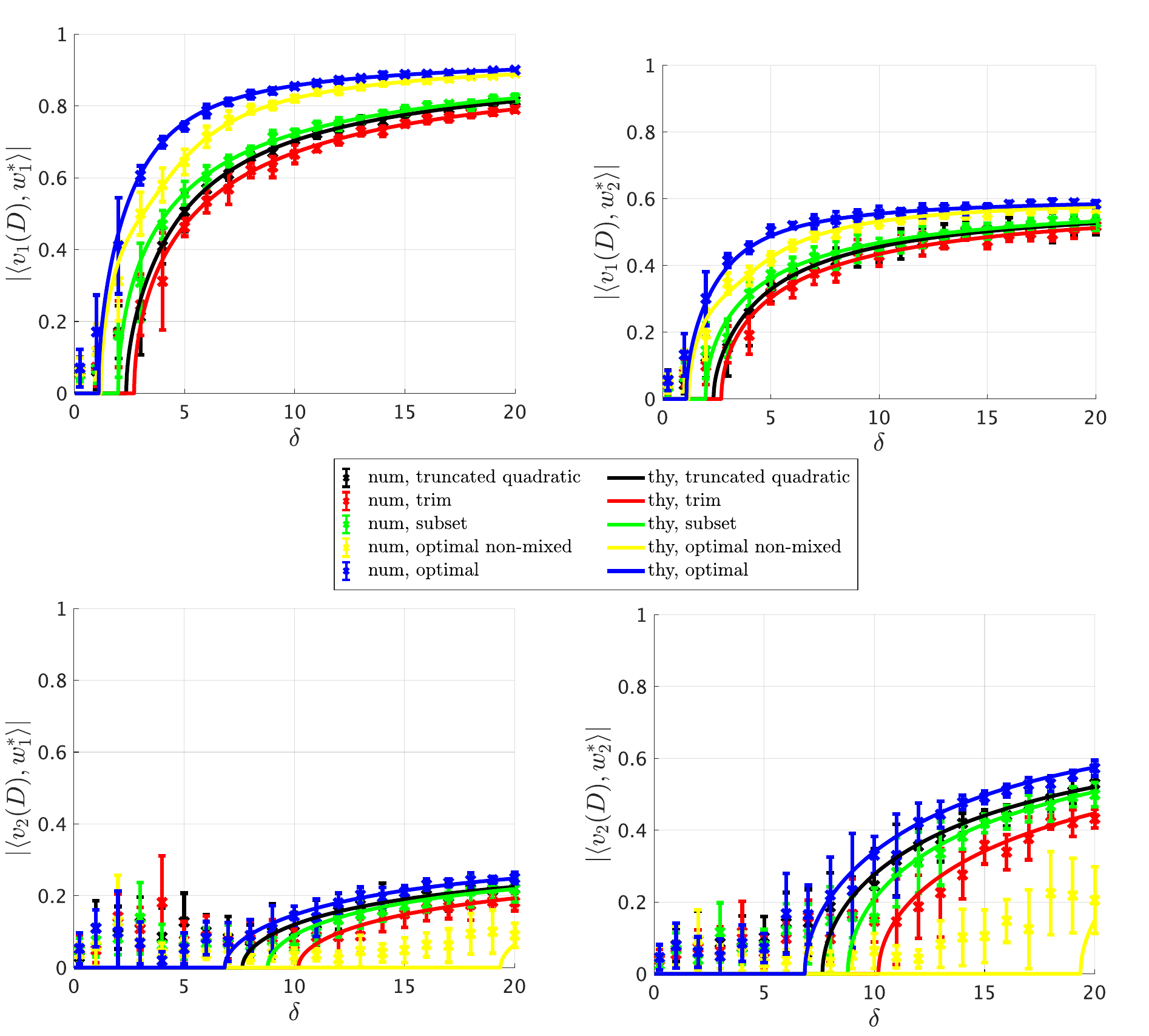}
    \caption{$ q(\xi_1, \xi_2, \eps) = \abs{\xi_\eps} $.
    Overlaps $ \brace{ \abs{\inprod{v_i^D}{w_j^*}} : 1\le i,j\le 2 } $ and theoretical predictions from \Cref{thm:main} are plotted as a function of $\delta$ for five preprocessing functions. 
    Our optimal preprocessing function in \Cref{eqn:T_mpr} attains both the lowest weak recovery threshold (as ensured by \Cref{thm:opt}) and the highest overlap. 
     }
    \label{fig:mpr}
 
\end{figure}

A few remarks %regarding the numerical results i
on \Cref{fig:mpr,fig:prod} are in order. First, 
%\begin{itemize}
%    \item T
the numerical and theoretical results exhibit accurate agreement even for %mildly large dimensions 
$ d = 1500 $, suggesting a rapid rate of convergence. 
%
%    \item 
Second, for mixed phase retrieval, %as can be seen from the expression \Cref{eqn:T_mpr} of the optimal preprocessing function, 
the mixing effect and the correlation between signals %, quantified by $\eta$ and $\rho$ respectively, 
have crucial effects on the design of the optimal preprocessing function. 
    Naively applying the optimal preprocessing function for \emph{non-mixed} phase retrieval results in poor performance of $ v_2^D $ which achieves positive overlap with either signal only if $ \delta > 19 $ (see the bottom row of \Cref{fig:mpr}). 
    In contrast, the weak recovery threshold of $ v_2^D $ resulting from our proposed choice of preprocessing % \Cref{eqn:T_mpr}
    is less than $7$. 
%
%    \item 
    % We remind the readers of the notion of optimality of our preprocessing function. 
 Third,  %  the functions 
while our choice of preprocessing is designed to minimize %\Cref{eqn:T_prod,eqn:T_mpr} focuses on minimizing %specialized from \Cref{thm:opt} are derived to minimize 
the weak recovery threshold of $ v_1^D $, 
 %   However, 
% as noted in \Cref{fig:mpr}, the performance given by our preprocessing function %\Cref{eqn:T_mpr} 
the performance is also competitive in terms of %the value of the 
overlaps, outperforming all alternatives.  %
%    Indeed, it attains the highest value among other choices under comparison. 
%    Understanding refined notions of optimality for our proposed functions is left for future investigation. 
%\end{itemize}

\section{Proof techniques}\label{sec:pftec}

\paragraph{Equivalent spectral characterization.}
%\label{subsec:general matrix}

Let us write 
$a_i = \matrix{ s_i& u_i }^\top$, with $s_i\iid \cN(0,I_p)$ and $u_i\iid \cN(0,I_{d-p})$. Thus, as $w_i^* \in \spn\{e_1^{(d)},\dots,e_p^{(d)}\}$ (see \Cref{eqn:W*}), $y_i$ only depends on $s_i$ and, more precisely, %we have that % Namely,
%\begin{align}
    $y_i = %q\paren{ \inprod{a_i}{w^*_1}, \cdots, \inprod{a_i}{w^*_p}, \eps_i } = 
    q\paren{ \inprod{s_i}{w^*_1}, \cdots, \inprod{s_i}{w^*_p}, \eps_i }$, 
with the abuse of notation that the $s_i$'s are now vectors in $\bbR^d$ with the last $d-p$ coordinates set to $0$.
Extending the notation to matrices, we have that $A  = \matrix{S& U},$
where $ S \coloneqq \matrix{ s_1 & \cdots & s_n }^\top \in \bbR^{n\times p} $ and $U \coloneqq \matrix{ u_1 & \cdots & u_n }^\top \in \bbR^{n\times (d-p)}$. Thus, we can re-write the spectral matrix $D_n$ in \eqref{eqn:D} as 
%Using this notation
\begin{align}\label{eq:Ddef}
    D_n = \frac{1}{n}A^\top Z A  &= \frac{1}{n} \matrix{ S^\top \\ U^\top} Z  \matrix{ S & U}
                                 = \frac{1}{n} \matrix{ S^\top Z S & S^\top Z U\\
                                                         U^\top Z S & U^\top Z U}=  \matrix{ a & q^\top\\ 
                                             q & P}, 
\end{align}
where $a\coloneqq \frac{1}{n}S^\top Z S \in \bbR^{p\times p}$, $q \coloneqq \frac{1}{n}U^\top Z S\in \bbR^{(d-p)\times p}$ and $P \coloneqq \frac{1}{n }U^\top Z U\in \bbR^{(d-p)\times (d-p)}$.

\begin{SCfigure}[50][h]
    \centering
    \includegraphics[width=0.4\textwidth]{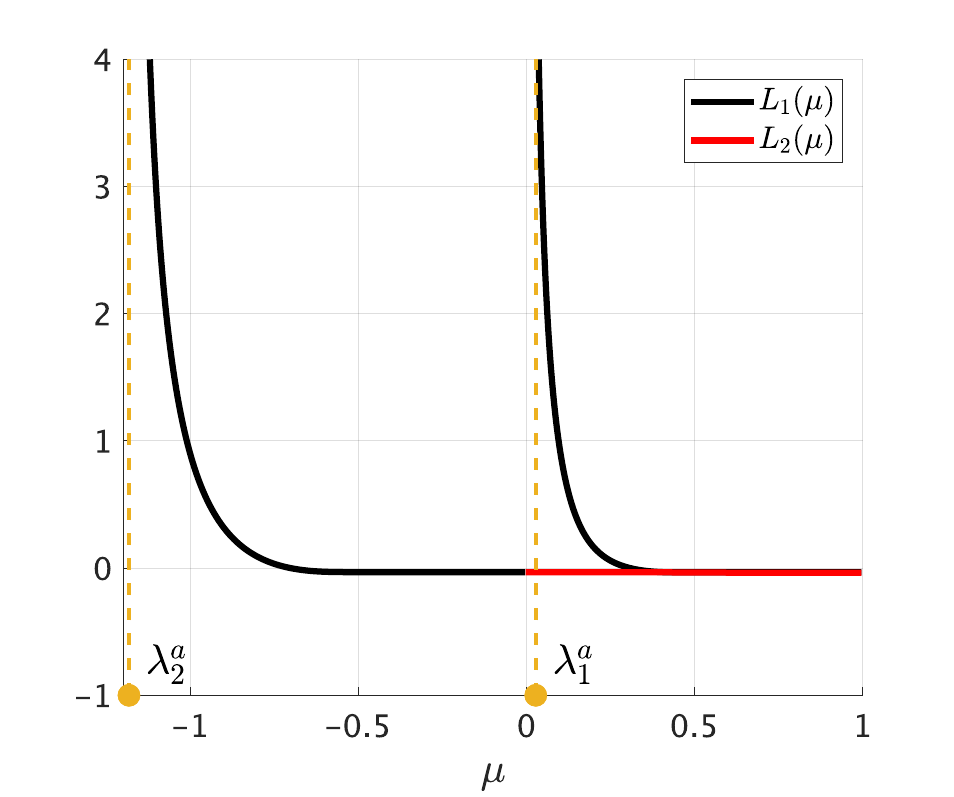}

    \caption{Plot of the functions $ \wt{L}_1, \wt{L}_2 $ (as well as $L_1, L_2$) defined in \Cref{eq:defLtilde} for the model $ y_i = \inprod{a_i}{w_1^*} \inprod{a_i}{w_2^*} $ ($1\le i\le n$), where $ w_1^*, w_2^* \in\bbR^{1500} $ are orthogonal unit vectors and $ n = 5d $. The preprocessing function is taken to be the optimal one given in \Cref{eqn:T_prod}. }
    \label{fig:L}
\end{SCfigure}

We start by working with a generic matrix $D$ of the form in the RHS of \eqref{eq:Ddef},
such that $a, P$ %\in \bbR^{p\times p}$ and $P\in\bbR^{(d-p)\times (d-p)}$ 
are symmetric %matrices, 
and, for all eigenvectors $v_i^a$ of $a$, it holds that $qv_i^a\neq 0$ (this is the case for the matrix $D_n$ in \Cref{eq:Ddef} as showed in \Cref{lemma:eignon} in \Cref{sec:auxpf}). For simplicity, we assume that all eigenvalues of $a$ are different (if any multiplicity exists, notation and exposition can easily be adjusted accordingly).
%We are interested in finding the eigenvalues $\lambda_1^D,\dots ,\lambda_p^D$, and the corresponding eigenvectors $v_1^D,\dots ,v_p^D$. 
%We start with a characterization of the eigenvalues $\lambda_1^D,\dots ,\lambda_p^D$, and note that  %by analyzing eigenvalues.
%\subsubsection{Eigenvalues}\label{subsubsec:deteigval}
%The eigenvalues 
%$\lambda_i^D$ solves % the following equation:
%$$\det\paren{D-\lambda I_{p}} = \det\paren{\matrix{ a-\lambda I_p & q^\top\\
%                q & P-\lambda I_{d-p}}} = 0. $$
Let $L_i:\bbR\setminus\Lambda^a\to \bbR$ (for $i\in [d-p]$) and $\tilde{L}_i: ]\lambda_{i}^a,+\infty[\setminus \Lambda^a\to\bbR$ (for $i\in [p]$) be defined as %\fk{spacing with the graphs and all here looks weird. TODO: fix it once everything else is stable for camera-ready.}
%For $ 1\le i\le d-p $, consider a function $L_i:\bbR\setminus\Lambda^a\to \bbR$ defined as     \begin{equation}\label{eq:defLi}
%        L_i(\mu) \coloneqq \lambda_i(P-q(a-\mu I_p)^{-1}q^\top).
%    \end{equation}
    %where $\lambda_i(\cdot)$ symbolizes the $i$-th largest eigenvalue of the matrix. 
%    Note that in each of the intervals forming the domain, $L_i(\mu)$ is continuous. In fact, the eigenvalues are roots of a polynomial with continuous coefficients, and zeros of such a polynomial are continuous  \cite[Theorem 1]{zedek1965continuity}. Let us for simplicity of exposition assume that all eigenvalues of $a$ are different (if any multiplicity exists, notation and exposition can easily be adjusted accordingly).
%Consider, for $i\in[p]$,

\begin{equation}\label{eq:defLtilde}    
L_i(\mu) \coloneqq \lambda_i(P-q(a-\mu I_p)^{-1}q^\top),\qquad \tilde{L}_i(\mu) \coloneqq
\begin{cases}
L_1(\mu) & \text{if } \lambda_i^a < \mu < \lambda_{i-1}^a, \\
\quad \vdots\\
L_i(\mu) & \text{if } \lambda_1^a < \mu < +\infty, 
\end{cases}
\end{equation}
%where we restrict the domain to $\mu \in ]\lambda_{i}^a,+\infty[\setminus \Lambda^a$. 
see \Cref{fig:L} for a plot. 
The $\tilde{L}_i$'s are piecewise continuous, since the $L_i$'s are continuous in the respective domains. Furthermore, % (eigenvalues are roots of a polynomial with continuous coefficients, and zeros of such a polynomial are continuous  \cite[Theorem 1]{zedek1965continuity}). For each $i$, 
the domain of $\tilde{L}_i$ can be continuously extended to $\Lambda^a$ (\Cref{lemma:samelimits} in \Cref{sec:auxpf}), and  $\tilde{L}_i(\mu)$ is non-increasing in this domain (\Cref{lem:eiglimits} in \Cref{sec:auxpf}).

We use the functions in \Cref{eq:defLtilde} to characterize the top $p$ eigenvalues $(\lambda_i^D)_{i\in [p]}$ of the matrix $D$. Specifically, \Cref{prop:eigvalrec} in \Cref{sec:auxpf} shows that, for $i\in[p]$, $\lambda_i^D$ is the unique solution to
    \begin{equation}\label{eq:eigvalrec-body}
        \tilde{L}_i(\mu) = \mu,\qquad \mbox{for }\mu\in ]\lambda_i^a,\infty[.
    \end{equation}  
%While the strategy of the argument is similar to the case $p=1$ treated in \cite[Proposition 3.1]{lu2020phase}, there is a bit more nuance in the multidimensional case, as the matrix $D$  can have eigenvalues matching those of the matrix $a$, which requires additional care. The proof is deferred to \Cref{sec:auxpf}.
%\subsubsection{Eigenvectors}\label{subsubsec:deteigvec}
Moving to eigenvectors, we write $v_i^D\coloneqq\matrix{h_i\\ g_i}$ ($h_i\in \bbR^{p}$, $g_i\in \bbR^{d-p}$) %It holds that
%\begin{equation}\label{eq:generaleigvec}
%    D v_i^D = \matrix{ a & q^\top\\
%                    q & P}\matrix{h_i\\ g_i} = \lambda^D_i \matrix{h_i\\ g_i}.
%\end{equation} 
and characterize $h_i$ in terms of the function $R(\lambda)\coloneq a - q^\top(P-\lambda I_{d-p})^{-1}q$ defined for $\lambda\in]\lambda_1^p,+\infty[$. 
%\begin{proposition}\label{prop:eigenvec}
Specifically, \Cref{prop:eigenvec} in \Cref{sec:auxpf} shows that, for all $i$ s.t.\ $\lambda_i^D>\lambda_1^P$,
    \begin{equation}\label{eq:eigenvec-body}
    h_i = \frac{\tilde{h}_i}{\sqrt{1-\tilde{h}_i^\top\frac{d}{d\lambda}R(\lambda_i^D)\tilde{h}_i}},
    \end{equation}
    where $\tilde{h}_i$ is the %unit norm 
    eigenvector of $R(\lambda_i^D)$ and $\frac{d}{d\lambda}R(\lambda)$ is the entry-wise derivative of %the matrix 
    $R(\lambda)$. Finally, \Cref{lemma:matrixRchar} in \Cref{sec:auxpf} %uses the matrix determinant lemma to 
    expresses the entries of $R(\lambda)$ and $\frac{d}{d\lambda}R(\lambda)$ in terms of %the auxiliary functions
    \begin{equation}\label{eq:auxfun}
    \cLi(\mu)\coloneqq \lambda_1(P+\mu q_iq_i^\top),\qquad  \cL_{i,j}(\mu)\coloneqq \lambda_1(P+\mu (q_i+q_j)(q_i+q_j)^\top),
    \end{equation}
where $q_i$ is the $i$-th column of $q$.

In summary, the eigenvalues of $D$ are given by the fixed points of $\tilde L_i$ (see \Cref{eq:eigvalrec-body}); the eigenvectors are computed in \Cref{eq:eigenvec-body} via the $p\times p$ matrices $R(\lambda)$ and $\frac{d}{d\lambda}R(\lambda)$, which are in turn related to $\cLi, \cL_{i,j}$ (see \Cref{eq:auxfun}). The rationale for this equivalent spectral characterization is that $\tilde L_i, \cLi, \cL_{i,j}$ are all low-rank perturbations of $P$ (the first is rank-$p$, the other two are rank-$1$). Thus, when considering the proportional asymptotic regime ($n, d\to\infty$ with $n/d$ and $p$ fixed), we can leverage the wide random matrix theoretic literature on low-rank perturbations and, specifically, results from \cite{bai-yao-2012}. As a final, more technical note, the diagonal entries of $R(\lambda)$ and $\frac{d}{d\lambda}R(\lambda)$ are given by
\begin{equation}\label{eq:matrixRminusonediag-body}
R(\lambda)_{i,i} = a_{i,i}+\frac{1}{\cLi^{-1}(\lambda)},\qquad 
        \frac{d}{d\lambda}R(\lambda)_{i,i} = -\frac{1}{\paren{\cLi^{-1}(\lambda)}^2\cLi'(\cLi^{-1}(\lambda))},      
\end{equation}
which is reminiscent of the expressions for the single-index case, see \cite[Lemma 3.1]{lu2020phase}. In contrast, the off-diagonal entries are more cumbersome (see \Cref{eq:matrixRminusoneoffdiag}), and they are linked to the multi-index nature of the model.

\paragraph{Sketch of the proof of \Cref{thm:eigvalconv}.}
    We start with some manipulations: \eqref{eq:master_eq2} is equivalent to $\prod_{i=1}^p( \zetadelta(\alpha)-\lambda_i^\infty(\alpha))=0$, where $\lambda_i^\infty (\alpha)$
is the $i$-th largest eigenvalue of $ R^\infty(\alpha) $. Thus, the assumption of the theorem implies that $\alpha_i$ is the unique solution to 
    $\zetadelta(\alpha) - \lambda_i^\infty(\alpha)=0$.
    Let $\mu_i^*\coloneq {\lambda_i^\infty}(\alpha_i) $ and $\tildeLiinfty(\mu)\coloneq\zeta_\delta\paren{(\lambda^\infty_i)^{-1}(\mu)}$. Then,  $\mu_i^*$ is the unique solution to 
    $\tildeLiinfty(\mu)-\mu=0$.
    
    Next, recall that $\lambda_i^D$ is the unique solution to \eqref{eq:eigvalrec-body}. \Cref{prop:tildeLiconv} in \Cref{app:asymptotbehav} establishes the limit of $\tildeLi$, by relying on the classical results by \cite{bai-yao-2012} adapted to the analysis of single-index models in \cite{lu2020phase,mondelli-montanari-2018-fundamental}. Formally, this gives that 
    \begin{equation}\label{eq:asconvtildeli-body}
        \tildeLi(\mu) - \mu \asconv\tildeLiinfty(\mu) - \mu.
    \end{equation}
    As $\mu_i^*$ is the unique solution to 
    $\tildeLiinfty(\mu)-\mu=0$, we conclude that 
    \begin{equation}\label{eq:fplast-body}
\lambda_i^D\asconv\tildeLiinfty(\mu_i^*).
    \end{equation}
    Substituting $\mu_i^* = {\lambda_i^\infty}(\alpha_i)$ in \Cref{eq:fplast-body} gives the desired result in \eqref{eq:outsidebulkeigconv} for the top-$j$ eigenvalues.
    
The passages to show the claim for the remaining $p-j$ eigenvalues are similar. As \eqref{eq:master_eq2} has only $j$ solutions by assumption, $\zetadelta(\alpha) - \lambda_i^\infty(\alpha) = \tildeLiinfty(\lambda_i^\infty(\alpha))-\lambda_i^\infty(\alpha)=0$
    has no solutions for $\alpha>\tau$ and $i> j$. Thus, $\tildeLiinfty(\mu)-\mu=0$
    has no solutions for $\mu \in ]\lambda_i^{a^\infty},t_i^\infty[$, with $t_i^\infty\coloneq \lim_{\alpha\to\tau} \lambda^\infty_i(\alpha)$. This in turn implies that, for large enough $n$, the solution to 
    $\tildeLi(\mu) - \mu =0 $
    must be for $\mu>t_i^\infty$.
    
    At this point, we use again \Cref{prop:tildeLiconv} to show that, for $\mu>t_i^\infty$, 
    $\tildeLi(\mu)\asconv \zetadelta(\lambdabardelta),$
    which implies that 
    $\lambda_i^D\asconv \zetadelta(\lambdabardelta)$
    for $i\in\{j+1, \ldots, p\}$, thus proving the desired result. The detailed proof is in \Cref{app:pf2}.
    
\paragraph{Sketch of the proof of \Cref{thm:main}.}
As for the eigenvectors, we start with an asymptotic characterization of  $R(\lambda)$ and $\frac{d}{d\lambda}R(\lambda)$. Specifically, \Cref{thm:matrixRconvergence} in \Cref{app:auxeig} shows that
\begin{equation}\label{eq:matrixRconv-body}
         R(\lambda_k^D)\asconv R^\infty(\alpha_k),\qquad \frac{d}{d\lambda}R(\lambda_k^D)\asconv \frac{1}{\zetadelta'(\alpha_k)}\frac{d}{d\alpha}R^\infty(\alpha_k),
     \end{equation}
     where $\alpha_k$ is the $k$-th largest solution of \Cref{eq:master_eq2} and $\alpha_k>\lambdabardelta$. To obtain \Cref{eq:matrixRconv-body}, the idea is to analyze the functions in \Cref{eq:auxfun}, 
%$$\cLi(\mu)= \lambda_1(P+\mu q_iq_i^\top), \text{ and } \cL_{i,j}(\mu)= \lambda_1(P+\mu (q_i+q_j)(q_i+q_j)^\top),$$
as $n,d\to \infty$. We do so by using results from \cite{bai-yao-2012}, which give
\begin{equation}\label{eq:sketch1}
\cLi(\mu)\asconv \zetadelta \circ {Q_i}^{-1}\circ G(\mu),\qquad \cL_{i,j}(\mu)\asconv \zetadelta \circ {Q_{i,j}}^{-1}\circ G(\mu),
\end{equation}
where $G(\mu) = -1/\mu$, $Q_i(\alpha) \coloneq \expt{\frac{s_i^2z^2}{z-\alpha}}$ and $Q_{i,j}(\alpha) \coloneq \expt{\frac{(s_i+s_j)^2z^2}{z-\alpha}}$. 
As $\alpha_k>\lambdabardelta$, \Cref{thm:eigvalconv} guarantees that $\lambda_k^D\asconv  \zetadelta(\alpha_k)$ and therefore
\begin{equation}\label{eq:sketch2}
\cLi^{-1}(\lambda_k^D)\asconv G \circ Q_i \circ \zetadelta^{-1} \circ \zetadelta(\alpha_k) = G \circ Q_i (\alpha_k).
\end{equation}
By the law of large numbers, $a\asconv a^\infty \coloneq \expt{zss^\top}$. Thus, by combining \Cref{eq:matrixRminusonediag-body,eq:sketch1,eq:sketch2} we conclude
\begin{equation}\label{eq:tempdiagconv-body}
    R(\lambda_k^D)_{i,i}\asconv a^{\infty}_{i,i}-Q_i (\alpha_k)= \expt{s_i^2z} -  \expt{\frac{s_i^2z^2}{z-\alpha_k}} = \expt{\frac{\alpha_k s_i^2z}{\alpha_k-z}} = R^\infty(\alpha_k)_{i,i}.
\end{equation}
Moreover, $\cLi(\mu)$ is differentiable (see \Cref{lemma:matrixRchar} in \Cref{sec:auxpf}), so for its derivative it holds that
$$\cLi'(\mu)\asconv \zetadelta' \circ {Q_i}^{-1}\circ G(\mu) \cdot \paren{Q_i^{-1}}'\circ G(\mu) \cdot G'(\mu).$$
Plugging this into \eqref{eq:matrixRconv-body} we get 
$$\frac{d}{d\lambda}R(\lambda_k^D)_{i,i}\asconv \frac{\frac{d}{d\alpha}(R^\infty(\alpha_k)_{i,i})}{\zetadelta'(\alpha_k)}.$$
By performing similar calculations on the off-diagonal entries of $R(\lambda)$ and $\frac{d}{d\lambda}R(\lambda)$, \Cref{eq:matrixRconv-body} follows (the complete proof is deferred to \Cref{app:auxeig}).

We are now ready to prove \Cref{eq:liminfconv_mult} when there is no multiplicity ($m=1$). Having a multiplicity adds technical complications (and also the assumption on the invariance of the eigenspace $E_k^\infty$), which are handled in \Cref{app:pfthmmain} where we also prove \Cref{eq:liminfconv,eq:liminfconv-conv}. 

%\begin{corollary}\label{cor:mastereigenvectorsmultiplewassumption}
%    Let $\alpha_k$ be the $k$-th solution to the equation
%    $$\det\paren{\zetadelta(\alpha)I-R^\infty(\alpha)}=0,$$
%    with multiplicity j and corresponding eigenspace $E_k^\infty$. If $\alpha_k>\lambdabardelta$, then it holds 
%    \begin{equation}
 %       \max_{l\in[p]}\liminf_{d\to\infty}\sum_{i=k}^{k+j-1}\abs{\inprod{v_i^D}{e_l^{(d)}}}^2>0. \notag 
        % \label{eq:liminfconv}
  %  \end{equation}
   % Also, under the assumption that eigenspace $E_k^\infty$ stays the same in neighbourhood of $\alpha_k$, it holds for any $j\in[p]$ that
    %$$\sum_{i=k}^{k+j-1}\abs{\inprod{v_i^D}{e_l^{(d)}}}^2\asconv \frac{\zetadelta'(\alpha_k) \sum_{i=k}^{k+j-1}\abs{\inprod{h_i^\infty}{e_l^{(p)}}}^2}{\zetadelta'(\alpha_k)+{h_k^\infty}^\top \frac{d}{d\alpha}R^\infty(\alpha_k)h_k^\infty},$$
    %where $h_i^\infty$ form an orthonormal eigenbasis of $E_k^\infty$. 
%\end{corollary}
%\begin{proof}

First, note that the desired overlaps can be expressed as 
\begin{equation}\label{eq:innerp}    
\abs{\inprod{v_k^D}{e_j^{(d)}}}^2 = (e_j^{(p)})^\top h_kh_k^\top (e_j^{(p)}).
\end{equation}
Furthermore, recall that \emph{(i)} $h_k$ is related to the unit norm eigenvector $\tilde{h}_k$ of $R(\lambda_k^D)$ and to $\frac{d}{d\lambda}R(\lambda_k^D)$ via \Cref{eq:eigenvec-body}, and that \emph{(ii)} the limits of $R(\lambda_k^D)$ and to $\frac{d}{d\lambda}R(\lambda_k^D)$ are given by \Cref{eq:matrixRconv-body}. 
    Then, applying the results from \cite[II.1.4]{kato2013perturbation}, %(as proved in this StackExchange \href{https://math.stackexchange.com/questions/4054792/convergence-of-eigenvalues-and-spaces-of-sequence-of-compact-szmmetric-and-posi}{answer}).
    the orthonormal projection to the eigenspace corresponding to the $k$-th eigenvalue also converges, i.e., 
    $\Pi_{h_k} \asconv \Pi_{h^\infty_k},$
    where $\Pi_{h_k} = \frac{h_kh_k^\top}{\norm{2}{h_k}^2} = \tilde{h}_k\tilde{h}_k^\top$ and $\Pi_{h^\infty_k} = \frac{h^\infty_k{h^\infty_k}^\top}{\norm{2}{h^\infty_k}^2}=h^\infty_k{h^\infty_k}^\top$. As a consequence, we have that
    $$\norm{2}{h_k}=\frac{1}{\sqrt{1+\tilde{h}_k^\top\frac{d}{d\lambda}R(\lambda_k^D)\tilde{h}_k}}\asconv \frac{\sqrt{\zetadelta'(\alpha_k)}}{\sqrt{\zetadelta'(\alpha_k)+{h^\infty_k}^\top\frac{d}{d\lambda}R^\infty(\alpha_k)h^\infty_k}}.$$
    Combining these results, we obtain that
    $$h_kh_k^\top\asconv \frac{\zetadelta'(\alpha_k)h^\infty_k {h^\infty_k}^\top}{\zetadelta'(\alpha_k)+{h^\infty_k}^\top\frac{d}{d\lambda}R^\infty(\alpha_k)h^\infty_k},$$
    which, together with \Cref{eq:innerp}, proves the claim.

%\subsection{Optimizing $\cT$}

\section{Concluding remarks} \label{sec:conc}

This paper provides the first asymptotic characterization of spectral estimators for multi-index models: we unveil a phase transition in the top-$p$ eigenvalues of the spectral matrix $D$ in \Cref{eqn:D}, giving a low-dimensional (and simple to check) condition for spikes to emerge from the bulk of the spectrum; the eigenvalue phase transition is associated to the recovery of the subspace spanned by the signals via the corresponding eigenvectors and, under some technical conditions, we prove a precise expression for the asymptotic overlap; finally, we optimize the data preprocessing and identify the spectral estimator that weakly recovers the signal subspace with the smallest sample complexity. 

Spectral methods are commonly used as a warm start for other procedures, often of iterative nature. Thus, our analysis provides the starting point to combine spectral estimators either with a simple linear estimator \cite{mondelli2022optimal,mixed-zmv-arxiv}  or with AMP \cite{Mondelli_Venk}, with the objective of achieving -- at least in absence of statistical-to-computational gaps -- the Bayes-optimal limits of inference as characterized by \citet{Aubin_comm_machine}.

\section*{Acknowledgements}

This work was done when Y.\ Z.\ was at the Institute of Science and Technology Austria. 
Y.\ Z.\ and M.\ M.\ are funded by the European Union (ERC, INF$^2$, project number 101161364). Views and opinions expressed are however those of the author(s) only and do not necessarily reflect those of the European Union or the European Research Council Executive Agency. Neither the European Union nor the granting authority can be held responsible for them. The authors would like to acknowledge (in alphabetical order) discussions with Yatin Dandi, Leonardo Defilippis and Bruno Loureiro concerning their parallel work \cite{defilippis2025optimal}.

\bibliographystyle{alpha} 
\bibliography{ref} 

\newpage

\appendix

\crefalias{section}{appendix}

\section{Formal statements and proofs for the equivalent spectral characterization}\label{sec:auxpf}

We start with two auxiliary results.

\begin{lemma}\label{lemma:fr}
    Under \Cref{asmp:T}, almost surely for all sufficiently large $n$ it holds that $$\rk(ZS) = p.$$
\end{lemma}
\begin{proof}
 Notice first that, for all sufficiently large $n$, there are almost surely at least $p$ elements in the array $\brack{z_1 ,z_2 ,\dots}$ that are non-zero. This follows from \Cref{asmp:T} that $\prob{Z = 0}<1$, as done in the proof of \cite[Proposition 3.2]{lu2020phase}. Now, the $i$-th column of $ZS$ is obtained by scaling a standard $p$ dimensional Gaussian by $z_i$. As the columns of $S$ are almost surely independent and, for sufficiently large $n$, at least $p$ elements $z_i$'s are non-zero, it must be that at least $p$ columns of $ZS$ are linearly independent, which proves the claim. 
\end{proof}

\begin{lemma}\label{lemma:eignon}
    For every eigenvector $v_i^a$, it holds almost surely that 
    $$qv_i^a\neq 0.$$
\end{lemma}
\begin{proof}
    By definition, it holds that $v_i^a\neq 0$. Thus, \Cref{lemma:fr} implies that almost surely $Z S v^a_i \neq 0$.
    Furthermore, the elements of the matrix $U$ are sampled independently from $Z$ and $S$, so we can fix $x_i\coloneqq ZSv^a_i \neq 0$ and conclude 
    $$\prob{U^\top x_i=0} = 0,$$
    for the probability measure associated to the elements of $U$. This is due to the fact that, for $x_i\neq 0$, 
    $$\prob{U^\top x_i=0}=\prob{\forall \,j\in [d-p],\ \inprod{u^j}{x_i} = 0} = \prod_{i=1}^{d-p}\prob{\inprod{u_j}{x_i} = 0}=0,$$
    as each $u_j\in\bbR^n$ is sampled from a multi-variate Gaussian. 
    Lastly, using the union bound, we have %it holds almost surely that
    $$\prob{\exists \,i\in [p], U^\top x_i=0} \leq \sum_{i=1}^p \prob{U^\top x_i=0}=0,$$
    implying  that almost surely $qv^a_i\neq 0$ for every eigenvector of the matrix $a$. 
\end{proof}

For $i$ such that $p+1\leq i\leq d-p$, let us define
\begin{equation}\label{eq:defLtildeother}    
\tilde{L}_i(\mu) \coloneqq
\begin{cases}
L_{i-p}(\mu) & \text{if } \lambda_p^a < \mu < \lambda_{p-1}^a, \\
\quad \vdots\\
L_i(\mu) & \text{if } \lambda_1^a < \mu < +\infty,
\end{cases}
\end{equation}
where we restrict the domain to $\mu \in ]\lambda_{p}^a,+\infty[\setminus \Lambda^a$. This complements the definition of $\tilde{L}_i(\mu)$ for $i\in [p]$ in \Cref{eq:defLtilde}. 
For $1\leq i\leq d-p$, all functions $\tilde{L}_i$ can be continuously extended, as proved below. %the result we formulate in the following lemma.
\begin{lemma}\label{lemma:samelimits}
    For $i \in[p]$ and $j\geq 2$, it holds that
    $$\lim_{\mu_1\to{\lambda^a_i}^+}L_j(\mu_1) = \lim_{\mu_2\to{\lambda^a_i}^-}L_{j-1}(\mu_2).$$
\end{lemma}
\begin{proof}
To simplify exposition, let us denote by $M_\mu\coloneqq P-q(a-\mu I_p)^{-1}q^\top$, as well as by $r=\frac{qv_i^a}{\norm{2}{qv_i^a}}$ which is well defined as $qv_i^a\neq0$.
Using Weyl's inequality (see e.g. \cite[Section 4.3]{horn2012matrix}), one has
$$\lambda_{j}(M_{\mu_1}) \leq \lambda_{j-1}(M_{\mu_2}) + \lambda_2\paren{M_{\mu_1}-M_{\mu_2}}.$$
Note that $M_{\mu_1}-M_{\mu_2} = q\paren{(a-\mu_2 I_p)^{-1}-(a-\mu_1 I_p)^{-1}}q^\top$ and that the SVD decomposition of $(a-\mu I_d)^{-1}$ is
$$(a-\mu I_p)^{-1} = \sum_{k=1}^p \frac{v_k^av_k^{a\top}}{\lambda_k^a - \mu}.$$
Plugging that in, we get
$$(a-\mu_2 I_p)^{-1} - (a-\mu_1 I_p)^{-1} = v_i^av_i^{a\top}\paren{\frac{1}{\lambda_i^a - \mu_2}+\frac{1}{\mu_1-\lambda_i^a}}+\sum_{k\neq i}^p v_k^av_k^{a\top}\paren{\frac{\mu_2-\mu_1}{(\lambda_k^a - \mu_2)(\lambda_k^a - \mu_1)}}.$$ 
Taking the limits $\mu_1\to{\lambda^a_i}^+$ and $\mu_2\to{\lambda^a_i}^-$, one readily obtains that $\lambda_2(M_{\mu_1}-M_{\mu_2}):=\varepsilon_{\mu_2}\to 0$. 
Before continuing, let us denote by $\Pi_{r}\coloneqq rr^\top$ the orthogonal projection to the subspace defined by $r$, and by $\Pi_{{r}^\perp}$ the orthogonal projection to the subspace $r^\perp$. Obviously it holds that
$$\Pir+\Pirperp = I_{d-p}.$$
Thus, 
\begin{align*}
    \lambda_{j-1}(M_{\mu_2}) &=  \lambda_{j-1}\paren{(\Pir+\Pirperp)M_{\mu_2}(\Pir+\Pirperp)}\\
                             &\leq \lambda_{j-1}\paren{\Pirperp M_{\mu_2}\Pirperp}+\lambda_1(\Pir M_{\mu_2}\Pir +\Pirperp M_{\mu_2}\Pir +\Pir M_{\mu_2}\Pirperp),
\end{align*}
where the last line is due to the fact that the matrix $\Pir M_{\mu_2}\Pir +\Pirperp M_{\mu_2}\Pir +\Pir M_{\mu_2}\Pirperp$ is symmetric and through a subsequent application of Weyl's inequality. Let us analyze the eigenvector corresponding to the largest eigenvalue. Namely, for an eigenvector $t = \Pirperp t + \Pir t =: %t_r+t_{r^\perp}$
t_{r^\perp}+t_r$ of the discussed matrix with corresponding eigenvalue $\lambda$, we have that
\begin{align*}
    \Pir M_{\mu_2}\Pir\ t_r +\Pir M_{\mu_2}\Pirperp\ t_{r^\perp} &= \lambda\ t_r,\\
    \Pirperp M_{\mu_2}\Pir\ t_r &= \lambda\ t_{r^\perp}.
\end{align*}
If $t_r$ is the 0 vector, then $\lambda=0$. Otherwise, as $\Pir M_{\mu_2}\Pirperp\ t_{r^\perp}$ has a convergent, finite limit as $\mu_2\to{\lambda^a_i}^-$ and $t_r^\top\Pir M_{\mu_2}\Pir\ t_r\to - \infty$, it must hold that $\lambda \to -\infty$. This gives the following inequality 
$$\lambda_{j}(M_{\mu_1}) \leq \lambda_{j-1}(M_{\mu_2})+\epsilon_{\mu_2}\leq \lambda_{j-1}(\Pirperp M_{\mu_2}\Pirperp)+\epsilon_{\mu_2}',$$
where $\epsilon_{\mu_2}'\to 0$ as $\mu_2\to{\lambda^a_i}^-$.

Let us now lower bound $\lambda_{j}(M_{\mu_1})$. In a similar manner as before, we have that 
\begin{align*}
    \lambda_{j}(M_{\mu_1})&=\lambda_{j}\paren{(\Pir+\Pirperp)M_{\mu_1}(\Pir+\Pirperp)}\\
    &\geq\lambda_{j}\paren{\frac{1}{2}\Pir M_{\mu_1}\Pir+\Pirperp M_{\mu_1}\Pirperp}
                          +\lambda_{d-p}\paren{\frac{1}{2}\Pir M_{\mu_1}\Pir+\Pirperp M_{\mu_1}\Pir +\Pir M_{\mu_1}\Pirperp}\\
                          &\geq \lambda_{j}\paren{\frac{1}{2}\Pir M_{\mu_1}\Pir+\Pirperp M_{\mu_1}\Pirperp} - \epsilon_{\mu_1},
\end{align*}
where as $\mu_1\to{\lambda^a_i}^+$ it holds that $\epsilon_{\mu_1}\to 0$ with the same arguments as above. 

Note that each of the eigenvectors of $\frac{1}{2}\Pir M_{\mu_1}\Pir$ corresponding to a non-zero eigenvalue is orthogonal to each of the eigenvectors of  $\Pirperp M_{\mu_1}\Pirperp$ corresponding to a non-zero eigenvalue. Furthermore, $\lambda_1\paren{\frac{1}{2}\Pir M_{\mu_1}\Pir}\to \infty$ as $\mu_1\to{\lambda^a_i}^+$. Lastly, since $\Pirperp M_{\mu_1}\Pirperp$ has a convergent, finite limit, it holds that 
$$\lambda_{j}\paren{\frac{1}{2}\Pir M_{\mu_1}\Pir+\Pirperp M_{\mu_1}\Pirperp} = \lambda_{j-1}\paren{\Pirperp M_{\mu_1}\Pirperp},$$
as $\mu_1\to{\lambda^a_i}^+$.
This allows us to conclude that
$$\lambda_{j-1}\paren{\Pirperp M_{\mu_1}\Pirperp} - \epsilon_{\mu_1}\leq  \lambda_{j}(M_{\mu_1})\leq \lambda_{j-1}(M_{\mu_2})+\epsilon_{\mu_2}\leq\lambda_{j-1}(\Pirperp M_{\mu_2}\Pirperp)+\epsilon_{\mu_2}'.$$
Finally, by taking the limits $\mu_1\to{\lambda^a_i}^+$ and $\mu_2\to{\lambda^a_i}^-$ we get
\begin{equation}
    \lim_{\mu_1\to{\lambda^a_i}^+}L_j(\mu_1) = \lim_{\mu_2\to{\lambda^a_i}^-}L_{j-1}(\mu_2) = \lambda_{j-1}\paren{\Pirperp M_{\lambda_i^a}\Pirperp}.
    \label{eq:limitL_i}
\end{equation}
\end{proof}

From now on, we will refer to the function $\tilde{L}_i$ as the one extended to the whole $]\lambda_{i}^a,+\infty[$ for $i\in[p]$, and to the whole $]\lambda_{p}^a,+\infty[$ for $i\geq p+1$. The next result characterizes the behavior of $\tilde{L}_i$  
at the edges of its domain.

\begin{lemma}\label{lem:eiglimits}
    For $i\in [d-p]$, $\tilde{L}_i$ is non-increasing with the limit on the right edge of the domain given by 
    $$\lim_{\mu\to\infty}\tilde{L}_i(\mu)=\lambda_i(P).$$
    Moreover, for $i\in[p]$, the limit on the left edge of the domain is
    $$\lim_{\mu\to{\lambda_{i}^a}^+}\tilde{L}_i(\mu)=+\infty.$$
\end{lemma}
\begin{proof}
    First, we prove that each $L_i(\mu)$ is non-increasing in an arbitrary domain $]\lambda_j^a,\lambda_{j-1}^a[$. In fact, for any $\mu_1>\mu_2$ in that interval, it holds that
    \begin{align*}
        L_i(\mu_1)-L_i(\mu_2) &= \lambda_i(P-q(a-\mu_1 I_p)^{-1}q^\top) - \lambda_i(P-q(a-\mu_2 I_p)^{-1}q^\top)\\
        &\leq \lambda_{1}(q(a-\mu_2 I_p)^{-1}q^\top-q(a-\mu_1 I_p)^{-1}q^\top)\\
        &= \lambda_{1}\left(q\sum_{k=1}^p v_k^av_k^{a\top}\paren{\frac{\mu_2-\mu_1}{(\lambda_k^a - \mu_2)(\lambda_k^a - \mu_1)}}q^\top\right)\\
        &\leq 0,
    \end{align*}
    where the first inequality is due to Weyl's inequality.
    
    Thus, by the definition in \eqref{eq:defLtilde} and by \Cref{lemma:samelimits},  we have that the function $\tilde{L}_i$ is non-increasing in $]\lambda_i^a,+\infty[$, for $i\in[p]$. In the same manner, for $i>p$, it also holds that $\tilde{L}_i$ is non-increasing in $]\lambda_p^a,+\infty[$.
    Moreover, since $qv_i^a\neq 0$ and
    $$P-q(a-\mu I_p)^{-1}q^\top = P - \frac{(qv_i^a)(qv_i^a)^\top}{\lambda_i^a - \mu}-\sum_{k\neq i}^p \frac{(qv_k^a)(qv_k^a)^\top}{\lambda_k^a - \mu},$$
    it holds that, for any $i\in[p]$,
    $$\lim_{\mu\to{\lambda_{i}^a}^+}\tilde{L}_i(\mu) = \lim_{\mu\to{\lambda_{i}^a}^+}\lambda_1(P-q(a-\mu I_p)^{-1}q^\top) = +\infty.$$
    Finally, using the same formula, one also obtains that, for $i\in [d-p]$,
    \begin{align}
        \lim_{\mu\to\infty}\tilde{L}_i(\mu) &= \lim_{\mu\to\infty}\lambda_i(P-q(a-\mu I_p)^{-1}q^\top) = \lambda_i(P), \notag 
    \end{align}
    which concludes the proof.
\end{proof}
Having proven this properties, let us get back to discussing the eigenvalues $\lambda_i^D$. We do so by considering two cases.

\begin{lemma}\label{lemma:iffeig}
    An eigenvalue $\lambda_i^D\notin \Lambda^a$, $i\in [d-p]$, is a solution to 
    \begin{equation}\label{eq:eigvalchar}
        L_k(\mu) = \mu,
    \end{equation}
    for some $k$.
    Conversely, any solution to the previous equation is an eigenvalue of $D$ that is also not an eigenvalue of $a$.
\end{lemma}
\begin{proof}
    All eigenvalues of $D$ are exactly the solutions to
    \begin{equation}\label{eq:detwhole}
        \det\paren{D-\lambda I_{d}}=0.
    \end{equation}
   Applying the formula for the determinant of a block matrix implies
    $$\det\paren{D-\lambda I_{d}} = \det\paren{P-\lambda I_{d-p}-q(a-\lambda I_p)^{-1}q^\top}\det(a-\lambda I_p).$$
    As by assumption $\det(a-\lambda I_d)\neq 0$, \eqref{eq:detwhole} is equivalent to
    $$\det\paren{P-\lambda I_{d-p}-q(a-\lambda I_p)^{-1}q^\top}=0.$$ 
    Moreover by definition of the determinant and the fact that the matrix in questions is symmetric, it holds that
    \begin{align*}
        \det\paren{P-\lambda I_{d-p}-q(a-\lambda I_p)^{-1}q^\top} &= \prod_{i=1}^{d-p}\lambda_i\paren{P-\lambda I_{d-p}-q(a-\lambda I_p)^{-1}q^\top}\\
        &= \prod_{i=1}^{d-p}(L_i(\lambda)-\lambda).
    \end{align*}
    Therefore, we have that $$\det\paren{P-\lambda I_{d-p}-q(a-\lambda I_p)^{-1}q^\top}=0,$$
    if and only if there exists a $k$ and $\mu$ such that
    \begin{align}
        L_k(\mu)&=\mu.\notag
    \end{align}
\end{proof}
The case in which the eigenvalues of $D$ and $a$ overlap is covered by the result below. % the following lemma.
\begin{lemma}\label{lemma:iffeiga}
    An arbitrary eigenvalue $\lambda_i^D\in \Lambda^a$ is equal to $\lambda_j^a$ if and only if  
    \begin{equation}\label{eq:eigvalchara}
        \lim_{\mu\to{\lambda_j^a}^+} L_k(\mu) = \lambda_j^a,
    \end{equation}
    for some $k\geq 2$.
\end{lemma}
\begin{proof}
    Let us first prove the if part of the statement. We denote the eigenvector corresponding to $\lambda_i^D$ as $v_i^D = \matrix{h_i \\ g_i}$, where $h_i\in \bbR^p$, $g\in \bbR^{d-p}$. It follows that
$$ D v_i^D = \matrix{ a & q^\top\\
                    q & P}\matrix{h_i \\ g_i} = \lambda^a_j \matrix{h_i \\ g_i}.$$
Splitting this equation into $p$ and $d-p$ coordinates gives
\begin{align}
    a h_i + q^\top g_i &= \lambda^a_j h_i,\label{eq:firsteq}\\
    q h_i + P g_i &= \lambda^a_j g_i.\label{eq:secndeq}
\end{align}
Since $(a-\lambda^a_jI_p)$ is singular, its SVD decomposition is
$$(a-\lambda^a_jI_p) = \sum_{k=1}^p (\lambda_k^a-\lambda_j^a)v_k^av_k^{a\top}= \sum_{k\neq j}^p (\lambda_k-\lambda_j)v_k^av_k^{a\top}.$$
From  \eqref{eq:firsteq}, it follows that $(a-\lambda^a_jI_p)h_i = -q^\top g_i$. Then, plugging in the SVD, it holds that $\sum_{k\neq j}^p (\lambda_k-\lambda_j)v_k^av_k^{a\top} h_i= -q^\top g_i$. Multiplying both sides by $v_j^a$, due to the matrix being symmetric and thus eigenvectors orthogonal, it holds that
$$\sum_{k\neq j}^p (\lambda_k-\lambda_j)\inprod{v_k^a}{h_i}\inprod{v_k^a}{v_j^a} = 0 = -\inprod{q^\top g_i}{v_j^a}.$$
From there, we can conclude that $q^\top g_i \perp v_j^a$, which is equivalent to $g_i\perp qv_j^a$. Moreover, \eqref{eq:firsteq} can be rewritten as %out we can get that it is equivalent to having 
$$h_i = -(a-\lambda^a_jI_p)^\dagger q^\top g_i + \alpha v_j^a,$$
for some $\alpha$. Plugging this result into \eqref{eq:secndeq} gives %one has that it is equivalent to having a solution to
\begin{equation}\label{eq:dc} -q (a-\lambda^a_j I_p)^\dagger q^\top g_i + \alpha qv_j^a + P g_i = \lambda^a_j g_i,
\end{equation}
for some $\alpha$.
Let us, as before, denote by $r =q v_j^a/\|q v_j^a\|_2$, and by $\Pirperp$ the orthogonal projection to the subspace defined by $r^\perp$. As noted before, $g_i\perp qv_j^a$, which implies that $g_i=\Pirperp g_i$. Plugging it into \Cref{eq:dc} gives
$$ -q (a-\lambda^a_j I_p)^\dagger q^\top \Pirperp g_i + \alpha qv_j^a + P \Pirperp g_i = \lambda^a_j \Pirperp g_i.$$
Multiplying the previous equation on the left by $\Pirperp$ results in 
$$ -\Pirperp q (a-\lambda^a_j I_p)^\dagger q^\top \Pirperp g_i + \Pirperp P \Pirperp g_i = \lambda^a_j \Pirperp g_i.$$
This means that $\Pirperp g_i$ is an eigenvector of the matrix $-\Pirperp q (a-\lambda^a_j I_p)^\dagger q^\top \Pirperp + \Pirperp P \Pirperp$ with the corresponding eigenvalue $\lambda^a_j$, i.e.,
$$\lambda_k(\Pirperp (P-q (a-\lambda^a_jI_p)^\dagger q^\top)\Pirperp) = \lambda_j^a,$$
for some $k$. From \eqref{eq:limitL_i} we can see the LHS is exactly $\lim_{\mu\to{\lambda_j^a}^+} L_k(\mu)$ for some $k$. As proved in \Cref{lem:eiglimits}, it holds that $\lim_{\mu\to{\lambda_j^a}^+} L_1(\mu) = +\infty$, so it must be that $k\geq 2$.

Conversely, by following the same steps in reverse, if 
$\lim_{\mu\to{\lambda_j^a}^+} L_k(\mu)=\lambda_j^a$, then there is a vector $g_i$ that solves
$$ -q (a-\lambda^a_j)^\dagger q^\top g_i + \alpha qv_j^a + P g_i = \lambda^a_j g_i,$$
for some $\alpha$. By setting 
$$h_i = -(a-\lambda^a_jI_p)^\dagger q^\top g_i + \alpha v_j^a,$$
it follows that $w = \matrix{h_i \\ g_i}$ is an eigenvector of $D$ with eigenvalue $\lambda_j^a$ as stated.
\end{proof}

Combining the previous properties, we are now ready to state and formally prove the characterization in \Cref{eq:eigvalrec-body}.

\begin{proposition}\label{prop:eigvalrec}
    For $i\in[p]$, the eigenvalue $\lambda_i^D$ is the unique solution to the equation 
    \begin{equation}\label{eq:eigvalrec}
        \tilde{L}_i(\mu) = \mu,
    \end{equation}  
    in the respective domain $]\lambda_i^a,\infty[$.
\end{proposition}

\begin{proof}
    Let us first prove that \eqref{eq:eigvalrec} has a unique solution, for $i\in[p]$. 
    Since $\tilde{L}_i(\mu)$ is non-increasing and continuous, we have that $\tilde{L}_i(\mu) -\mu$ is decreasing and continuous. Moreover, for $i\in[p]$ it has limits
    $$\lim_{\mu\to{\lambda_{i}^a}^+}\tilde{L}_i(\mu)-\mu=+\infty,\text{ and } \lim_{\mu\to\infty}\tilde{L}_i(\mu)-\mu=-\infty,$$
    due to \Cref{lem:eiglimits}. Then, applying the intermediate value theorem implies that there must be a unique $\mu_i$ for which $\tilde{L}_i(\mu_i) - \mu_i=0$.
    
    Next, let us prove that the unique solution $\mu_i$ of \eqref{eq:eigvalrec} is indeed an eigenvalue of $D$. First, suppose that $\mu_i=\lambda_j^a$, for some $j\in[p]$. Then, \eqref{eq:limitL_i} would imply that 
    $$\lim_{\mu\to{\lambda_j^a}^+} L_k(\mu) = \lambda_j^a,$$
    for some $k\geq 2$.
    Then, \Cref{lemma:iffeiga} implies that $\lambda_j^a$ is an eigenvalue of $D$. Next,  suppose  that $\mu_i\notin \Lambda^a$. Then, by definition of $\tilde{L}_i$, it holds that
    $$L_k(\mu_i)=\mu_i,$$
    for some $k$. \Cref{lemma:iffeig} then implies that $\mu_i$ must be an eigenvalue of $D$.
    
    Finally, let us prove that $\mu_i$ is exactly the $i$-th eigenvalue of $D$. To do so, we first prove that every eigenvalue of $D$ that is larger or equal to $\lambda_p^a$ is a solution to the following equation in $\mu$:
    $$\tilde{L}_m(\mu)=\mu,$$
    for some $m\in [d-p]$. This follows from \Cref{lemma:iffeig,lemma:iffeiga}, which imply that any eigenvalue of $D$ is covered by checking the conditions  
    $$ L_k(\mu) = \mu \text{ or } \lim_{\mu\to{\lambda_j^a}^+} L_k(\mu) = \lambda_j^a,$$
    which are all covered by considering $\tilde{L}_m(\mu)$ for $m\in [d-p]$.
    
    As $\tilde{L}_1(\mu)\geq \tilde{L}_2(\mu) \geq \dots \geq  \tilde{L}_p(\mu) \geq\dots\geq\tilde{L}_{d-p}(\mu)$ and $\lambda_1^D\geq \lambda_2^D\geq \dots \lambda_p^D$,
    it must be that the solution to \eqref{eq:eigvalrec} is exactly the $i$-th eigenvalue of the matrix $D$, and the proof is complete.
\end{proof}

Next, we prove the eigenvector characterization in \Cref{eq:eigenvec-body}.

\begin{proposition}\label{prop:eigenvec}
Let $j\in[p]$ be s.t.\ $\lambda_i^D>\lambda_1^P$ for all $i\leq j$. Then, for all $i\leq j$, it holds that
    $$h_i = \frac{\tilde{h}_i}{\sqrt{1-\tilde{h}_i^\top\frac{d}{d\lambda}R(\lambda_i^D)\tilde{h}_i}},$$
    where $\tilde{h}_i$ is the %unit norm 
    eigenvector of $R(\lambda_i^D)$ and $\frac{d}{d\lambda}R(\lambda)$ is the entry-wise derivative of %the matrix 
    $R(\lambda)$.
\end{proposition}
%\subsection{Eigenvectors}\label{subsec:appengenraleigvec}

\begin{proof}
Note that the eigenvector equation is equivalent to the system of two equations
    \begin{align}
        a h_i + q^\top g_i &= \lambda^D_i h_i,\label{eq:eig1}\\
        q h_i + P g_i &= \lambda^D_i g_i.\label{eq:eig2}
    \end{align}
    
    As we consider only the eigenvectors $v_i^D$ for $i\leq j$, the matrix $(P-\lambda_i^DI_{d-p})$ is invertible, and solving \eqref{eq:eig1} gives
    $$g_i = -(P-\lambda_i^DI_{d-p})^{-1}qh_i.$$
    Substituting in \eqref{eq:eig2} yields
        $$ah_i-q^\top(P-\lambda_i^DI_{d-p})^{-1}qh_i=\lambda_i^Dh_i.$$
    Let us denote by $\tilde{h}_i=\frac{h_i}{\norm{2}{h_i}}$ the unit norm eigenvector of $a - q^\top(P-\lambda_i^DI_{d-p})^{-1}q$ corresponding to the eigenvalue $\lambda_i^D$, and also define $\tilde{g}_i := -(P-\lambda_i^DI_{d-p})^{-1}q\tilde{h}_i$. Then, $\tilde{h}_i$ and $\tilde{g}_i$ satisfy equations \eqref{eq:eig1} and \eqref{eq:eig2}, so $\tilde{v}_i^D = \matrix{\tilde{h}_i\\ \tilde{g}_i}$ is aligned with an eigenvector corresponding to eigenvalue $\lambda_i^D$. However, $\tilde{v}_i^D$  does not necessarily have unit norm. It holds that
    $$\matrix{h_i\\ g_i} = v_i^D = \frac{\tilde{v}_i^D}{\norm{2}{\tilde{v}_i^D}} = \frac{\matrix{\tilde{h}_i\\ \tilde{g}_i}}{\sqrt{\tilde{h}_i^\top\tilde{h}_i+\tilde{g}_i^\top\tilde{g}_i}}=\frac{\matrix{\tilde{h}_i\\ \tilde{g}_i}}{\sqrt{1+\tilde{h}_i^\top q^\top(P-\lambda_i^DI_{d-p})^{-2}q\tilde{h}_i}},$$ from which follows that
    $$h_i = \frac{\tilde{h}_i}{\sqrt{1+\tilde{h}_i^\top q^\top(P-\lambda_i^DI_{d-p})^{-2}q\tilde{h}_i}}.$$
    The last thing to notice is that
    $$q^\top(P-\lambda I_{d-p})^{-2}q = -\frac{d}{d\lambda}(a-q^\top(P-\lambda I_{d-p})^{-1}q),$$
    from which the statement of the proposition follows.
\end{proof}

We conclude by expressing the entries of the matrix $R(\lambda)$ and its derivative in terms of the auxiliary functions in \Cref{eq:auxfun}.

\begin{lemma}\label{lemma:matrixRchar}
    Let us assume that $\lambda>\lambda_1^P$. Then, for the  diagonal elements of $R(\lambda)$, it holds that
    \begin{align}
        R(\lambda)_{i,i} &= a_{i,i}+\frac{1}{\cLi^{-1}(\lambda)},\label{eq:matrixRminusonediag} \\
        \frac{d}{d\lambda}R(\lambda)_{i,i} &= -\frac{1}{\paren{\cLi^{-1}(\lambda)}^2\cLi'(\cLi^{-1}(\lambda))}.\label{eq:matrixRminustwodiag}
    \end{align}
    Moreover, for the off-diagonal elements, we have
    \begin{align}
        2R(\lambda)_{i,j} &=2a_{i,j} + \frac{1}{\cL_{i,j}^{-1}(\lambda)} - \frac{1}{\cL_i^{-1}(\lambda)} - \frac{1}{\cL_j^{-1}(\lambda)},\label{eq:matrixRminusoneoffdiag}\\
        2\frac{d}{d\lambda}R(\lambda)_{i,j} &= \frac{1}{\paren{\cLi^{-1}(\lambda)}^2\cLi'(\cLi^{-1}(\lambda))}+\frac{1}{\paren{\cL_j^{-1}(\lambda)}^2\cL_j'(\cL_j^{-1}(\lambda))}-\frac{1}{\paren{\cL_{i,j}^{-1}(\lambda)}^2\cL_{i,j}'(\cL_{i,j}^{-1}(\lambda))} .\label{eq:matrixRminustwooffdiag}
    \end{align}
\end{lemma}
\begin{proof}
Note that
\begin{equation}\label{eq:matrixRelements}
        R(\lambda)_{i,i} = a_{i,i}-q_i^\top(P-\lambda I_{d-p})^{-1}q_i, \qquad R(\lambda)_{i,j} = a_{i,j}-q_i^\top(P-\lambda I_{d-p})^{-1}q_j.
    \end{equation}
    Transforming \eqref{eq:matrixRelements} with the matrix determinant lemma yields that, for any $\lambda>\lambda_1^P$,    \begin{equation}\label{eq:ijelementmatrixR}
        R(\lambda)_{i,i} = a_{i,i}-q_i^\top(P-\lambda I_{d-p})^{-1}q_i=a_{i,i}+\frac{1}{\cLi^{-1}(\lambda)}.
    \end{equation}
    Moreover, it holds that
    \begin{equation}
        R(\lambda)_{i,j} =a_{i,j}- \frac{(q_i+q_j)^\top(P-\lambda I_{d-p})^{-1}(q_i+q_j) - q_i^\top(P-\lambda I_{d-p})^{-1}q_i - q_j^\top(P-\lambda I_{d-p})^{-1}q_j}{2}. \notag 
    \end{equation}
    From the same transformation with the matrix determinant lemma, it follows that 
    $$(q_i+q_j)^\top(P-\lambda I_{d-p})^{-1}(q_i+q_j) = -\frac{1}{\cL_{i,j}^{-1}(\lambda)}.$$
    Substituting the previous identity in \Cref{eq:ijelementmatrixR} gives \Cref{eq:matrixRminusoneoffdiag}.
    Note that, for $\mu$ such that $\cLi(\mu)>\lambda_1^P$, it holds that $\cLi$ is an increasing differentiable function, so its inverse and derivative are well defined. 
    Finally, by differentiating %equations 
    \Cref{eq:matrixRminusonediag} and \Cref{eq:matrixRminusoneoffdiag}, we get the other two equations.
\end{proof}

\section{Proofs for the characterization of eigenvalues}\label{app:asymptotbehav}

\subsection{Auxiliary results}

As a consequence of \Cref{prop:eigvalrec}, the top $p$ eigenvalues of $D$ are entirely characterized by the functions $\tilde{L}_i$. As these functions are nothing more than patches of the functions $L_i$ on different domains, we first direct our attention to analyzing asymptotic behavior of $L_i(\mu) = \lambda_i(P-q(a-\mu I_p)^{-1}q^\top)$. Notice that
\begin{equation}\label{eq:defM}
    P-q(a-\mu I_p)^{-1}q^\top = \frac{1}{n} U^\top M_n U,
\end{equation}
where $M_n\coloneqq Z-\frac{1}{n}ZS(a-\mu I_p)^{-1}S^\top Z$ is a rank $p$ perturbation of the matrix $Z$.
\begin{lemma}\label{lemma:mueigvalconv}
    For each $\mu>0$,    let $\alpha_1\geq \dots\ \geq \alpha_j > \tau$ be all the solutions to the equation 
    \begin{equation}\label{eq:mueigvalmaster_eq}
        \det\paren{\mu I_p-R^\infty(\alpha)}=0.
    \end{equation}
    Then, for the top $j$ eigenvalues of $M_n$, it holds that
    \begin{equation}\label{conv:jeigconv}
        \lambda_1^M,\dots,\lambda_j^M \asconv \alpha_1, \dots, \alpha_j,
    \end{equation}
    and for the remaining $p-j$ eigenvalues, it holds that
    $$\lambda_{j+1}^M,\dots,\lambda_p^M \asconv \tau.$$
\end{lemma}
\begin{proof}
    Let us denote by $v\coloneqq ZS$.
    An arbitrary eigenvalue $\lambda_k^M$ of $M_n$ satisfies the equation
    $$\det\paren{Z-\frac{1}{n}v(a-\mu I_p)^{-1}v^\top-\lambda_k^M I_n}=0.$$
    Thus, for $\alpha>\max\{z_i\}$, consider the following equation
    $$\det\paren{Z-\frac{1}{n}v(a-\mu I_p)^{-1}v^\top-\alpha I_n}=0.$$
    As $Z-\alpha I_n$ is invertible for $\alpha>\max\{z_i\}$, we can apply the matrix determinant lemma to obtain the equivalent equation 
    \begin{equation}\label{eq:matdetlemmu}
        \det\paren{\mu I_p-a+\frac{1}{n}v^\top(Z-\alpha I_n)^{-1}v}=0.
    \end{equation}
    Moreover,
    $$ a - \frac{1}{n}v^\top(Z-\alpha I_n)^{-1}v = \frac{1}{n}\sum_{i=1}^n z_i s_i s_i^\top - \frac{1}{n}\sum_{i=1}^n \frac{z_i^2s_is_i^\top}{z_i-\alpha} = \frac{1}{n}\sum_{i=1}^n \frac{\alpha z_is_is_i^\top}{\alpha-z_i}.$$
    Thus, \eqref{eq:matdetlemmu} becomes 
    \begin{equation}\label{eq:prodeigenval}
        \det\paren{\mu I_p -  \frac{1}{n}\sum_{i=1}^n \frac{\alpha z_is_is_i^\top}{\alpha-z_i}}=0.
    \end{equation}
    Let us prove that, for $n$ large enough, this equation indeed has its top $j$ solutions for $\alpha>\max\{z_i\}$.
    First, note that
    $$\det\paren{\mu I_p -  \frac{1}{n}\sum_{i=1}^n \frac{\alpha z_is_is_i^\top}{\alpha-z_i}}=\prod_{i=1}^p\paren{\mu-\lambda_i(\alpha)},$$
    where through abuse of notation we define $\lambda_i(\alpha):]\max\{z_i\},+\infty[\to \bbR$ as $ \lambda_i\paren{\frac{1}{n}\sum_{i=1}^n \frac{\alpha z_is_is_i^\top}{\alpha-z_i}}$. Each function $\lambda_i$ is continuous and strictly decreasing. This can be seen by taking arbitrary $\alpha_2>\alpha_1$ to get
    $$\lambda_i(\alpha_1)-\lambda_i(\alpha_2)\geq \lambda_p\paren{(\alpha_2-\alpha_1)\frac{1}{n}\sum_{i=1}^n \frac{z_i^2s_is_i^\top}{(\alpha_1-z_i)(\alpha_2-z_i)}}>0,$$
    by using Weyl's inequality and the fact that there are almost surely at least $p$ linearly independent vectors $z_is_i$ by \Cref{lemma:fr}. 
    Consequently, to prove that \eqref{eq:prodeigenval} has $j$ solutions for $\alpha>\max\{z_i\}$, it equivalent to prove that, for each $i$, 
    \begin{equation}\label{eq:betas}
        \lambda_i(\beta_i')>\mu>\lambda_i(\beta_i''),
    \end{equation}
    for some $\beta_i''>\beta_i'>\max{z_i}$. 
    
    Note that, for any fixed $\alpha$, it holds that
    \begin{equation}\label{eq:eigenvalconvergence}
        \frac{1}{n}\sum_{i=1}^n \frac{\alpha z_is_is_i^\top}{\alpha-z_i}\asconv \expt{\frac{\alpha zss^\top}{\alpha-z}} = \Rinfty(\alpha),
    \end{equation}
    due to the law of large numbers. Due to the continuity of eigenvalues, it further follows that 
    $$\lambda_i\paren{\frac{1}{n}\sum_{i=1}^n \frac{\alpha z_is_is_i^\top}{\alpha-z_i}}\asconv \lambda_i\paren{\expt{\frac{\alpha zss^\top}{\alpha-z}}}=\lambda_i^\infty(\alpha),$$
    where $\lambda_i^\infty(\alpha)$ is continuous and strictly decreasing. This can be seen as, for any $\alpha>\tau$ and any arbitrary vector $x\in \bbR^p$, it holds that
    \begin{equation}\label{eq:matrixRdecreases}
        \frac{d}{d\alpha}\paren{x^\top\paren{\expt{\frac{\alpha zss^\top}{\alpha-z}}}x} = -\expt{\frac{\inprod{x}{s}^2z^2}{(\alpha-z)^2}}<0,
    \end{equation}
    since $\prob{z=0}<1$. 
    Moreover, for $i\in[p]$,
    $$\lim_{\alpha\to\infty}\lambda_i^\infty(\alpha) = \lambda_i^{a^\infty},$$
    where the matrix $a^\infty = \expt{zss^\top}$ is the limit of the matrix $a$.
    The condition of the lemma states that there exist $\alpha_1\geq \dots\ \geq \alpha_j > \tau$ such that
    $$\det\paren{\mu I_p-R^\infty(\alpha)}=0.$$
    Let us denote by $k\in\{0,\dots, p\}$ the index such that $\lambda_{k+1}^{a^\infty}\leq\mu<\lambda_k^{a^\infty}$, with the abuse of notation $\lambda_0^{a^\infty}\coloneq +\infty$ and $\lambda_{p+1}^{a^\infty}\coloneq -\infty$. 
    By assumption %it holds 
    \begin{equation}\label{eq:deteq}
        \det\paren{\mu I_p-R^\infty(\alpha)}=\prod_{i=0}^p(\mu-\lambda_i^\infty(\alpha)),
    \end{equation}
    has $j$ solutions in $\alpha \in ]\tau,+\infty[$. Note that $\lambda_i^\infty$ is a strictly decreasing continuous function, so the only way that $\lambda_i^\infty-\mu$ does not have a solution in $\alpha\in]\tau,+\infty[$ is if either
    $\lim_{\alpha\to\tau^+} \lambda_i^\infty(\alpha)<\mu$ or $\lim_{\alpha\to\infty} \lambda_i^\infty(\alpha)>\mu.$
    Moreover, since $\lim_{\alpha\to\infty} \lambda_i^\infty(\alpha) = \lambda_i^{a^\infty}$, it will exactly hold for $i\in\{1,\dots k\}$ that 
    \begin{equation}\label{eq:larberthanmu}
        \lim_{\alpha\to\infty} \lambda_i^\infty(\alpha)\geq\lambda_k^{a^\infty}>\mu.
    \end{equation}
    The fact that there are only $j$ solutions to \eqref{eq:deteq} and $\lambda_i^\infty(\alpha)>\lambda_{i+1}^\infty(\alpha)$ implies that 
    $$\lambda_{i+k}^\infty(\alpha_i) = \mu,$$
    for $i\in [j]$,
    as well as
    \begin{equation}\label{eq:smallerthanmu}
        \lim_{\alpha\to\infty} \lambda_i^\infty(\alpha)<\mu,
    \end{equation}
    for $i\in\{j+k+1,\dots,p\}$.
    
    As each $\lambda_i^\infty$ is a strictly decreasing continuous function, this further implies that there exists some constant $\epsilon>0$ and $\alpha_1',\dots, \alpha_j'>\tau$ such that 
    $$\lambda_{i+k}^\infty(\alpha_i') = \mu+\epsilon.$$
    Applying the convergence of \eqref{eq:eigenvalconvergence} it further holds that 
    $$\lambda_{i+k}\paren{\frac{1}{n}\sum_{i=1}^n \frac{\alpha_i' z_is_is_i^\top}{\alpha_i'-z_i}} \asconv \lambda_{i+k}^\infty(\alpha_i')=\mu+\epsilon.$$
    Thus, for each $i$ and $\epsilon>\epsilon_1>0$, there exists $n_0$ s.t.\ for $n>n_0$ 
    $$\abs{\lambda_{i+k}\paren{\frac{1}{n}\sum_{i=1}^n \frac{\alpha_i' z_is_is_i^\top}{\alpha_i'-z_i}} - (\mu+\epsilon)}< \epsilon_1.$$
    Developing the absolute value, it holds that 
    $$ \lambda_{i+k}\paren{\frac{1}{n}\sum_{i=1}^n \frac{\alpha_i' z_is_is_i^\top}{\alpha_i'-z_i}} > \mu+\epsilon-\epsilon_1>\mu,$$
    e.g.\ by taking $\epsilon_1=\epsilon/2$.
    As $\lambda_{i+k}(\alpha) = \lambda_{i+k}\paren{\frac{1}{n}\sum_{i=1}^n \frac{\alpha z_is_is_i^\top}{\alpha-z_i}}$ is a continuous decreasing function, starting from some $n_0$ there exists $\beta_i'>\tau$ s.t.\ $\lambda_{i+k}(\beta_i') > \mu$. Notice that by definition $\tau>\max{z_i}$ almost surely. In the same way as for $\beta_i'$ it can be proved that there exists $\beta_i''>\tau$ such that $\lambda_{i+k}(\beta_i'') < \mu$.
    
    Thus, we conclude that, for large enough $n$,  $\lambda_{i+k}(\alpha)=\mu$ has $j$ solutions larger than $\max\{z_i\}$.
    These are indeed $\lambda_i^M$ for $1\leq i\leq j$. Due to monotonicity, each $\lambda_i$ admits a functional inverse and it holds that 
    $$\lambda_i^M = \lambda_{i+k}^{-1}(\mu).$$
    As $\lambda_{i+k} \asconv \lambda_{i+k}^\infty$, applying \cite[Lemma A.1]{lu2020phase} implies that
    \begin{equation}\label{eq:klimit}
        \lambda_i^M \asconv (\lambda_{i+k}^\infty)^{-1}(\mu),
    \end{equation}
    for $1\leq i\leq j$, which is exactly the statement \eqref{conv:jeigconv} of the lemma.

    Let us now prove the second part of the statement. To do so, we prove that, for large enough $n$, \eqref{eq:prodeigenval} has no more than $j$ solution for $\alpha>\max\{z_i\}$. 
    As stated in \eqref{eq:larberthanmu} and \eqref{eq:smallerthanmu}, it holds that $
\lim_{\alpha\to\infty}\lambda_i^\infty(\alpha)>\mu$ for $i$ s.t.\ $1\leq i\leq k$ and that $\lim_{\alpha\to\tau^+}\lambda_i^\infty(\alpha)<\mu$ for $i$ s.t.\ $j+1+k\leq i\leq p$. Thus, using the same argument as before, we also have that, for large enough $n$ and any $\alpha>\tau$, $\lambda_i(\alpha)>\mu$ for $i$ s.t.\ $1\leq i\leq k$ and $\lambda_i(\alpha)<\mu$ for $i$ s.t.\ $j+1+k\leq i\leq p$. Since  %   \begin{equation}\label{eq:lambdainequalities}
%        \lambda_i(\alpha)>\mu, \text{ or respectively } \lambda_i(\alpha)<\mu
%    \end{equation}
%    for any $\alpha>\tau$, and $1\leq i\leq k$, or respectively $j+1+k\leq i\leq p$. 
%    Note that 
    $\max\{z_i\}\asconv \tau$, as $\tau$ is the right edge of the support of $z$, such inequalities also hold for $\alpha>\max\{z_i\}$ (and $n$ large enough).
    Hence, there cannot exist $\beta_i'$ and $\beta_i''$ that satisfy \eqref{eq:betas}, so \eqref{eq:prodeigenval} cannot have more than $j$ solutions in $\alpha> \max\{z_i\}$. 
    
    From this, it directly follows that, for $n$ large enough and any $l\in\{j+1\dots p\}$,  
    $$\lambda_l^M\leq \max\{z_i\},$$
    almost surely. Furthermore, from the interlacing theorem, it holds that 
    $$\lambda_{p+1}^Z\leq\lambda_l^M,$$
    for any $l\in\{j+1,\dots, p\}$. Thus, the $\lambda_l^M$'s are sandwiched between the first and the $p$-th largest value of $Z$, both of which converge to the right edge of the bulk $\tau$ as a finite order statistic, which gives the desired result.
\end{proof}

\begin{proposition}\label{prop:tildeLiconv}
    For any fixed $\mu\in]\lambda_i^{a^\infty}, t_i^\infty[$, it holds that
    \begin{equation}\label{eq:p1}
    \tildeLi(\mu) \asconv \tildeLiinfty(\mu)=\zeta_\delta\paren{(\lambda^\infty_i)^{-1}(\mu)}.    
    \end{equation}
Furthemore, for any fixed $\mu\in]t_i^\infty,+\infty[$, it holds that 
\begin{equation}\label{eq:p2}
    \tildeLi(\mu) \asconv\zeta_\delta\paren{\lambdabardelta}.
\end{equation}
\end{proposition}
\begin{proof}
    Let $k$ be such that $\lambda_{k+1}^{a^\infty}\leq\mu<\lambda_k^{a^\infty}$, with the abuse of notation $\lambda_0^{a^\infty}\coloneq +\infty$. Then, by the definition of $\tildeLi$ in \eqref{eq:defLtilde}, it holds that
    \begin{equation}\label{eq:Likconv}
        \tildeLi(\mu) = L_{i-k}(\mu),
    \end{equation}
    for $n$ large enough. This is true since $\lambda_i^a\asconv  \lambda_i^{a^\infty}$, as $a\asconv a^\infty$.
    
    To obtain the convergence of the RHS in \eqref{eq:Likconv}, we rely on the results from \cite{bai-yao-2012}.
    Towards this end, we recall the definition of $L_{i-k}(\mu) = \lambda_{i-k}(P-q(a-\mu I_p)^{-1}q^\top)$ as in \eqref{eq:defLtilde}. Given the equality in \eqref{eq:defM}, %it holds
%    $$P-q(a-\mu I_p)^{-1}q^\top = \frac{1}{n} U^\top M U,$$
%    and 
    we turn our attention to the eigenvalues of $M_n$. Recall that the functions $\lambda^\infty_i(\alpha)$ are strictly decreasing with limits
    $$\lim_{\alpha\to\tau^+} \lambda^\infty_i(\alpha) = t_i^\infty, \text{ and } \lim_{\alpha\to+\infty} \lambda^\infty_i(\alpha) = \lambda_i^{a^\infty}.$$
    Then, as $\mu\in]\lambda_i^{a^\infty}, t_i^\infty[$, the equation
    $$\lambda^\infty_i(\alpha)=\mu$$
    has a unique solution in $\alpha>\tau$. Let us denote that solution by $\alpha_{i-k}$. Due to the fact that 
    $$\lambda_{i}^{a^\infty}\dots\leq\lambda_{k+1}^{a^\infty}\leq\mu<\lambda_k^{a^\infty} \text{ and } \mu < t_i^\infty\dots<t_{1}^\infty,$$ using the same argument as in the proof of \Cref{lemma:mueigvalconv} below \eqref{eq:deteq}, we conclude that there are unique solutions $\alpha_1,\dots,\alpha_{i-k}$ s.t.\ $\lambda_{j+k}(\alpha_j)=\mu$ for $j\in[i-k]$. Then, it holds that $\alpha_1, \dots,  \alpha_{i-k}$ satisfy the conditions of \Cref{lemma:mueigvalconv}. From its proof, specifically \eqref{eq:klimit}, it follows that
    $$\lambda_{i-k}^M\asconv {\lambda_{i}^\infty}^{-1}(\mu) = \alpha_{i-k}.$$
    Furthermore, the empirical spectral distribution of $M_n$ converges almost surely to the distribution of $z$.
    This claim follows from Cauchy's interlacing theorem, using the same argument as in the proof of \cite[Proposition 3.2]{lu2020phase}. Assuming that the preprocessing function $\cT$ is positive, we can apply \cite[Theorems 4.1 and 4.2]{bai-yao-2012} to get
    $$L_{i-k}(\mu)\asconv\zeta_\delta\paren{(\lambda^\infty_i)^{-1}(\mu)},$$
    following the steps in \cite[Proposition 3.3]{lu2020phase}. Finally, the adjustment in  \cite[Lemma 3]{mondelli-montanari-2018-fundamental} covers the case in which the preprocessing function is not necessarily positive, and the proof of \Cref{eq:p1} is complete. 

Let us now consider  $\mu\in]t_i^\infty,+\infty[$ and prove \Cref{eq:p2}. 
    Note that, in this interval of $\mu$, the equation
    $$\lambda^\infty_i(\alpha)=\mu$$
    has no solutions in $\alpha>\tau$. Let us examine the equation \eqref{eq:mueigvalmaster_eq} in \Cref{lemma:mueigvalconv}:
    $$\det\paren{\mu I_p-R^\infty(\alpha)}=\prod_{l=1}^p(\mu-\lambda_l^\infty(\alpha))=0.$$
    The previous equation has a solution $\mu-\lambda_l^\infty(\alpha)=0$, as long as
    \begin{equation}\label{eq:intervalinequal}
        \lambda_l^{a^\infty}<\mu<t_l^\infty,
    \end{equation}
    due to the monotonicity of each $\lambda_l^\infty(\alpha)$.
    As 
    $$\lambda_{i}^{a^\infty}\dots\leq\lambda_{k+1}^{a^\infty}\leq\mu<\lambda_k^{a^\infty} \text{ and } \mu>  t_i^\infty\dots\geq t_{p}^\infty,$$ 
    it follows that \eqref{eq:intervalinequal}, thus also \eqref{eq:mueigvalmaster_eq}, can have at most $i-(k+1)$ solutions. Thus, applying \Cref{lemma:mueigvalconv} it follows that 
    $$\lambda_{i-k}^M\asconv \tau,$$
    implying that $M_n$ in limit has at most $i-(k+1)$ spikes outside the bulk.
    Moreover, the empirical spectral distribution of $M_n$ converges almost surely to the distribution of $z$, so we conclude that 
    $$\tildeLi(\mu) = L_{i-k}(\mu)\asconv \zetadelta(\lambdabardelta),$$
    as the limit of the right edge of the spectral distribution of $\frac{1}{n} U^\top ZU$ by 
    \cite[Lemma 3.1]{bai-yao-2012}, which is due to \cite[Section 4]{silverstein1995analysis}.
\end{proof}

\subsection{Proof of \Cref{thm:eigvalconv}}\label{app:pf2}

We start by recalling some definitions. % that will be useful in the proof of the main result,
For each $i$, let
$\lambda_i^\infty (\alpha) = \lambda_i(R^\infty(\alpha))$
be the $i$-th largest eigenvalue of $ R^\infty(\alpha) $ and let  $t_i^\infty$ be 
$$t_i^\infty\coloneq \lim_{\alpha\to\tau} \lambda^\infty_i(\alpha).$$
Furthermore, due to the law of large numbers, we have %Also, we define $a^\infty\in \bbR^{p\times p}$ which, due to law of large numbers, is the limit of  matrices $a$ as defined in \eqref{eq:Ddef}. Namely,
$$a\asconv a^\infty \coloneq \expt{zss^\top}\in \bbR^{p\times p}.$$
Lastly, consider
\begin{equation}\label{eq:tildeLiinftydef}
\tildeLiinfty(\mu)\coloneq\zeta_\delta\paren{(\lambda^\infty_i)^{-1}(\mu)}, 
\end{equation}
on the domain $\mu\in]\lambda_i^{a^\infty}, t_i^\infty[$. Note that $\tildeLiinfty(\mu)$ is a continuous non-increasing function, as it is the composition of a non-decreasing function and a strictly increasing function.
We are now ready to prove our characterization of the eigenvalues of $D_n$.

\begin{proof}
    Note that \eqref{eq:master_eq2} can be reformulated as
    \begin{equation}\label{eq:zetadeltalambda}
        \det\paren{\zetadelta(\alpha)I_p-R^\infty(\alpha)}=\prod_{i=1}^p( \zetadelta(\alpha)-\lambda_i^\infty(\alpha))=0.
    \end{equation}
    The assumption of the theorem implies that $\alpha_1\geq\dots \geq\alpha_j>\tau$ satisfy \eqref{eq:zetadeltalambda}.
    Recall that each function $\zetadelta(\alpha) - \lambda_i^\infty(\alpha)$ is strictly increasing on $]\tau,+\infty[$, with the right edge limit $+\infty$. Therefore, the implicit assumption that $j\in[p]$ is justified as there could be at most $1$ solution to the equation
    $$\zetadelta(\alpha) - \lambda_i^\infty(\alpha)=0,$$
    for each $i\in[p]$.
    Moreover, it holds by definition that $\lambda_1^\infty(\alpha)\geq\dots\geq\lambda_{p}^\infty(\alpha)$. Thus, it must be that each $\alpha_i$ is the unique solution to 
    $$\zetadelta(\alpha) - \lambda_i^\infty(\alpha)=0,$$
    for $i\in[j]$ and $\alpha>\tau$.
    Let us denote by $\mu_i^*\coloneq {\lambda_i^\infty}(\alpha_i)\in]\lambda_i^{a^\infty},\tau_i^\infty[$. Then $\mu_i^*$ is a solution to the equation
    $$\tildeLiinfty(\mu)-\mu=0,$$
    in the domain of definition $\tildeLiinfty(\mu)$, as $\tildeLiinfty(\mu_i^*) = \zetadelta((\lambda_i^\infty)^{-1}(\lambda_i^\infty(\alpha_i)))$. Moreover, $\mu_i^*$ is the unique such solution, due to the strict monotonicity of $\tildeLiinfty(\mu)-\mu$ on its domain $]\lambda_i^{a^\infty},\tau_i^\infty[$.
    
    Furthermore, \Cref{prop:eigvalrec} implies that each $\lambda_i^D$ is the unique solution to \eqref{eq:eigvalrec}. For each $\mu$ where $\tildeLiinfty(\mu)$ is defined, it holds that
    \begin{equation}\label{eq:asconvtildeli}
        \tildeLi(\mu) - \mu \asconv\tildeLiinfty(\mu) - \mu,
    \end{equation}
    by \Cref{prop:tildeLiconv} in \Cref{app:asymptotbehav}.
    
    As both $\tildeLi(\mu)$ and $\tildeLiinfty(\mu)$ are non-increasing, the functions $\tildeLi(\mu) - \mu$ and  $\tildeLiinfty(\mu)-\mu$ are strictly decreasing. Hence, by \cite[Lemma A.1]{lu2020phase}, it holds that
    \begin{equation}\label{eq:fplast}
\lambda_i^D\asconv\tildeLiinfty(\mu_i^*).
    \end{equation}
    Substituting $\mu_i^* = {\lambda_i^\infty}(\alpha_i)$ in \Cref{eq:fplast} gives \eqref{eq:outsidebulkeigconv} for $i\in[j]$.
    
    It remains to show the claim for the remaining $p-j$ eigenvalues. As \eqref{eq:master_eq2} has only $j$ solutions by assumption, it follows that
    $$\zetadelta(\alpha) - \lambda_i^\infty(\alpha) = \tildeLiinfty(\lambda_i^\infty(\alpha))-\lambda_i^\infty(\alpha)=0$$
    has no solutions for $\alpha>\tau$ and $i> j$. Denoting $\mu = \lambda_i^\infty(\alpha)$, it further holds that
    $$\tildeLiinfty(\mu)-\mu=0$$
    has no solutions for $\mu \in ]\lambda_i^{a^\infty},t_i^\infty[$. Since 
    $$\lim_{\mu\to\lambda_i^{a^\infty}} \tildeLiinfty(\mu) - \mu = +\infty, $$
    it must be that $\tildeLiinfty(\mu)-\mu>0$ for all $\mu\in]\lambda_i^{a^\infty},t_i^\infty[$.
    Using \eqref{eq:asconvtildeli}, we %can conclude that it holds 
    have that
    $$\tildeLi(\mu)-\mu>0,$$
    for all $\mu \in ]\lambda_i^{a^\infty},t_i^\infty[$ and $n$ large enough. As $\lambda_i^a\asconv \lambda_i^{a^\infty}$ and each $\tildeLi(\mu)$ is defined on $]\lambda_i^a,+\infty[$, the solution to the equation
    $$\tildeLi(\mu) - \mu =0 $$
    must be for $\mu>t_i^\infty$.
    Then, applying \Cref{prop:tildeLiconv}, for any fixed $\mu$ it holds that 
    $$\tildeLi(\mu)\asconv \zetadelta(\lambdabardelta).$$
    Lastly, as both $\mu-\tildeLi(\mu)$ and $\mu-\zetadelta(\lambdabardelta)$ are increasing functions, Lemma A.1 in \cite{lu2020phase} implies that 
    $$\lambda_i^D\asconv \zetadelta(\lambdabardelta),$$
    for all $i>j$, which proves the claim.
\end{proof}

\subsection{Number of solutions to \Cref{eq:master_eq2}}\label{app:numsol}

\begin{proposition}\label{prop:exactlyp}
The equation in \Cref{eq:master_eq2} has at most $p$ solutions. Furthermore, if \Cref{assmp:psol} holds, then \Cref{eq:master_eq2} has exactly $p$ solutions. 
\end{proposition}

\begin{proof}
    As stated in the proof of \Cref{thm:eigvalconv}, it holds that
    \begin{equation}\label{eq:detzetali}
        \det\paren{\zetadelta(\alpha)I-R^\infty(\alpha)}= \prod_{i=1}^p (\zetadelta(\alpha) - \lambda_i^\infty(\alpha)).
    \end{equation}
    Note that the function $\zetadelta(\alpha) - \lambda_i^\infty(\alpha)$ is continous and strictly increasing for $\alpha\in]\tau,+\infty[$, so that 
    $$\lim_{\alpha\to\infty}\zetadelta(\alpha) - \lambda_i^\infty(\alpha) = +\infty.$$
    Thus, the equation in \Cref{eq:master_eq2} has at most $p$ solutions. Furthermore, the assumption in \Cref{assmp:psol} implies that 
    $$\inf_{\norm{2}{x}=1}\lim_{\alpha\to\tau^+}x^\top R^\infty(\alpha)x = +\infty,$$
    which is equivalent to
    $$\lim_{\alpha\to\tau^+}\lambda_i^\infty(\alpha) = +\infty.$$
    As $\lim_{\alpha\to\tau^+}\zetadelta(\alpha) = \lambdabardelta<+\infty$, it then holds that
    $$\lim_{\alpha\to\tau^+}\zetadelta(\alpha) - \lambda_i^\infty(\alpha) = -\infty,$$
    proving there must be exactly $p$ solutions to \Cref{eq:master_eq2} due to the intermediate  value theorem.
\end{proof}

\section{Proofs for the characterization of eigenvectors}\label{app:asymptotbehav-eigvec}

\subsection{Auxiliary results}\label{app:auxeig}

\begin{proposition}\label{thm:matrixRconvergence}
     Let $k\in[p]$ be such that $\alpha_k>\lambdabardelta$, where $\alpha_k$ is the $k$-th largest solution of \Cref{eq:master_eq2}. %to the equation
%     $$\det\paren{\zetadelta(\alpha)I-R^\infty(\alpha)}=0.$$
     Then, it holds that
     \begin{equation}\label{eq:matrixRconv}
         R(\lambda_k^D)\asconv R^\infty(\alpha_k),
     \end{equation}
     and
     \begin{equation}\label{eq:matrixderRconv}
         \frac{d}{d\lambda}R(\lambda_k^D)\asconv \frac{1}{\zetadelta'(\alpha_k)}\frac{d}{d\alpha}R^\infty(\alpha_k).
     \end{equation}
\end{proposition}
\begin{proof}
By \Cref{lemma:matrixRchar}, %from subsection \Cref{subsubsec:deteigvec} 
it suffices to understand the behavior of the functions in \Cref{eq:auxfun}, 
%$$\cLi(\mu)= \lambda_1(P+\mu q_iq_i^\top), \text{ and } \cL_{i,j}(\mu)= \lambda_1(P+\mu (q_i+q_j)(q_i+q_j)^\top),$$
as $n,d\to \infty$. However, we first need to verify the assumption that $\lambda_k^D>\lambda_1^P$ almost surely. From \Cref{thm:eigvalconv}, it follows that 
$$\lambda_k^D\asconv \zetadelta(\alpha_k) ,$$
and $\zetadelta(\alpha_k)> \zetadelta(\lambdabardelta)$ as $\alpha_k>\lambdabardelta$ and $\zetadelta$ is strictly increasing on $]\lambdabardelta,+\infty[$. 
Furthermore, $\lambda_1^P\asconv \zetadelta(\lambdabardelta)$, hence $\lambda_k^D>\lambda_1^P$ almost surely.

Let us denote by $G$ the function 
$$G(\mu) = -\frac{1}{\mu},$$
which we will use in the continuation of the proof.
Using \cite{bai-yao-2012} as in the proof of \Cref{prop:tildeLiconv}, 
we get that 
$$\cLi(\mu)\asconv \zetadelta \circ {Q_i}^{-1}\circ G(\mu),$$
where $Q_i(\alpha) \coloneq \expt{\frac{s_i^2z^2}{z-\alpha}}$. Notice that $Q_i(\alpha)$ is invertible by \cite[Remark 3.3]{lu2020phase}, which is stated for the analogous function $Q$. 
In the same manner, it holds that
$$\cL_{i,j}(\mu)\asconv \zetadelta \circ {Q_{i,j}}^{-1}\circ G(\mu),$$
where $Q_{i,j}(\alpha) \coloneq \expt{\frac{(s_i+s_j)^2z^2}{z-\alpha}}$. As $\alpha_k>\lambdabardelta$, we have that $\zetadelta$ is strictly increasing and invertible, hence
$$\cLi^{-1}(\lambda_k^D)\asconv G \circ Q_i \circ \zetadelta^{-1} \circ \zetadelta(\alpha_k) = G \circ Q_i (\alpha_k),$$
which follows from \cite[Lemma A.1]{lu2020phase}.
Plugging this into \eqref{eq:matrixRminusonediag} we get
\begin{equation}\label{eq:tempdiagconv}
    R(\lambda_k^D)_{i,i}\asconv a^{\infty}_{i,i}-Q_i (\alpha_k).
\end{equation}
Note that 
$$a^{\infty}_{i,i}-Q_i (\alpha_k) = \expt{s_i^2z} -  \expt{\frac{s_i^2z^2}{z-\alpha_k}} = \expt{\frac{\alpha_k s_i^2z}{\alpha_k-z}} = R^\infty(\alpha_k)_{i,i}.$$
Similarly, it holds that
$$R(\lambda_k^D)_{i,j}\asconv a^{\infty}_{i,j}-Q_{i,j} (\alpha_k),$$
which combined with \eqref{eq:tempdiagconv} proves \eqref{eq:matrixRconv}.

Moreover, we have that $\cLi(\mu)$ is differentiable (see \Cref{lemma:matrixRchar}), so for its derivative it holds that
$$\cLi'(\mu)\asconv \zetadelta' \circ {Q_i}^{-1}\circ G(\mu) \cdot \paren{Q_i^{-1}}'\circ G(\mu) \cdot G'(\mu),$$
which follows from \cite[Lemma A.2]{lu2020phase}.
Plugging this into \eqref{eq:matrixRminustwodiag} we get 
$$\frac{d}{d\lambda}R(\lambda_k^D)_{i,i}\asconv \frac{\frac{d}{d\alpha}(R^\infty(\alpha_k)_{i,i})}{\zetadelta'(\alpha_k)}.$$
Similarly, it holds that
$$\frac{d}{d\lambda}R(\lambda_k^D)_{i,j}\asconv \frac{\frac{d}{d\alpha}(R^\infty(\alpha_k)_{i,j})}{\zetadelta'(\alpha_k)}.$$
Combining the last two equations we obtain \eqref{eq:matrixderRconv}.
\end{proof}

\subsection{Proof of \Cref{thm:main}}\label{app:pfthmmain}

     \begin{proof}
     We start by proving \Cref{eq:liminfconv}. Let $v_i^D=\matrix{h_i\\ g_i}$, for $i\in\{k,\dots, k+m-1\}$. Since $\alpha_k>\lambdabardelta$, the conditions of \Cref{prop:eigenvec} are satisfied as in the proof of \Cref{thm:matrixRconvergence}. Thus, it holds that
    \begin{equation}\label{eq:hkhtildek}
        h_i = \frac{\tilde{h}_i}{\sqrt{1-\tilde{h}_i^\top\frac{d}{d\lambda}R(\lambda_i^D)\tilde{h}_i}},
    \end{equation}
    where $\tilde{h}_i = \frac{h_i}{\norm{2}{h_i}}$ is the unit norm eigenvector of $R(\lambda_i^D)$. Note that the vectors $\tilde{h}_i$ are orthogonal. %The vectors $v_i^D$ are either orthogonal as eigenvectors corresponding to different eigenvalues, or should a multiplicity in $\lambda_i^D$ occur, through the choice of an orthogonal basis of the corresponding eigenspace. Consequently, the vectors $\tilde{h}_i$ will be orthogonal as eigenvectors of the matrix corresponding to eigenvalues of different value, or through choice of the orthogonal eigenbasis should there be eigenvalue multiplicty. 
    Furthermore, \Cref{thm:matrixRconvergence} gives that
    \begin{equation}\label{eq:rlambdaconv}
        R(\lambda_k^D)\asconv R^\infty(\alpha_k), \qquad \frac{d}{d\lambda}R(\lambda_k^D) \asconv \frac{1}{\zetadelta'(\alpha_k)}\frac{d}{d\alpha}R^\infty(\alpha_k).
    \end{equation}
    Then, applying the results from \cite[II.1.4]{kato2013perturbation}, %(as proved in this StackExchange \href{https://math.stackexchange.com/questions/4054792/convergence-of-eigenvalues-and-spaces-of-sequence-of-compact-szmmetric-and-posi}{answer}).
    it holds that the orthonormal projection to the eigenspace corresponding to the $k$-th eigenvalue also converges, that is 
    \begin{equation}\label{eq:pieqkconv}
        \Pi_{E_k} \asconv \Pi_{E^\infty_k},
    \end{equation}
    where $E_k$ is the space spanned by the eigenvectors $h_k,\dots, h_{k+m-1}$ and $E^\infty_k$ is the eigenspace of the limiting matrix $R^\infty(\alpha_k)$, corresponding to the eigenvalue $\zetadelta(\alpha_k)$ of multiplicity $m$.
    Due to orthonormality of $\tilde{h}_i$, we can write the orthonormal projection more explicitly as
    $$\Pi_{E_k} = \sum_{i=k}^{k+m-1}\frac{h_ih_i^\top}{\norm{2}{h_i}^2} = \sum_{i=k}^{k+m-1}\tilde{h}_i\tilde{h}_i^\top,$$
    and 
    \begin{equation}\label{eq:orthonormalbasisinfty}
        \Pi_{E^\infty_k} = \sum_{i=k}^{k+m-1}\frac{h^\infty_i{h^\infty_i}^\top}{\norm{2}{h^\infty_i}^2}=\sum_{i=k}^{k+m-1}h^\infty_i{h^\infty_i}^\top,
    \end{equation}
    where $h^\infty_k\dots, h^\infty_{k+j-1}$ is any choice of the orthonormal eigenbasis of $E^\infty_k$. 
    From \eqref{eq:hkhtildek}, it follows that 
    $$\sum_{i=k}^{k+m-1} \frac{1}{\norm{2}{h_i}^2} = m - \sum_{i=k}^{k+m-1}\tilde{h}_i^\top\frac{d}{d\lambda}R(\lambda_i^D)\tilde{h}_i\geq m - m\cdot\lambda_p\left(\frac{d}{d\lambda}R(\lambda_i^D)
    \right).$$
    Moreover, due to  \eqref{eq:rlambdaconv} and the continuity of eigenvalues, the RHS has a convergent limit
    \begin{equation}\label{eq:convlimit}
    m - m\cdot\lambda_p\left(\frac{d}{d\lambda}R(\lambda_i^D)\right) \asconv m - m \frac{1}{\zetadelta'(\alpha_k)}\lambda_p\left(\frac{d}{d\alpha}R^\infty(\alpha_k)\right).        
    \end{equation}
    Note that the matrix $\frac{d}{d\alpha}R^\infty(\alpha_k) = -\expt{\frac{ss^\top z^2}{(\alpha_k-z)^2}}$ is strictly negative definite for $\alpha_k>\tau$, which implies that $\lambda_p(\frac{d}{d\alpha}R^\infty(\alpha_k))<0$.
    As $\alpha_k>\lambdabardelta$, it holds that $\zetadelta'(\alpha_k)>0$ and the RHS of \Cref{eq:convlimit} is finite. This further implies that, for each $i$ s.t.\ $k\leq i\leq k+m-1$, it must hold
    \begin{equation}\label{eq:liminfhi}
        \liminf_{d\to\infty}\norm{2}{h_i}>0.
    \end{equation}
    Note that
    \begin{equation}\label{eq:hiinequality}
    \begin{split}
        \sum_{i=k}^{k+m-1}\abs{\inprod{v_i^D}{e_l^{(d)}}}^2 = \sum_{i=k}^{k+m-1}\abs{\inprod{h_i}{e_l^{(p)}}}^2 &= \sum_{i=k}^{k+m-1} \norm{2}{h_i}^2\abs{\inprod{\tilde{h}_i}{e_l^{(p)}}}^2\\
        &\geq \min_{t\in \{k, \ldots, k+m-1\}}\norm{2}{h_t}^2\sum_{i=k}^{k+m-1}\abs{\inprod{\tilde{h}_i}{e_l^{(p)}}}^2\\
        &=  \min_{t\in \{k, \ldots, k+m-1\}}\norm{2}{h_t}^2 \cdot {e_l^{(p)}}^\top \Pi_{E_k}{e_l^{(p)}}.
    \end{split}
    \end{equation}
    Let us pick $e_l^{(d)}$ such that $\Pi_{E_k^\infty}(e_l^{(p)})\neq 0$. Then, \eqref{eq:rlambdaconv} implies that 
    \begin{equation}\label{eq:pineq}
        \Pi_{E_k}(e_l^{(p)})\neq 0,
    \end{equation}
    for all $d$ large enough.  Finally, combining \eqref{eq:liminfhi}, \eqref{eq:hiinequality} and \eqref{eq:pineq} proves
    $$\liminf_{d\to\infty}\sum_{i=k}^{k+m-1}\abs{\inprod{v_i^D}{e_l^{(d)}}}^2>0,$$
    which gives the claim in \eqref{eq:liminfconv}.
    
    Next, we prove \Cref{eq:liminfconv_mult} for $m>1$. We assume, as in the statement, that $E_k^\infty$ is also the eigenspace corresponding to the $k$-th eigenvalue of $R^\infty(\alpha+\Delta)$  for any small enough $\Delta$. 
    For arbitrary eigenvectors $h_{i_1}$ and $h_{i_2}$ from $E_k^\infty$, it holds that 
    \begin{equation}\label{eq:equalderivative}
    \begin{split}
        {h_{i_1}^\infty}^\top\frac{d}{d\alpha}R^\infty(\alpha_k)h_{i_1}^\infty &= \lim_{\Delta\to 0} \frac{{h_{i_1}^\infty}^\top R^\infty(\alpha_k+\Delta)h_{i_1}^\infty-{h_{i_1}^\infty}^\top R^\infty(\alpha_k)h_{i_1}^\infty }{\Delta}\\
        &=\lim_{\Delta\to 0} \frac{{h_{i_2}^\infty}^\top R^\infty(\alpha_k+\Delta)h_{i_2}^\infty-{h_{i_2}^\infty}^\top R^\infty(\alpha_k)h_{i_2}^\infty}{\Delta}\\
        &={h_{i_2}^\infty}^\top\frac{d}{d\alpha}R^\infty(\alpha_k)h_{i_2}^\infty,
    \end{split}
    \end{equation}
     since ${h_{i_1}^\infty}^\top R^\infty(\alpha_k+\Delta)h_{i_1}^\infty = {h_{i_2}^\infty}^\top R^\infty(\alpha_k+\Delta)h_{i_2}^\infty$ for any small enough $\Delta$.
     Note that, for any $\epsilon$ and large enough $d$, it holds that
     \begin{equation}\label{eq:epsinequal}
         \norm{2}{\Pi_{E_k} - \Pi_{E_k^\infty}}<\epsilon.
     \end{equation}
     due to \eqref{eq:pieqkconv}. Let us now fix  $\tilde{h}_i$, for some $i\in\{k,\dots,k+m-1\}$. As we can choose any orthonormal basis when writing out $\Pi_{E_k^\infty}$ in $\eqref{eq:orthonormalbasisinfty}$, let us choose one such that $h_i^\infty = \frac{\Pi_{E_k^\infty}(\tilde{h}_i)}{\norm{2}{\Pi_{E_k^\infty}(\tilde{h}_i)}}$. Then, \eqref{eq:epsinequal} implies that 
     $$\norm{2}{\sum_{i=k}^{k+m-1}\tilde{h}_i\tilde{h}_i^\top - \sum_{i=k}^{k+m-1}h_i^\infty{h_i^\infty}^\top}<\epsilon.$$
     From the orthonormality of the chosen eigenbasis, it holds that
     \begin{equation*}
         \begin{split}
             \norm{2}{\tilde{h}_i - \Pi_{E_k^\infty}\tilde{h}_i} &= \norm{2}{(\Pi_{E_k} - \Pi_{E_k^\infty})(\tilde{h}_i)} \\
             &\leq \norm{2}{\Pi_{E_k} - \Pi_{E_k^\infty}} \norm{2}{\tilde{h}_i} \\
             &<\epsilon.
         \end{split}
     \end{equation*}
    This also implies $1+\epsilon>\norm{2}{\Pi_{E_k^\infty}\tilde{h}_i}\geq 1-\epsilon$, hence 
     \begin{equation}
         \begin{split}
             \norm{2}{\tilde{h}_i - h_i^\infty} &= \norm{2}{\tilde{h}_i - \frac{\Pi_{E_k^\infty}(\tilde{h}_i)}{\norm{2}{\Pi_{E_k^\infty}(\tilde{h}_i)}}} \\
             &= \norm{2}{\frac{\tilde{h}_i - \Pi_{E_k^\infty}(\tilde{h}_i)}{\norm{2}{\Pi_{E_k^\infty}(\tilde{h}_i)}}+\tilde{h}_i\frac{1-\norm{2}{\Pi_{E_k^\infty}(\tilde{h}_i)}}{\norm{2}{\Pi_{E_k^\infty}(\tilde{h}_i)}}}\\
             &\leq \frac{\epsilon}{1-\epsilon}+\frac{\epsilon}{1-\epsilon}<4\epsilon.
         \end{split} \notag 
     \end{equation}
    % As we have that $h_i^\infty\sim\tilde{h}_i$, 
    Thus, \eqref{eq:rlambdaconv} implies that, for large enough $d$,
     $$\norm{2}{\tilde{h}_i^\top\frac{d}{d\lambda}R(\lambda_k^D)\tilde{h}_i - \frac{1}{\zetadelta'(\alpha_k)}{h_i^\infty}^\top\frac{d}{d\alpha}R^\infty(\alpha_k)h_i^\infty}<c\cdot \epsilon,$$
     for some constant $c$ independent of $\epsilon$.
     Plugging in the expression \eqref{eq:equalderivative} makes $h_i^\infty$ not depend on $\tilde{h}_i$ anymore, resulting in 
     \begin{equation}\label{eq:normconv}
         \tilde{h}_i^\top\frac{d}{d\lambda}R(\lambda_k^D)\tilde{h}_i\asconv \frac{1}{\zetadelta'(\alpha_k)}{h_{l}^\infty}^\top\frac{d}{d\alpha}R^\infty(\alpha_k)h_{l}^\infty,
     \end{equation}
     for an arbitrary unit norm eigenvector $h_{l}^\infty\in E^\infty_k$.
     Finally, combining \eqref{eq:hkhtildek} and \eqref{eq:normconv} with \eqref{eq:pieqkconv} implies 
     $$\sum_{i=k}^{k+m-1}\tilde{h}_i\tilde{h}_i^\top\asconv \frac{\zetadelta'(\alpha_k) \sum_{i=k}^{k+m-1}h_i^\infty {h_i^\infty}^\top}{\zetadelta'(\alpha_k)+{h_k^\infty}^\top \frac{d}{d\alpha}R^\infty(\alpha_k)h_k^\infty},$$
     which proves \Cref{eq:liminfconv_mult}.

     Finally, we prove the converse statement in \Cref{eq:liminfconv-conv}. Let us assume that $\zetadelta(\lambdabardelta)\neq \lambda_l^{a^\infty}$ for all $l$. This is without loss of generality, as justified at the end of the argument. %Notice that this assumption is always fulfilled for $i=1$. 
     By assumption, we have 
    \begin{equation}\label{eq:bulkconvthm}
    \lambda_i^D\asconv \zetadelta(\lambdabardelta).
    \end{equation}
    Thus, for $n$ large enough, it follows that $\lambda_i^D\notin \Lambda^a$, since $\lambda_l^a\asconv \lambda_l^{a^\infty}$. We denote by $k\in[i]$ the index such that $\lambda_{k}^a<\lambda_i^D<\lambda_{k-1}^a$, with the abuse of notation that $\lambda_0^a=+\infty$. By \Cref{prop:eigvalrec}, it holds that \begin{equation}\label{eq:ktheigenvec}
        \lambda_i^D=L_{k}(P - q(a-\lambda_i^DI_{p})^{-1}q^\top).
    \end{equation}
    Recall that the eigenvector condition for $v_i^D$ is
    $$ D v_i^D = \matrix{ a & q^\top\\
                        q & P}\matrix{h_i \\ g_i} = \lambda^D_i \matrix{h_i \\ g_i}.$$
    Splitting this equation into $p$ and $d-p$ coordinates gives
    \begin{align}
        a h_i + q^\top g_i &= \lambda^D_i h_i,\label{eq:eig11}\\
        q h_i + P g_i &= \lambda^D_i g_i.\label{eq:eig12}
    \end{align}
    Solving \eqref{eq:eig11} gives $$h_i = -(a-\lambda_i^DI_{p})^{-1}q^\top g_i,$$
    where $(a-\lambda_i^DI_{p})$ is invertible as $\lambda_i^D\notin \Lambda^a$.
    Substituting in \eqref{eq:eig12} yields
        $$Pg_i-q(a-\lambda_i^DI_{d-p})^{-1}q^\top g_i=\lambda_i^Dg_i.$$
    Let us denote by $\tilde{g}_i=\frac{g_i}{\norm{2}{g_i}}$ the unit norm eigenvector of $P - q(a-\lambda_i^DI_{p})^{-1}q^\top$ corresponding to the eigenvalue $\lambda_i^D$, and also define $\tilde{h}_i := -(a-\lambda_i^DI_{p})^{-1}q^\top\tilde{g}_i$. By \eqref{eq:ktheigenvec}, it holds that $\tilde{g}_i$ is the eigenvector corresponding to the $k$-th eigenvalue of the matrix $P - q(a-\lambda_i^DI_{p})^{-1}q^\top$.
    Moreover, $\tilde{h}_i$ and $\tilde{g}_i$ satisfy equations \eqref{eq:eig11} and \eqref{eq:eig12}, so $\tilde{v}_i^D = \matrix{\tilde{h}_i\\ \tilde{g}_i}$ is aligned with an eigenvector corresponding to eigenvalue $\lambda_i^D$. However, $\tilde{v}_i^D$  does not necessarily have unit norm. It holds that
    $$\matrix{h_i\\ g_i} = v_i^D = \frac{\tilde{v}_i^D}{\norm{2}{\tilde{v}_i^D}} = \frac{\matrix{\tilde{h}_i\\ \tilde{g}_i}}{\sqrt{\tilde{g}_i^\top\tilde{g}_i+\tilde{h}_i^\top\tilde{h}_i}}=\frac{\matrix{\tilde{h}_i\\ \tilde{g}_i}}{\sqrt{1+\tilde{g}_i^\top q(a-\lambda_i^DI_{p})^{-2}q^\top\tilde{g}_i}},$$ from which follows that
    $$h_i = \frac{\tilde{h}_i}{\sqrt{1+\tilde{g}_i^\top q(a-\lambda_i^DI_{p})^{-2}q^\top\tilde{g}_i}}.$$
    Moreover, notice that  
    \begin{align}
        \norm{2}{h_i}^2 &= \frac{\tilde{g}_i^\top q(a-\lambda_i^DI_{p})^{-2}q^\top\tilde{g}_i}{1+\tilde{g}_i^\top q(a-\lambda_i^DI_{p})^{-2}q^\top\tilde{g}_i} = 1 - \frac{1}{1+\tilde{g}_i^\top q(a-\lambda_i^DI_{p})^{-2}q^\top\tilde{g}_i}.
    \end{align}
    The term at the denominator can be simplified as
    \begin{equation}\label{eq:Lancaster}
    \begin{split}        
        \tilde{g}_i^\top q(a-\lambda_i^DI_{p})^{-2}q^\top\tilde{g}_i &= \tilde{g}_i^\top\paren{ \frac{d}{d\lambda}q(a-\lambda I_{p})^{-1}q^\top}(\lambda_i^D)\tilde{g}_i\\
        & = \paren{\frac{d}{d\lambda}\lambda_k(q(a-\lambda I_{p})^{-1}q^\top)}(\lambda_i^D)\\
        & = \frac{d}{d\lambda} L_k(\lambda_i^D).
            \end{split}
    \end{equation}
    Here, when the eigenvalue $\lambda_i^D$ is simple, the second equality follows from \cite[Theorem 5]{lancaster1964eigenvalues} and the fact that $\tilde{g}_i$ is the eigenvector corresponding to the $k$-th eigenvalue of the matrix $q(a-\lambda_i^DI_{p})^{-1}q^\top$. The case in which the eigenvalue $\lambda_i^D$ has multiplicity $m>1$ is handled similarly via  \cite[Theorem 7]{lancaster1964eigenvalues}, and it is discussed at the end of the argument. 
    
    Note that, by definition, it holds 
    $$\frac{d}{d\lambda} L_k(\lambda_i^D) = \frac{d}{d\lambda} \tildeLi(\lambda_i^D).$$
    %as $\tildeLi(\lambda_i^D) = L_k(\lambda_i^D).$
    We will prove that 
    \begin{equation}\label{eq:lastclaim}
    \frac{d}{d\lambda}\tildeLi(\lambda_i^D) \asconv 0,    
    \end{equation}
    which implies $\norm{2}{h_i}^2\asconv 0$ and hence gives the desired statement in \Cref{eq:liminfconv-conv}. Recall that the assumption underlying \Cref{eq:liminfconv-conv} is that \Cref{eq:bulkconvthm} holds. This happens if either the $\alpha_i$ associated to $\lambda_i^D$ is s.t.\ $\alpha_i\le \lambdabardelta$ or $i$ exceeds the number of solutions of \Cref{eq:master_eq2} (i.e., $\lambda_i^D$ is not associated to any solution of \Cref{eq:master_eq2}). We handle these two cases separately.

    \emph{Case 1: $\lambda\in]\lambda_i^{a^\infty}, t_i^\infty[$}. This corresponds to the case in which $\alpha_i\le \lambdabardelta$.
    Applying \Cref{prop:tildeLiconv}, we have 
    $$\tildeLi(\lambda) \asconv \tildeLiinfty(\lambda) = \zetadelta((\lambda_i^\infty)^{-1}(\lambda)).$$
    By combining \Cref{eq:bulkconvthm} and Lemma A.2 in \cite{lu2020phase}, it follows that
     $$\frac{d}{d\lambda}\tildeLi(\lambda_i^D) \asconv \zetadelta'((\lambda_i^\infty)^{-1}(\zetadelta(\lambdabardelta)))\cdot \frac{1}{\lambda_i'^\infty((\lambda_i^\infty)^{-1}(\zetadelta(\lambdabardelta)))},$$
    where the denominator is non-zero, as $\lambda_i^\infty$ is strictly decreasing. As $(\lambda_i^\infty)^{-1}(\zetadelta(\lambdabardelta))=\alpha_i\le \lambdabardelta$, the derivative $\zetadelta'$ computed at that point is $0$, which gives \Cref{eq:lastclaim}.
    
\emph{Case 2: $\lambda \in ]t_i^\infty,+\infty[$.} This corresponds to the case in which $\lambda_i^D$ is not associated to any $\alpha_i$. Applying again \Cref{prop:tildeLiconv}, we have 
    $$\tildeLi(\lambda) \asconv \zetadelta(\lambdabardelta).$$
    Then, again by \cite[Lemma A.2]{lu2020phase} it follows that
    $$\frac{d}{d\lambda}\tildeLi(\lambda) \asconv \frac{d}{d\lambda}(\zetadelta(\lambdabardelta)) = 0,$$
    as the LHS is a constant function that does not depend on $\lambda$.

We conclude by handling the two special cases mentioned during the argument above. 
    If $\lambda_i^D$ has a multiplicity $m>1$, then the $m$ derivatives $\frac{d}{d\lambda} L_k(\lambda_i^D)$ corresponding to the $m$ eigenvalues vanish using the same argument. An application of \cite[Theorem 7]{lancaster1964eigenvalues} gives that such derivatives are the eigenvalues of the $m\times m$ matrix obtained by replacing $\tilde g_i$ in the LHS of \Cref{eq:Lancaster} with a basis of the $m$-dimensional eigenspace associated to $\lambda_i^D$. As the eigenvalues of such matrix vanish, so does its trace and therefore the $m$ norms $\|h_i\|_2$.  

    Finally, if  $\zetadelta(\lambdabardelta)=\lambda_l^{a^\infty}$ for some $l$, we can just add a vanishing perturbation to the matrix $a$ such that $\zetadelta(\lambdabardelta)\neq \lambda_l^{a^\infty}$. Applying the proved result, the overlaps vanish for any such small perturbation, and due to continuity, taking the perturbation to $0$ would give the claim.
\end{proof}

%\subsection{The case with multiplicity one ($m=1$)}

%\begin{proposition}
%\label{prop:mastereigenvectors}
%    Let $\alpha_k$ be the $k$-th solution to the equation
%    $$\det\paren{\zetadelta(\alpha)I-a^\infty+R^\infty(\alpha)}=0,$$
%    with multiplicity one. If $\alpha_k>\lambdabardelta$, then for $j\in[p]$ it holds that
%    $$\abs{\inprod{v_k^D}{e_j^{(d)}}}^2\asconv \frac{\zetadelta'(\alpha_k)\cdot \abs{\inprod{h_k^\infty}{e_j^{(p)}}}^2}{\zetadelta'(\alpha_k)+{h_k^\infty}^\top \frac{d}{d\alpha}R^\infty(\alpha_k)h_k^\infty},$$
%    where $h_k^\infty$ is the unit norm eigenvector corresponding to the eigenvalue $\zetadelta(\alpha_k)$ of the matrix $R^\infty(\alpha_k)$. 
%\end{proposition}
%\begin{proof}
 
%\end{proof}

\section{Invariance of the eigenspace for permutation-invariant link functions}\label{app:invariance}

\begin{proposition}\label{prop:invlink}
    If the link function $q$ is permutation invariant in $m$ coordinates, then the matrix $R^\infty(\alpha)$  has eigenspaces that do not change with $\alpha$, of combined dimension $m-1$.
\end{proposition}
\begin{proof}
    Let us denote by $w_i$ the $i$-th column of the matrix $\wt{W}^*$ as in \eqref{eqn:W*}, representing the top-$p$ entries of the reparametrized signal.
    Without loss of generality, we can assume that $q$ is permutation invariant in the first $m$ coordinates, i.e.,
    $$q(t_1,\dots, t_m,t_{m+1}\dots,t_p,\epsilon) = q(t_{\pi(1)},\dots, t_{\pi(m)},t_{m+1}\dots,t_p,\epsilon),$$
for any permutation $\pi:[m]\to[m]$.
    Let $E$ be the span of $\{w_i-w_{i+1}\,:\, i\in[m-1]\}$. Note that $E$ has dimension $m-1$, due to the linear independence of the signals $w_i$. We will prove that $E$ is a direct sum of eigenspaces of $R^\infty(\alpha)$, neither of which depends on $\alpha$.
    
    Let us define $u_i$ as the image of $w_i-w_{i+1}$ under $R^\infty(\alpha)$, that is
    \begin{align*}
        u_i \coloneq R^\infty(\alpha)(w_i-w_{i+1}) = \alpha \expt{ \frac{s \inprod{s}{w_i-w_{i+1}}\cT(y)}{\alpha -\cT(y)}},
    \end{align*}
    for $i\in[m-1]$. Let us denote by $x_i$ the component of $u_i$ that is orthogonal to $w_i$ and $w_{i+1}$, i.e., $x_i \coloneq \Pi_{\{w_i,w_{i+1}\}^\perp}u_i$. Then, it holds that
    \begin{equation}\label{eq:uidecomp}
        u_i=a_1 w_i+ a_2 w_{i+1} + x_i,
    \end{equation}
    for some coefficients $a_1$ and $a_2$. We will first prove that $x_i=0$. Towards that end, let $\cS_i:\bbR^p\to \bbR^p$ be an isometric reflection that sends $w_i$ to $w_{i+1}$, $w_{i+1}$ to $w_i$, and keeps $w_j$ fixed for $j\notin \{i,i+1\}$. Such a reflection exists due to the assumed linear independence of the $w_i$'s. Notice that the normal distribution in $\bbR^p$ is invariant to the transformation $\cS_i$. Thus, it follows that
    \begin{align*}
        \inprod{u_i}{x_i} &=\alpha \expt{ \frac{\inprod{s}{x_i} \inprod{s}{w_i-w_{i+1}}\cT(y)}{\alpha -\cT(y)}}\\
        &=\alpha \expt{ \frac{\inprod{s}{x_i} \inprod{s}{w_i}\cT(y)}{\alpha -\cT(y)}} - \alpha \expt{ \frac{\inprod{s}{x_i} \inprod{s}{w_{i+1}}\cT(y)}{\alpha -\cT(y)}}\\
        &=\alpha \expt{ \frac{\inprod{\cS_i s}{x_i} \inprod{\cS_i s}{w_i}\cT(q((\wt{W}^*)^\top \cS_i s, \eps))}{\alpha -\cT(q((\wt{W}^*)^\top \cS_i s, \eps))}} - \alpha \expt{ \frac{\inprod{s}{x_i} \inprod{s}{w_{i+1}}\cT(y)}{\alpha -\cT(y)}}\\
        &=\alpha \expt{ \frac{\inprod{s}{x_i} \inprod{s}{w_{i+1}}\cT(q((\wt{W}^*)^\top s, \eps))}{\alpha -\cT(q((\wt{W}^*)^\top s, \eps))}} - \alpha \expt{ \frac{\inprod{s}{x_i} \inprod{s}{w_{i+1}}\cT(y)}{\alpha -\cT(y)}}\\
        &= 0,
    \end{align*}
    due to the permutation invariance of $q$, the fact that $\cS_i x_i = x_i$ and $\cS_i w_i = w_{i+1}$. Therefore, it must be that $x_i=0$.    
    Moreover,
   \begin{align*}
       \inprod{u_i}{w_i}&=\alpha \expt{ \frac{\inprod{s}{w_i} \inprod{s}{w_i-w_{i+1}}\cT(y)}{\alpha -\cT(y)}}\\
       &=\alpha \expt{ \frac{\inprod{\cS_i s}{w_i} \inprod{\cS_i s}{w_i-w_{i+1}}\cT(q((\wt{W}^*)^\top \cS_i s)}{\alpha -\cT(q((\wt{W}^*)^\top \cS_i s)}}\\
       &=\alpha \expt{ \frac{\inprod{s}{w_{i+1}} \inprod{s}{w_{i+1}-w_i}\cT(q((\wt{W}^*)^\top  s)}{\alpha -\cT(q((\wt{W}^*)^\top s)}}\\
       &= \alpha \expt{ \frac{\inprod{s}{w_{i+1}} \inprod{s}{w_{i+1}-w_i}\cT(y)}{\alpha -\cT(y)}} \\
       &= -\inprod{u_i}{w_{i+1}}.
   \end{align*}
    Combining this with the decomposition in \eqref{eq:uidecomp} implies that
    $a_1+a_2\inprod{w_i}{w_{i+1}} = -a_1\inprod{w_{i+1}}{w_i} - a_2.$
    As by assumption the signals $w_i$ and $w_{i+1}$ are linearly independent, it cannot be that $\inprod{w_i}{w_{i+1}}=-1$, so it must be that $a_1=-a_2$. This proves that $w_i-w_{i+1}$ are indeed eigenvectors for every $\alpha$. 
    By definition, it holds that
    $$E=\bigoplus_{i=1}^{m-1}\operatorname{span}\{u_i\}.$$
    Thus, it is left to prove that any pair of eigenvectors $u_i$ and $u_j$, for $i,j\in[m-1]$, either have the same eigenvalue for all $\alpha$, or for no $\alpha$. We denote by $\lambda_{u_i}(\alpha)$ the eigenvalue that corresponds to the eigenvector $u_i$. Similar to before, let $\cS^{(i)}:\bbR^p\to \bbR^p$ be an isometric reflection that sends $w_i$ to $w_{i+2}$, $w_{i+2}$ to $w_i$, and keeps $w_j$ fixed for $j\notin \{i,i+2\}$. Such a reflection exists due to the assumed linear independence of the $w_i$'s. Also, the normal distribution is invariant to the transformation $\cS^{(i)}$. Then, it follows that
    \begin{align*}
        \lambda_{u_i}(\alpha)\cdot \norm{2}{w_i-w_{i+1}}^2 = \inprod{u_i}{w_i-w_{i+1}}&= \alpha \expt{ \frac{\inprod{s}{w_i-w_{i+1}}^2\cT(y)}{\alpha -\cT(y)}}\\
                      &= \alpha \expt{ \frac{\inprod{\cS^{(i)} s}{w_i-w_{i+1}}^2\cT(q((\wt{W}^*)^\top \cS^{(i)} s, \eps))}{\alpha -\cT(q((\wt{W}^*)^\top \cS^{(i)} s, \eps))}}\\
                      &= \alpha \expt{ \frac{\inprod{ s}{w_{i+2}-w_{i+1}}^2\cT(q((\wt{W}^*)^\top s, \eps))}{\alpha -\cT(q((\wt{W}^*)^\top s, \eps))}}\\
                      &=\inprod{u_{i+1}}{w_{i+1}-w_{i+2}} = \lambda_{u_{i+1}}(\alpha)\cdot \norm{2}{w_{i+2}-w_{i+1}}^2.
    \end{align*}
    This proves that $\lambda_{u_i}(\alpha)/\lambda_{u_{i+1}}(\alpha)$ does not depend on $\alpha$ and, hence, $\lambda_{u_i}(\alpha)/\lambda_{u_{j}}(\alpha)$ does not depend on $\alpha$ for all $i,j\in[m-1]$, implying  the existence of an eigenspace of dimension $m-1$ that does not change with $\alpha$.
\end{proof}

\section{Proof of \Cref{thm:opt}}\label{app:pfopt}

\begin{proof}
%\paragraph{Proof of \Cref{thm:opt}}
%We want to find a preprocessing function $\cT$ for a given $\delta$ satisfying \Cref{asmp:proportional}, such that the spectral estimator $\wh{W}^{\mathrm{s}} = \matrix{ v_1^D, & \cdots, & v_p^D } $  weakly recover the signals, as in the \Cref{def:weak_rec}. It is not too hard to see, that for the signals of the form in \eqref{eqn:W*},
%this is equivalent to finding a $\cT$ such that
%\begin{align*}
%        \max_{j\in[p]} \brace{ \liminf_{d\to\infty} \sum_{i=1}^p \abs{\inprod{v_i^D}{e_j}}^2 } &> 0 . 
%\end{align*}
%As a consequence of \Cref{thm:main}, it would be sufficient to prove that there exists a solution $\alpha$ to the equation \eqref{eq:master_eq2} such that $\alpha>\lambdabardelta$.
By \Cref{thm:main}, we have that  
$$
\max_{j\in[p]} \lim_{d\to\infty} \sum_{i=1}^{p}\abs{\inprod{v_i^D}{e_l^{(d)}}}^2 > 0
$$
holds if and only if there exists $\alpha> \lambdabardelta$ that solves \Cref{eq:master_eq2}.

Note that $\zetadelta(\alpha)$ is a strictly monotone function for $\alpha>\lambdabardelta$, and it is constant for $\alpha\leq \lambdabardelta$. Thus, the existence of $\alpha$ solving \Cref{eq:master_eq2} s.t.\ $\alpha>\lambdabardelta$ is equivalent to $\zetadelta'(\alpha_1)>0$, where $\alpha_1$ is the largest solution of \eqref{eq:master_eq2}.
Thus, $\alpha_1$ is the largest solution to
 $$\det\paren{\zetadelta(\alpha)I_p-R^\infty(\alpha)}=0,$$
 or equivalently
 $$\lambda_1 (R^\infty(\alpha)) = \zetadelta(\alpha).$$
 This means that, for \eqref{eq:master_eq2} to have solutions larger than $\lambdabardelta$, there has to exist $\alpha_1>\tau$ such that
 \begin{align*}
     \max_{\norm{2}{u}=1} u^\top R^\infty(\alpha_1)u &= \zetadelta(\alpha_1),\\
     \zetadelta'(\alpha_1) &>0,
 \end{align*}
or equivalently
 \begin{align}\label{eq:conditionlambdabar}
     \max_{\norm{2}{u}=1} u^\top R^\infty(\lambdabardelta)u &> \zetadelta(\lambdabardelta).
 \end{align}
 This follows from the fact that $\zetadelta(\alpha)$ is strictly increasing for $\alpha>\lambdabardelta$, and $u^\top R^\infty(\alpha)u$ is strictly decreasing, as proved in \eqref{eq:matrixRdecreases}.
 Recall that $\lambdabardelta$ is defined as the unique point that satisfies $ \psidelta'(\lambdabardelta) =0$. This is equivalent to
 \begin{equation}\label{eq:lambdabardef}
     \expt{\frac{z^2}{(\lambdabardelta-z)^2}}=\frac{1}{\delta}. 
 \end{equation}
By definition, it holds that
$$R^\infty(\lambdabardelta) =
\expt{\frac{\lambdabardelta z}{\lambdabardelta-z}s^\top s }.$$
Therefore, the condition in \eqref{eq:conditionlambdabar} becomes 
\begin{equation}\label{eq:conditioninequality}
    \max_{\norm{2}{u}=1} \expt{\frac{\lambdabardelta z}{\lambdabardelta-z}\inprod{s}{u}^2 } > \lambdabardelta\paren{\frac{1}{\delta}+\expt{\frac{z}{\lambdabardelta-z}}}.
\end{equation}
Note that $D_n$ and $D_n'\coloneqq D_n/\beta$ have the same principal eigenvector for an arbitrary scalar $\beta>0$ so, without loss of generality, we can take $\lambdabardelta=1$. This transforms \eqref{eq:conditioninequality} and \eqref{eq:lambdabardef} into 
\begin{align*}
    \max_{\norm{2}{u}=1} \expt{\frac{z(\inprod{s}{u}^2-1)}{1-z}} &> \frac{1}{\delta},\\
    \expt{\frac{z^2}{(1-z)^2}}&=\frac{1}{\delta}.
\end{align*}

We turn our attention to finding the critical threshold $\delta_c$ such that no preprocessing function $\cT$ exists which would satisfy these equations. To start, plugging in the definition of the expectation we get 
\begin{equation}\label{eq:condrew}
\begin{split}
    \max_{\norm{2}{u}=1}  \int_{\bbR}\frac{\cT(y)}{1-\cT(y)} \expt[s]{p(y \mathrel{\vert} s)\cdot(\inprod{s}{u}^2-1)} dy &> \frac{1}{\delta},\\
    \int_{\bbR}\paren{\frac{\cT(y)}{1-\cT(y)}}^2 \expt[s]{p(y \mathrel{\vert} s)} dy &=\frac{1}{\delta}.
\end{split}
\end{equation}
Let us denote by $f(y)\coloneqq \frac{\cT(y)}{1-\cT(y)}$. Note that %it holds that
\begin{align}
    \frac{1}{\delta}&<\max_{\norm{2}{u}=1}\int_{\bbR}f(y) \expt[s]{p(y \mathrel{\vert} s)\cdot(\inprod{s}{u}^2-1)} dy \nonumber\\
    &= \max_{\norm{2}{u}=1}
    \int_{\bbR}f(y) \sqrt{\expt[s]{p(y \mathrel{\vert} s)}} \frac{\expt[s]{p(y \mathrel{\vert} s)\cdot(\inprod{s}{u}^2-1)}}{\sqrt{\expt[s]{p(y \mathrel{\vert} s)}}}dy \nonumber\\
    &\leq \max_{\norm{2}{u}=1}\sqrt{\int_{\bbR} f^2(y) \expt[s]{p(y \mathrel{\vert} s)}dy}\sqrt{\int_{\bbR}\frac{\paren{\expt[s]{p(y \mathrel{\vert} s)\cdot(\inprod{s}{u}^2-1)}}^2}{\expt[s]{p(y \mathrel{\vert} s)}}dy}\label{ineq:Hoelders}\\
    &= \max_{\norm{2}{u}=1} \frac{1}{\sqrt{\delta}} \sqrt{\int_{\bbR}\frac{\paren{\expt[s]{p(y \mathrel{\vert} s)\cdot(\inprod{s}{u}^2-1)}}^2}{\expt[s]{p(y \mathrel{\vert} s)}}dy},\nonumber
\end{align}
where the third line follows from H\"older's inequality and the fourth line from the second condition in \eqref{eq:condrew}. This means that, regardless of which preprocessing function $\cT$ was chosen, if it satisfies \eqref{eq:condrew}, then it must hold that 
$$ \frac{1}{\delta} < \max_{\norm{2}{u}=1}\int_{\bbR}\frac{\paren{\expt[s]{p(y \mathrel{\vert} s)\cdot(\inprod{s}{u}^2-1)}}^2}{\expt[s]{p(y \mathrel{\vert} s)}}dy =:\frac{1}{\delta_t}.$$
This directly implies that %In other words, for any $\delta\leq \delta_t$, the set $\sT_\delta=\emptyset$.
%By the definition of $\delta_c$ in \eqref{eq:deltacdef}, it follows that
\begin{equation}\label{eq:ineqdelta}
    \delta_c \geq \delta_t.
\end{equation}
Now, let us prove that, for any $\delta>\delta_t$, there exists a preprocessing function that achieves weak recovery. Towards this end, we turn our attention to when equality holds in \eqref{ineq:Hoelders}, as this gives us the preprocessing function that exactly matches the upper bound. Namely, this is true if and only if almost everywhere
$$ f^2(y) \expt[s]{p(y \mathrel{\vert} s)} = c\cdot\frac{\paren{\expt[s]{p(y \mathrel{\vert} s)\cdot(\inprod{s}{u_c}^2-1)}}^2}{\paren{\expt[s]{p(y \mathrel{\vert} s)}}^2},$$
 where $u_c=\argmax_{\tilde u}\int_{\bbR}\frac{\paren{\expt[s]{p(y \mathrel{\vert} s)\cdot(\inprod{s}{\tilde u}^2-1)}}^2}{\expt[s]{p(y \mathrel{\vert} s)}}dy$ and $c$ is some constant.
Thus, we have that
$$ f(y) = \sqrt{c}\frac{\expt[s]{p(y \mathrel{\vert} s)\cdot(\inprod{s}{u_c}^2-1)}}{\expt[s]{p(y \mathrel{\vert} s)}},$$
which in turn gives the choice
\begin{equation}\label{eq:taustardelta}
    \bar\cT:= \frac{\sqrt{c}  \taustar}{1-(1-\sqrt{c})  \taustar},
\end{equation}
with
$$\taustar(y) = 1 - \frac{\expt[s]{p(y \mathrel{\vert} s)}}{\expt[s]{p(y \mathrel{\vert} s)\cdot\inprod{s}{u_c}^2}}.$$
The second condition in \eqref{eq:condrew} determines $c$. Namely,
\begin{equation}\label{eq:deltadletat}
    \frac{1}{\delta} = \int_{\bbR}\paren{\frac{\bar\cT(y)}{1-\bar\cT(y)}}^2 \expt[s]{p(y \mathrel{\vert} s)}dy = \int_{\bbR} c\cdot\frac{\paren{\expt[s]{p(y \mathrel{\vert} s)\cdot(\inprod{s}{u_c}^2-1)}}^2}{\expt[s]{p(y \mathrel{\vert} s)}} dy = c \frac{1}{\delta_t}.
\end{equation}
Therefore,  $c = \frac{\delta_t}{\delta}$. 
\eqref{eq:deltadletat} immediately proves that the second condition in \eqref{eq:condrew} is satisfied for $\bar\cT$. The first condition in \eqref{eq:condrew} is also satisfied since
\begin{align*}
    \max_{\norm{2}{u}=1}  \int_{\bbR}\frac{\bar\cT}{1-\bar\cT} \expt[s]{p(y \mathrel{\vert} s)\cdot(\inprod{s}{u}^2-1)} dy & \ge \sqrt{c} \cdot \int_{\bbR}\frac{\paren{\expt[s]{p(y \mathrel{\vert} s)\cdot(\inprod{s}{u_c}^2-1)}}^2}{\expt[s]{p(y \mathrel{\vert} s)}}dy\\
    &= \sqrt{\frac{\delta_t}{\delta}} \cdot \frac{1}{\delta_t}>\frac{1}{\delta}.
\end{align*}
Thus, $\bar\cT$ achieves weak recovery, hence %\cT\in\sT_\delta$ and $\sT_\delta\neq\emptyset$ for $\delta>\delta_t$. It follows that
$$\delta_c\leq\delta_t,$$
which combined with \eqref{eq:ineqdelta} proves $\delta_c=\delta_t$. This also implies that the expression in \Cref{eq:taustardelta} coincides with  $\cT_\delta^*$ as defined in  \Cref{eq:opt}.

Lastly, we need to verify that $\taustardelta$ satisfies \Cref{asmp:T}.
Let us first prove that $\taustardelta(y)$ is bounded. Since $\taustar(y)\leq 1$, it follows that
$$\sqrt{\delta}-(\sqrt{\delta}-\sqrt{\delta_c})\cdot \taustar(y)\geq \sqrt{\delta_c},$$
and
$$\cT_\delta^*(y) = \frac{\sqrt{\delta_c} \cdot \taustar(y)}{\sqrt{\delta}-(\sqrt{\delta}-\sqrt{\delta_c})\cdot \taustar(y)} \leq \frac{\sqrt{\delta_c}\taustar(y)}{\sqrt{\delta_c}}\leq 1.$$
Furthermore, for $\taustar(y)\neq 0$, we have
$$\cT_\delta^*(y) = \frac{\sqrt{\delta_c}}{\frac{\sqrt{\delta}}{\taustar(y)}-(\sqrt{\delta}-\sqrt{\delta_c})},$$
and it holds that $\frac{\sqrt{\delta}}{\taustar(y)}-(\sqrt{\delta}-\sqrt{\delta_c})\in ]-\infty,-(\sqrt{\delta}-\sqrt{\delta_c})[\ \cup\ ]\sqrt{\delta_c},+\infty[.$
Thus, for $\taustar(y)\neq 0$, 
$$\cT_\delta^*(y)\geq-\frac{\sqrt{\delta_c}}{\sqrt{\delta}-\sqrt{\delta_c}},$$
whereas $\taustardelta(y)=0$ for $\taustar(y)=0$. This proves that  $\taustardelta(y)$ is bounded.

Finally, by contradiction, suppose that $\prob{\taustardelta(y) = 0 }=1$. Note that $\taustardelta(y)=0$ if and only if $\taustar(y)=0$. This holds whenever  
\begin{equation}\label{eq:contr}
\prob{\expt[s]{p(y \mathrel{\vert} s)} = \expt[s]{p(y \mathrel{\vert} s)\cdot\inprod{s}{u_c}^2}}=1.    
\end{equation}
However, \Cref{eq:contr} implies that $$\prob{\expt[s]{p(y \mathrel{\vert} s)\cdot(\inprod{s}{u_c}^2-1)}=0}=1,$$
giving that $\delta_c=+\infty$, for which the statement of the theorem trivially holds. Consequently, $\prob{\taustar(y)=0}\neq 1$ and the proof is complete. 
\end{proof}

\section{Optimal preprocessing for the numerical experiments}
\label{sec:calc}

\subsection{$q(s) = \prod_{i=1}^2 s_i$}\label{subsec:prod}

In the numerical experiment, we employ the function $\taustar(y)$ in place of $\taustar_\delta(y)$, as the latter is introduced for technical reasons. Recall that %First, we have that 
$$\taustar(y) \coloneqq 1 - \frac{\expt[s]{p(y \mathrel{\vert} s)}}{\expt[s]{p(y \mathrel{\vert} s)\cdot\inprod{s}{u_c}^2}}.$$  
We calculate term by term. First, we have
\begin{align*}
    \expt[s]{p(y \mathrel{\vert} s)} &= \frac{1}{2\pi}\iint_{\bbR^2}e^{-\frac{s_1^2}{2}}e^{-\frac{s_2^2}{2}}\delta(y-s_1s_2)ds_1ds_2\\
    &= \frac{1}{2\pi}2\int_{0}^\infty\int_{0}^\infty e^{-\frac{s_1^2}{2}}e^{-\frac{s_2^2}{2}}\delta(y-s_1s_2)ds_1ds_2,
\end{align*}
where we have assumed that $y>0$ (similar passages hold for $y\leq0$). The constant $2$ pops out, since there are two symmetric cases for $y>0$, namely $s_1,s_2>0$ and $s_1,s_2<0$. Continuing with the change of variable $x_1=s_1s_2$ and $x_2=\frac{s_1}{s_2}$, we have
\begin{align*}
    \frac{1}{2\pi}2\int_{0}^\infty\int_{0}^\infty e^{-\frac{s_1^2}{2}}e^{-\frac{s_2^2}{2}}\delta(y-s_1s_2)ds_1ds_2 &= \frac{1}{\pi}\int_{0}^\infty\int_{0}^\infty e^{-\frac{x_1x_2}{2}}e^{-\frac{x_1/x_2}{2}}\delta(y-x_1)\frac{1}{2x_2}dx_1dx_2\\ 
    &=\frac{1}{\pi}\int_{0}^\infty e^{-y\paren{\frac{x_2}{2}+\frac{1}{2x_2}}}\frac{1}{2x_2}dx_2\\
    &= \frac{1}{2\pi} \int_{-\infty}^\infty e^{-y\cosh(t)}dt\\
    &= \frac{1}{\pi} \int_{0}^\infty e^{-y\cosh(t)}dt\\
    &= \frac{K_0(\abs{y})}{\pi},
\end{align*}
where we did the change of variable $e^t = \frac{1}{x_2}$, and $K_0$ is the modified Bessel function of the second kind. %Notice that this exactly fits the \href{https://mathworld.wolfram.com/NormalProductDistribution.html}{formula} for the PDF of a product of two normal independent gaussians with variance $1$ which $y=s_1s_2$ exactly is. 
Let us now calculate the second term for the choice $u_c=\frac{1}{\sqrt{2}}(1,1)$, which can be verified to be optimal, 
\begin{align*}
    \expt[s]{p(y \mathrel{\vert} s)\inprod{s}{u_c}^2} &= \frac{1}{2\pi}\iint_{\bbR^2}e^{-\frac{s_1^2}{2}}e^{-\frac{s_2^2}{2}}\delta(y-s_1s_2)\inprod{s}{u_c}^2ds_1ds_2\\
    &= \frac{1}{2\pi}2\int_{0}^\infty\int_{0}^\infty e^{-\frac{s_1^2}{2}}e^{-\frac{s_2^2}{2}}\delta(y-s_1s_2)\frac{1}{2}(s_1^2+s_2^2+2s_1s_2)ds_1ds_2,
\end{align*}
where, as before, we have assumed that $y>0$ (similar passages hold for $y\le 0$). Continuing with the change of variable $x_1=s_1s_2$ and $x_2=\frac{s_1}{s_2}$, we have
\begin{align*}
    \expt[s]{p(y \mathrel{\vert} s)\inprod{s}{u_c}^2} &=\frac{1}{\pi}\int_{0}^\infty\int_{0}^\infty e^{-\frac{s_1^2}{2}}e^{-\frac{s_2^2}{2}}\delta(y-s_1s_2)\frac{1}{2}(s_1^2+s_2^2+2s_1s_2)ds_1ds_2\\
    &= \frac{1}{\pi}\int_{0}^\infty\int_{0}^\infty e^{-\frac{x_1x_2}{2}}e^{-\frac{x_1/x_2}{2}}\delta(y-x_1)\frac{1}{2}x_1\paren{x_2+\frac{1}{x_2}+2}\frac{1}{2x_2}dx_1dx_2\\ 
    &=\frac{1}{\pi}\int_{0}^\infty e^{-y\paren{\frac{x_2}{2}+\frac{1}{2x_2}}}\frac{1}{2}y\paren{x_2+\frac{1}{x_2}+2}\frac{1}{2x_2}dx_2\\
    &=\frac{1}{2\pi}\int_{-\infty}^\infty e^{-y\cosh(t)}y(\cosh(t)+1)dx_2\\
    &=\frac{1}{\pi}\int_{0}^\infty e^{-y\cosh(t)}y(\cosh(t)+1)dx_2\\
    &=\frac{yK_0(\abs{y})+\abs{y}K_1(\abs{y})}{\pi},
\end{align*}
where $K_1$ is the modified Bessel function with parameter equal to 1. Moreover, the absolute value in the final result %with absolute value 
is readily obtained %easily 
after analyzing the case $y<0$, which is analogous.
Thus, we get 
\begin{align}
    \taustar(y) = 1 - \frac{K_0(\abs{y})}{yK_0(\abs{y})+\abs{y}K_1(\abs{y})}. \label{eqn:T_prod}
\end{align}
Finally, let us calculate the weak recovery threshold $\delta_c$:
\begin{align*}
    \frac{1}{\delta_c}=\int_{\bbR}\frac{\paren{\expt[s]{p(y\mathrel{\vert} s)(\inprod{s}{u}^2- 1)}}^2}{\expt[s]{p(y\mathrel{\vert} s)}} dy = \int_{\bbR} \frac{\paren{\frac{yK_0(\abs{y})+\abs{y}K_1(\abs{y})}{\pi} - \frac{K_0(\abs{y})}{\pi}}^2}{\frac{K_0(\abs{y})}{\pi}}dy\approx 1.68421 \approx (0.5937)^{-1},
\end{align*}
where the exact value was calculated in WolframAlpha. We note that this value exactly matches the threshold in \cite[pg.9]{troiani2024fundamental}.

\subsection{Mixed phase retrieval} \label{subsec:mp}

Let $\eta=\prob{\eps = 1}$ and define the following auxiliary quantities:
        \begin{align}
            \gamma &= \frac{1}{2} \brack{ 1 + \sqrt{4\rho^2\eta(1-\eta) + (2\eta - 1)^2} } , \notag \\
            a_1 &\coloneqq 1 + \frac{2(\gamma - \eta)}{\eta} + \paren{\frac{\gamma - \eta}{\eta\rho}}^2 , \notag \\
            a_2 &\coloneqq \eta + (1-\eta)\rho^2 + \frac{2(\gamma - \eta)}{\eta} + \paren{\frac{\gamma - \eta}{\eta\rho}}^2 \brack{ (1-\eta) + \eta\rho^2 } , \notag \\
            a_3 &\coloneqq (1-\eta)(1-\rho^2) + \paren{\frac{\gamma - \eta}{\eta\rho}}^2\eta(1-\rho^2) , \notag \\
            b &\coloneqq a_1 - a_3 . \notag 
        \end{align}
        Denote by $ \ell^* $ the unique solution in $ \ell \in ](\gamma / \eta)^2 - b, \infty[ $ to the following equation: 
        \begin{align}
            (\ell - a_3) \int_0^\infty\! \sqrt{\frac{2}{\pi}} e^{-y^2 / 2} \frac{(y^2 - 1)^2}{a_2 y^2 + \ell} \,d y &= \frac{1}{\gamma^2 \delta} . \notag 
        \end{align}
        Then, the optimal preprocessing function is given by 
        \begin{align}
            \cT(y) &= \frac{y^2 - 1}{\brack{ a_2 + \gamma (\ell^* - a_3) } y^2 + \ell^* - \gamma(\ell^* - a_3)} . \label{eqn:T_mpr} 
        \end{align}
This expression maximizes the asymptotic overlap $ \abs{\inprod{v_1^D}{w_1^*}} $ obtained by specializing our general formula in \Cref{thm:opt} to the mixed phase retrieval model. The derivations are along similar lines as detailed in \Cref{subsec:prod} for the model $ y = s_1s_2 $, and we leave out the explicit calculations.

%\section{Converse}
%\begin{theorem}
%    In the setting of \Cref{thm:eigvalconv}, let $\alpha_i$ be solutions to \Cref{eq:master_eq2}. If $\alpha_i\leq \lambdabardelta$, then %it holds 
%\vspace{-.5em}
    %\begin{equation}\label{eq:liminfconv}
        %\max_{l\in[p]}\liminf_{d\to\infty}\abs{\inprod{v_i^D}{e_l^{(d)}}}^2=0.
    %\end{equation}    
%\end{theorem}
%\begin{proof}
%
%\end{proof}

\section{Equivalence to \cite{troiani2024fundamental}}\label{subsec:equivTroiani}
\begin{proposition}
    For simultaneously diagonalizable matrices $E(y)$ and orthogonal signals, the recovery threshold from \Cref{thm:opt} matches the one in Lemma 4.1 of \cite{troiani2024fundamental}. This equivalence holds for all permutation-invariant link functions.
\end{proposition}
\begin{proof}
    The weak recovery threshold in \cite{troiani2024fundamental} is characterized as
    $$\frac{1}{\alpha_c} = \sup_{\substack{M\in \cS_p^+\\ \norm{F}{M}=1}}\norm{F}{\cF(M)}.$$
    Here, $\cF$ is an operator defined as 
    $$\cF(M) \coloneq \expt[y]{E(y)ME(y)},$$
    where $E(y)\coloneq \expt[s]{ss^\top - I_p\mathrel{\vert} y}$.
    As noted in the proof of Lemma 4.1 in \cite{troiani2024fundamental}, the linear operator $\cF$ is a self-map on the cone of positive semi-definite matrices. 
    Indeed, for any positive semi-definite matrix $ M $ and any vector $v$, $ v^\top \cF(M) v = \expt[y]{v^\top E(y) M E(y) v} = \expt[y]{\normtwo{M^{1/2} E(y) v}^2} \ge 0 $, meaning that $ \cF(M) $ is also positive semi-definite. 
    By the generalized Perron--Frobenius theorem for cone-preserving maps \cite{Krein-Rutman} (see also Theorem 1.1 in \cite{Du}, or Theorem 19.2 and Exercise 12 in \cite[\S19]{Deimling}), it holds that $1/\alpha_c$ is the largest eigenvalue of the linear operator $\cF$. 
    That is, there exists a positive semi-definite matrix $M^*$ with $ \norm{F}{M^*} = 1 $ for which it holds
    $$\sup_{\substack{M\in \cS_p^+\\ \norm{F}{M}=1}}\norm{F}{\cF(M)} = \norm{F}{\cF(M^*)}.$$
    The corresponding eigen-equation $ \cF(M^*) = \lambda^* M^* $ implies $ \inprod{\cF(M^*)}{M^*}_F = \lambda^* $. 
    The LHS must be non-negative since $ \inprod{\cF(M^*)}{M^*}_F = \sum_i \lambda_i^{M^*} \inprod{\cF(M^*)}{v_i^{M^*} (v_i^{M^*})^\top}_F \ge 0 $ due to the cone-preserving property of $\cF$ and non-negativity of each $ \lambda_i^{M^*} $. 
    % $ \inprod{\cF(vv^\top)}{vv^\top}_F = \expt[y]{(v^\top E(y) v)^2} \ge 0 $ for any vector $v$ and therefore $ \inprod{\cF(M)}{M}_F \ge 0 $ for any positive semi-definite matrix $M$ which is nothing but a non-negative combination of rank-$1$ matrices. 
    Now we have $ \norm{F}{\cF(M^*)}^2 = \lambda^* \inprod{\cF(M^*)}{M^*}_F = \inprod{\cF(M^*)}{M^*}_F^2 $, from which it follows that $ \norm{F}{\cF(M^*)} = \inprod{\cF(M^*)}{M^*}_F $, by non-negativity of the RHS just shown. 
    So, one can rewrite the optimal threshold as
    $$\sup_{\substack{M\in \cS_p^+\\ \norm{F}{M}=1}}\norm{F}{\cF(M)} =\sup_{\substack{M\in \cS_p^+\\ \norm{F}{M}=1}}\inprod{\cF(M)}{M}_{F}.$$
    Let us assume first that this maximum is obtained for a certain rank-1 matrix. Then, it would hold
\begin{equation}\label{eq:rankonemax}
        \sup_{\substack{M\in \cS_p^+\\ \norm{F}{M}=1}}\inprod{\cF(M)}{M}_{F} = \max_{\norm{2}{u}=1}\inprod{\cF(uu^\top)}{uu^\top}_{F},
    \end{equation}
    where $u\in \bbR^p$.
    Writing this out, we have
    \begin{align*}
        \max_{\norm{2}{u}=1}\inprod{\cF(uu^\top)}{uu^\top}_{F} &= \max_{\norm{2}{u}=1} \tr\paren{\expt[y]{E(y)uu^\top E(y)}uu^\top}\\
        &=\max_{\norm{2}{u}=1} \expt[y]{\tr\paren{u^\top E(y)uu^\top E(y)u}}\\
        &= \max_{\norm{2}{u}=1} \expt[y]{\paren{\expt[s]{\inprod{s}{u}^2- 1\mathrel{\vert} y}}^2}\\
        &= \max_{\norm{2}{u}=1} \expt[y]{\paren{\expt[s]{\frac{p(y\mathrel{\vert} s)}{p(y)}(\inprod{s}{u}^2- 1)}}^2}\\
        &= \max_{\norm{2}{u}=1} \expt[y]{\frac{\paren{\expt[s]{p(y\mathrel{\vert} s)(\inprod{s}{u}^2- 1)}}^2}{\paren{\expt[s]{ p(y\mathrel{\vert} s)}}^2}}\\
        &= \max_{\norm{2}{u}=1} \int_{\bbR}p(y)\frac{\paren{\expt[s]{p(y\mathrel{\vert} s)(\inprod{s}{u}^2- 1)}}^2}{\paren{\expt[s]{ p(y\mathrel{\vert} s)}}^2} dy\\
        &= \max_{\norm{2}{u}=1} \int_{\bbR}\frac{\paren{\expt[s]{p(y\mathrel{\vert} s)(\inprod{s}{u}^2- 1)}}^2}{\expt[s]{ p(y\mathrel{\vert} s)}} dy\\
        &= \frac{1}{\delta_c},
    \end{align*}
    where $\delta_c$ is exactly the threshold from \Cref{thm:opt}.
    
    Let us prove that when $E(y)$ has the same eigenvectors regardless of $y$, the equality in \eqref{eq:rankonemax} holds. Towards that end, we denote by $u_i$ ($i\in[p]$) the orthonormal eigenbasis of $E(y)$.
    We introduce the matrices
    $$A_{i} \coloneq u_iu_i^\top,\quad B_{j,k} = \frac{u_ju_k^\top+u_ku_j^\top}{\sqrt{2}},$$
    for $i,j,k\in[p]$ s.t.\ $j\neq k$. It holds that 
    \begin{align*}
        \inprod{A_{i_1}}{A_{i_2}}_F=0,\\
        \inprod{A_{i_1}}{A_{i_1}}_F=1,\\
        \inprod{A_{i_1}}{B_{j_1,k_1}}_F=0,\\
        \inprod{B_{j_1,k_1}}{B_{j_2,k_2}}_F=0,\\
        \inprod{B_{j_1,k_1}}{B_{j_1,k_1}}_F=1,
    \end{align*}
    for any $i_1,i_2,j_1,j_2,k_1,k_2\in[p]$ such that $i_1\neq i_2$ and $(j_1,k_1)\neq (j_2,k_2)$.
    Due to the fact that $u_1,\dots,u_p$ form an orthonormal basis, the set 
    $$U\coloneq \{A_i,B_{j,k};i,j,k\in[p],j<k\}$$
    is an orthonormal basis of the symmetric matrices $\cS_p$ equipped with Frobenius inner product. 
    Notice also that
    $$\sup_{\substack{M\in \cS_p^+\\ \norm{F}{M}=1}}\inprod{\cF(M)}{M}_{F}\leq \sup_{\substack{M\in \cS_p\\ \norm{F}{M}=1}}\inprod{\cF(M)}{M}_{F},$$
    since all positive definite matrices are by definition symmetric. By the previous discussion, we decompose the matrix $M\in \cS_p$ such that $\norm{F}{M}=1$, as
    $$M= \sum_{i=1}^p \alpha_i A_i + \sum_{1\leq j<k\leq p} \beta_{j,k} B_{j,k},$$
    where $\alpha_i, \beta_{j,k}\in \bbR$ such that 
    $$\sum_{i=1}\alpha_i^2 + \sum_{1\leq j<k\leq p} \beta_{j,k}^2=1.$$
    Since the $u_i$'s are eigenvectors of $E(y)$, it follows that every element of the set $U$ is an eigenvector of the operator $\cF$. Let us denote the corresponding eigenvalues as $\lambda_{A_i}$ and $\lambda_{B_{j,k}}$. Furthermore, it holds that
    \begin{align*}
        \lambda_{B_{j,k}} &= \inprod{\cF(B_{j,k})}{B_{j,k}}\\
                         &= \expt[y]{u_j^\top E(y)u_j\cdot u_k^\top E(y)u_k}\\
                         &\leq  \sqrt{\expt[y]{(u_j^\top E(y)u_j)^2}}\cdot \sqrt{\expt[y]{(u_k^\top E(y)u_k)^2}}\\
                         & = \sqrt{\inprod{\cF(A_j)}{A_j}} \cdot \sqrt{\inprod{\cF(A_k)}{A_k}}\\
                         & = \sqrt{\lambda_{A_j}\lambda_{A_k}},
    \end{align*}
    where H\"older's inequality was used to get the bound.
    Then, it holds that
    \begin{equation}\label{eq:coeffeigenval}
        \begin{split}
            \inprod{\cF(M)}{M}_{F} &= \inprod{\cF\paren{\sum_{i=1}^p \alpha_i A_i + \sum_{1\leq j<k\leq p} \beta_{j,k} B_{j,k}}}{\sum_{i=1}^p \alpha_i A_i + \sum_{1\leq j<k\leq p} \beta_{j,k} B_{j,k}}\\
            &= \inprod{\sum_{i=1}^p\alpha_i \lambda_{A_i} A_i + \sum_{1\leq j<k\leq p} \beta_{j,k}\lambda_{B_{j,k}} B_{j,k}}{\sum_{i=1}^p \alpha_i A_i + \sum_{1\leq j<k\leq p} \beta_{j,k} B_{j,k}}\\
            &= \sum_{i=1}^p\alpha_i^2 \lambda_{A_i} + \sum_{1\leq j<k\leq p} \beta_{j,k}^2\lambda_{B_{j,k}}\\
            &\leq \sum_{i=1}^p\alpha_i^2 \lambda_{A_i} + \sum_{1\leq j<k\leq p} \beta_{j,k}^2\sqrt{\lambda_{A_j}\lambda_{A_k}}\\
            &\leq \max_{i}\lambda_{A_i}\cdot\paren{\sum_{i=1}\alpha_i^2 + \sum_{1\leq j<k\leq p} \beta_{j,k}^2}\\
            & = \max_{i}\lambda_{A_i}\\
            & \leq \max_{\norm{2}{u}=1}\inprod{\cF(uu^\top)}{uu^\top}_{F},
        \end{split}
    \end{equation}
which proves \eqref{eq:rankonemax}. % holds, hence giving the claim.

    Lastly, we verify that $E(y)$ is diagonalizable when the link function $q$ is permutation invariant. Towards that end, it holds for arbitrary $i,j\in[p]$ that 
    \begin{align*}
        (E(y))_{i,i} &=\expt[s]{s_i^2 - 1\mathrel{\vert} y} \\
        &=\expt[s]{s_i^2 - 1\mathrel{\vert} q(\dots, s_i,\dots,s_j,\dots)=y}\\
        &=\expt[s]{s_j^2 - 1\mathrel{\vert} q(\dots, s_j,\dots,s_i,\dots)=y}\\
        &=\expt[s]{s_j^2 - 1\mathrel{\vert} q(\dots, s_i,\dots,s_j,\dots)=y}\\
        &=(E(y))_{j,j},
    \end{align*}
    since the link function $q$ is permutation invariant. Similarly, for arbitrary $i_1,i_2,j_1,j_2\in[p]$ such that $i_1\neq j_1$ and $i_2\neq j_2$, it holds that
    \begin{align*}
        (E(y))_{i_1,j_1} &=\expt[s]{s_{i_1}s_{j_1} \mathrel{\vert} y} \\
        &=\expt[s]{s_{i_1}s_{j_1} \mathrel{\vert} q(\dots, s_{i_1},\dots,s_{i_2},\dots,s_{j_1},\dots,s_{j_2},\dots)=y}\\
        &=\expt[s]{s_{i_2}s_{j_2} \mathrel{\vert} q(\dots, s_{i_2},\dots,s_{i_1},\dots,s_{j_2},\dots,s_{j_1},\dots)=y}\\
        &=\expt[s]{s_{i_2}s_{j_2} \mathrel{\vert} q(\dots, s_{i_1},\dots,s_{i_2},\dots,s_{j_1},\dots,s_{j_2},\dots)=y}\\
        &=(E(y))_{i_2,j_2}.
    \end{align*}
    Thus, if we denote by $a(y)\coloneq E(y)_{i,i}$ for arbitrary $i\in [p]$ and by $b(y)\coloneq E(y)_{i_1,j_1}$ for arbitrary $i_1,j_1\in[p]$ s.t.\ $i_1\neq j_1$, it follows that 
    \begin{equation}\label{eq:matey}
        E(y) =\matrix{
        a(y) & b(y) & \cdots & b(y) \\
           b(y)     & a(y) & \cdots & b(y) \\
         \vdots       &  \vdots       & \ddots & \vdots  \\
          b(y)      &    b(y)     &  \cdots      & a(y)
    }.
    \end{equation}
    Notice that this matrix is going to have an eigenbasis 
    \begin{align*}
        u_1 &\coloneq (1,1,1,\dots 1)/\sqrt{p},\\
        u_2 &\coloneq (1,-1,0,\dots 0)/\sqrt{2},\\
        \vdots\\
        u_p &\coloneq (0,0,\dots,0, 1,-1)/\sqrt{2},
    \end{align*} regardless of the value of $y$, which proves the final claim.
\end{proof}

\end{document}